\documentclass[11pt]{article}
\usepackage[utf8]{inputenc}
\usepackage[T1]{fontenc}
\usepackage[a4paper,margin=1in]{geometry}
\usepackage{microtype}
\usepackage{booktabs}
\usepackage{nicefrac}
\usepackage{tocloft}
\usepackage[toc,page,header]{appendix}
\usepackage{minitoc}

\usepackage[authoryear,round]{natbib}
\usepackage[small,bf]{caption}
\usepackage{amsmath}
\usepackage{amssymb}
\usepackage{amsthm}
\usepackage{graphicx}
\usepackage{placeins}
\usepackage{float}

\usepackage[most]{tcolorbox}
\usepackage{xcolor}
\allowdisplaybreaks

\definecolor{resultblue}{RGB}{45,86,135}
\definecolor{resultamber}{RGB}{166,111,0}
\definecolor{resultgreen}{RGB}{41,116,76}
\newtcolorbox{questionbox}{
	enhanced,colback=black!2,colframe=resultblue!65!black,
	boxrule=0.65pt,arc=1.2pt,left=5pt,right=5pt,top=4pt,bottom=4pt,
	before skip=7pt,after skip=8pt}

\newtheorem{theorem}{Theorem}[section]
\newtheorem{proposition}[theorem]{Proposition}
\newtheorem{lemma}[theorem]{Lemma}
\newtheorem{corollary}[theorem]{Corollary}

\theoremstyle{definition}

\theoremstyle{remark}

\usepackage{amsmath,amsfonts,bm}

\def\eqref#1{equation~\ref{#1}}

\def\1{\bm{1}}

\DeclareMathAlphabet{\mathsfit}{\encodingdefault}{\sfdefault}{m}{sl}
\SetMathAlphabet{\mathsfit}{bold}{\encodingdefault}{\sfdefault}{bx}{n}

\usepackage[linktocpage=true,pagebackref=true]{hyperref}
\renewcommand*{\backref}[1]{}
\renewcommand*{\backrefalt}[4]{(\small%
  \ifcase #1 not cited%
  \or cited on page~#2%
  \else cited on pages #2%
  \fi)}
\hypersetup{
	colorlinks=true,
	citecolor=blue,
	linkcolor=blue,
	urlcolor=blue
}
\usepackage{url}

\usepackage{algpseudocode}

\title{\bfseries When and How Does Muon Precondition? The Role of Structured Gradient Noise}

\author{%
    Xiaoran Chen$^{1,*}$\hspace{1em}
  Binghui Li$^{2,*,\dagger}$\hspace{1em}
  Liwei Wang$^{2,3,5,\dagger}$\hspace{1em}
  Pinyan Lu$^{4,\dagger}$
  \\[1em]
  \begin{tabular}{c}
    \small $^1$School of EECS, Peking University\\
    \small $^2$Center for Machine Learning Research, Peking University\\
    \small $^3$Center for Data Science, Peking University\\
    \small $^4$Key Laboratory of Interdisciplinary Research of Computation and Economics,\\
    \small Shanghai University of Finance and Economics\\
    \small $^5$State Key Laboratory of General Artificial Intelligence, Peking University
  \end{tabular}
  \\[0.8em]
  {\small\texttt{chenxiaoran@stu.pku.edu.cn}} \quad
  {\small\texttt{\{libinghui,wanglw\}@pku.edu.cn}}\\
  {\small\texttt{lu.pinyan@mail.shufe.edu.cn}}
}

\date{}

\begin{document}
	\raggedbottom 
	
	\maketitle

    \begingroup
  \renewcommand{\thefootnote}{\fnsymbol{footnote}}
  \footnotetext[1]{Equal contribution.}
  \footnotetext[2]{Corresponding authors.}
\endgroup
	
	\begin{abstract}
		Muon has become a popular optimizer for large language model training, yet its theoretical foundations remain poorly understood. In this work, we study how gradient noise geometry shapes Muon’s preconditioning benefits. We first show, theoretically and empirically, that under the classical isotropic noise assumption, Muon offers no acceleration beyond a rescaling of the learning rate. This motivates us to investigate structured noise arising from mini-batch sampling in associative memory with heterogeneous task frequencies. We establish an exact balancing threshold for two tasks and extend the analysis to multiple tasks, revealing how task coverage and gradient reliability jointly shape learning. Our analysis shows that Muon can soften the dependence on task frequency from linear to square-root, helping balance the learning of frequent and rare tasks. We further identify rotation-induced negative transfer in SignSGD that Muon avoids, and explain how momentum reduces sampling noise while tracking changing gradients. Controlled numerical experiments validate these mechanisms, and a preliminary LLM pretraining study illustrates the practical promise of fast–slow momentum.
	\end{abstract}
	
\doparttoc
\faketableofcontents
\part{}
\vspace*{-2.2em}

\section{Introduction}
	\label{sec:introduction}
	
	Muon has emerged as a prominent alternative to Adam for training the matrix parameters of large language models (LLMs) \citep{jordan2024muon,liu2025scalable}. Its defining operation replaces a gradient or momentum matrix by an approximate polar factor, flattening its nonzero singular values before the parameter update. This matrix-level normalization offers a different way to allocate updates across directions from Adam's coordinatewise rescaling. Recent frontier-model reports, including DeepSeek--V4.1--Flash~\citep{deepseek2026v41flash} and Kimi--K3~\citep{kimi2026k3}, adopt Muon variants for matrix optimization, with head-wise orthogonalization further adapting the update to attention structure. These developments make understanding when and why Muon improves learning an increasingly practical theoretical question.
	
	Existing theory has established convergence guarantees for practical orthogonalization procedures \citep{kim2026newtonschulz} and identified preconditioning benefits in structured optimization problems \citep{ma2026preconditioning}. Full-batch associative-memory analyses further show how spectral normalization can balance learning across tasks with heterogeneous frequencies \citep{wang2025tailend,li2026associative}. In LLM pretraining, however, Muon acts on stochastic mini-batch gradients, and nonlinear normalization need not preserve the benefits predicted from population gradients. In particular, writing $\operatorname{msgn}$ for the partial polar map, generally $\mathbb E[\operatorname{msgn}(G_t)\mid W_t]\neq\operatorname{msgn}(\mathbb E[G_t\mid W_t])$. Recent work has begun to characterize stochastic spectral dynamics and geometry-dependent noise scales \citep{paquette2026phases,naganuma2026adaptivebatch}; the interaction between task occurrence, gradient reliability, and temporal averaging motivates the question we study here:
	
    \begin{center}
        \textbf{Q1:} \textit{When does Muon retain a directional learning advantage under stochastic gradient noise, and how do noise structure, batch size, and momentum determine this advantage?}
    \end{center}

	\textbf{Isotropic noise limits preconditioning.}
	Standard stochastic optimization analyses often control gradient noise through a scalar variance bound \citep{ghadimi2013stochastic}, leaving its directional structure unspecified. We use isotropic Gaussian noise as a diagnostic baseline. In the fixed-dimensional low-SNR limit, we prove that Muon's conditional mean update is, to first order, a scalar multiple of the population gradient, as confirmed numerically in Figure~\ref{fig:isotropic-main}. SGD with a matched learning rate therefore reproduces this first-order mean drift: applying the partial polar map to each stochastic mini-batch gradient does not, by itself, produce directional preconditioning in the first-order mean response. Explaining such gains in this regime requires examining noise structure.
	
	\textbf{Sampling controls balancing.}
	We study associative memory with heterogeneous task frequencies $p_j$ \citep{wang2025tailend,li2026associative}. Mini-batch sampling can leave task directions entirely unobserved; without momentum, Muon's nullspace preservation then prevents updates along those directions. For two tasks, we prove an exact balancing threshold $B_{\rm bal}$: below it, margins grow at unequal linear rates; above it, their gap has uniformly bounded exponential moments. For rare tasks, $B_{\rm bal}\sim(\log 2)/p_{\min}$. In a higher-rank hierarchy, local perturbation estimates identify the expected task count $Bp_j$ as a controlling scale. Experiments exhibit corresponding frequency-dependent crossovers, with frequent groups approaching balanced growth while rare groups remain occurrence-limited.
	
	\textbf{Label noise limits reliability.}
	Label fluctuations persist at the population optimum, so frequent observation alone does not ensure a reliable gradient. Exact gradient statistics and a scalar projected counting model distinguish regimes governed by task counts $Bp_j$ and corruption counts $Bp_j\rho_\alpha$. Near the noisy optimum, an isotropic covariance closure with a locally homogeneous frequency approximation predicts a Muon contraction coefficient proportional to $\sqrt{Bp}$, compared with SGD's linear dependence on the local frequency scale $p$.
	
	\textbf{Orthogonal equivariance differs from coordinatewise sign updates.}
	We use SignSGD as a proxy for Adam's coordinatewise adaptivity \citep{balles2018dissecting,kunstner2023noise}. Muon is orthogonally equivariant whereas SignSGD is basis dependent. In a two-task rotated sparse-support construction, a batch containing only the frequent task leaves the rare-task Muon margin unchanged but can decrease the corresponding SignSGD margin under certain conditions.
	
	\textbf{Momentum trades averaging for tracking.}
	Momentum reduces sampling fluctuations, with frozen-state variance-equivalent batch size $B_{\rm eff}=B(1+\beta)/(1-\beta)$, but longer memory introduces tracking lag. Under local tracking assumptions, the sufficient window $\eta\ll1-\beta\ll Bp_j$ motivates longer memory for rare, slowly varying directions and shorter memory where rapid tracking is needed. This provides a sampling-based perspective on curvature-conditioned multiscale momentum \citep{zhu2026multiscale}: the local Hessian block $p_j\Gamma_j$ links rare tasks to low-curvature directions when label statistics are comparable. It also connects to occurrence-dependent forgetting in Activation-Keyed Momentum \citep{hong2026activation}, supporting selective long-memory averaging as a means to improve estimates in slowly varying directions.
	
	Our contributions are threefold:
	\begin{itemize}
		\item \textbf{A diagnostic for stochastic directional gains.} We prove a low-SNR expansion under isotropic Gaussian noise, showing that Muon's first-order conditional mean response can be matched by learning-rate-rescaled SGD despite samplewise spectral normalization.
		\item \textbf{A theory of structured sampling.} We establish an exact two-task balancing threshold and derive local higher-rank sampling estimates. With stochastic labels, we separate exact gradient statistics, scalar projected counting laws, and a square-root frequency prediction under isotropic covariance and locally homogeneous frequency approximations. A momentum-free Muon--SignSGD comparison isolates rotation-induced negative transfer under sparse sampling.
		\item \textbf{A statistical basis for heterogeneous momentum.} We characterize frozen-state variance reduction and local tracking error under temporal averaging, explaining when longer memory can help rare directions and providing a sampling-based interpretation of fast--slow momentum designs. Controlled experiments test these mechanisms, and a preliminary pretraining comparison illustrates their practical relevance.
	\end{itemize}
	
	\section{Related Work}
	\label{sec:related-work}
	

    \paragraph{Muon and adaptive optimization.}
    Muon applies approximate polar normalization to matrix-valued momentum \citep{jordan2024muon,liu2025scalable}. Existing work establishes convergence with Newton--Schulz iterations \citep{kim2026newtonschulz}, preconditioning benefits in structured matrix problems \citep{ma2026preconditioning}, and generalization gains from balanced learning on imbalanced data \citep{vasudeva2025spectraldesign}. \citet{magakyan2026spectral} further identify scaling laws for momentum spectra and their implications for orthogonalization accuracy. For coordinatewise adaptivity, \citet{li2026signsgdscaling} analyze SignSGD in feature-space linear regression, characterizing its loss dynamics and batch-size benefits relative to SGD. We complement these perspectives by studying how stochastic gradient structure governs Muon's conditional mean response and distinguishes polar normalization from coordinatewise sign updates under sparse task sampling.
	
	\paragraph{Stochastic spectral dynamics.}
	\citet{paquette2026phases} characterize batch-dependent preconditioning regimes for Muon and SignSGD in high-dimensional matrix least squares. \citet{naganuma2026adaptivebatch} derive non-Euclidean gradient-noise scales for adaptive batch selection, while \citet{eschenhagen2026clarifying} examine how stochasticity and temporal adaptation distinguish Shampoo from spectral descent. Our analysis focuses on random task support and persistent label fluctuations: we identify an exact finite-batch balancing threshold and separate task occurrence from gradient reliability. The isotropic Gaussian result provides a complementary fixed-dimensional diagnostic for the absence of first-order directional preconditioning.
	
	\paragraph{Associative memory learning.}
	\citet{wang2025tailend} connect Muon's advantage over Adam to balanced learning of rare associations. \citet{li2026associative} analyze hierarchical task frequencies, label noise, and scaling laws through implicit matrix preconditioning. Beyond orthogonal embeddings, \citet{kim2026capacity} derive recovery and capacity scaling with Gaussian embeddings, including critical batch-size behavior. Building on the orthogonal associative-memory framework, we isolate finite-batch sampling effects through exact two-task dynamics and a higher-rank frequency-dependent regime map. Our SignSGD comparison further identifies rotation-induced negative transfer when tasks are absent from a batch.
	
	\paragraph{Heterogeneous momentum.}
	Muon momentum has been studied as spectral filtering \citep{li2026denoise}, while activation-keyed and curvature-conditioned methods adapt memory across directions \citep{hong2026activation,zhu2026multiscale}. We complement these designs with a sampling-based account of variance reduction and tracking lag.
	
	\section{Preliminaries}
	\label{sec:preliminaries}
	
	\subsection{Notations.}
	
	For $G\in\mathbb R^{d\times d}$ with SVD $G=U\Sigma V^\top$, define the partial polar factor $\operatorname{msgn}(G):=U\operatorname{diag}(\mathbf 1\{\sigma_i>0\})V^\top$, which replaces each nonzero singular value by one and reduces to $UV^\top$ when $G$ is full rank. The map is positively scale invariant, $\operatorname{msgn}(aG)=\operatorname{msgn}(G)$ for $a>0$, and orthogonally equivariant, $\operatorname{msgn}(QGR^\top)=Q\operatorname{msgn}(G)R^\top$ for orthogonal $Q,R$. It also preserves the right nullspace: $\ker(\operatorname{msgn}(G))=\ker(G)$.
	
	Given the stochastic gradient $G_t$, we study the idealized Muon dynamics $M_t=\beta M_{t-1}+(1-\beta)G_t$ and $W_{t+1}=W_t-\eta\operatorname{msgn}(M_t)$, where $\eta>0$ is the learning rate and $\beta\in[0,1)$ is the momentum coefficient; $\beta=0$ recovers Muon without momentum. Our theoretical analysis uses exact partial polar normalization in task coordinates; rectangular representations are introduced where needed. Unless stated otherwise, mini-batches are independent across iterations and weight decay is omitted. The pretraining experiment uses the practical construction cited in Section~\ref{sec:fast-slow-pretraining}.
	
	\subsection{Warmup: Isotropic Noise Limits Preconditioning}
	\label{sec:isotropic}

    \begin{figure}[t]
		\centering
		\includegraphics[width=0.98\linewidth]{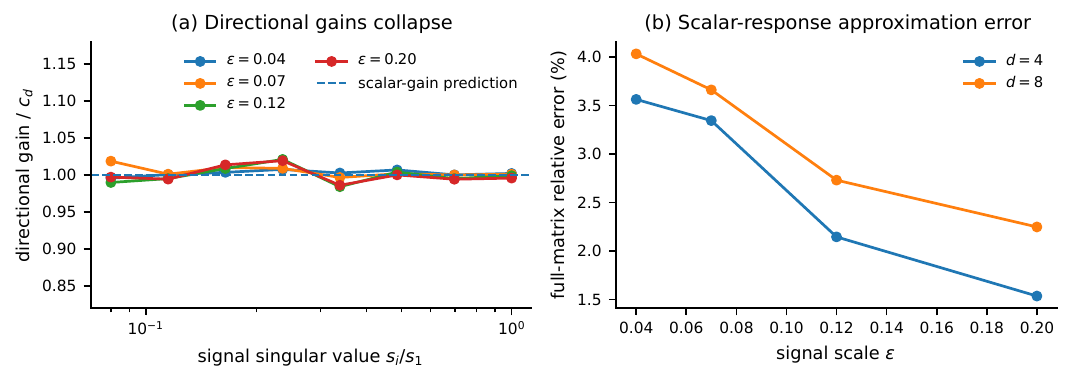}
		\caption{\textbf{Isotropic noise removes first-order directional preconditioning at low SNR.} Left: for an anisotropic $d=8$ signal, the response gain along each singular direction, normalized by $c_d$, collapses near one across a wide range of signal singular values. Right: the full conditional-mean matrix remains within a few percent of the scalar-response prediction over the tested low-SNR amplitudes.}
		\label{fig:isotropic-main}
	\end{figure}
	
	Consider momentum-free Muon. Condition on the current iterate $W$ and write $G=\bar G+\sigma\Xi$, where $\bar G=\nabla\mathcal L(W)$ and $\Xi\in\mathbb R^{d\times d}$ has i.i.d.\ standard Gaussian entries. Although each polar update has a flat nonzero spectrum, its mean need not retain a preconditioning benefit.
	
		\begin{proposition}[Isotropic Gaussian mean response]
			\label{prop:gaussian-isotropization}
			Fix $d$ and let $c_d=d^{-2}\mathbb E\|\Xi\|_*>0$. As $\|\bar G\|_F/\sigma\to0$,
			\[
			\mathbb E[\operatorname{msgn}(\bar G+\sigma\Xi)]
			=\frac{c_d}{\sigma}\bar G
			+O_d\!\left(\frac{\|\bar G\|_F^3}{\sigma^3}\right),
			\]
			with the remainder measured in Frobenius norm.
		\end{proposition}
	At the same parameter state, Muon with learning rate $\eta_{\rm M}$ has first-order conditional mean update $-\eta_{\rm M}(c_d/\sigma)\bar G$, matching SGD with $\eta_{\rm S}=\eta_{\rm M}c_d/\sigma$. Thus its effective first-order preconditioner is scalar: all singular directions receive the same multiplicative gain, leaving their relative scaling unchanged. This equivalence concerns conditional means, not update covariances or stochastic trajectories.
	Figure~\ref{fig:isotropic-main} confirms the common directional gain for an anisotropic signal. Appendix~\ref{app:isotropic-proof} gives the proof via rotational symmetry. In the fixed-dimensional low-SNR regime, directional preconditioning in the first-order mean response therefore requires structure absent from this isotropic benchmark, motivating our analysis of structured sampling.
	
	\section{Structured Sampling in Associative Memory}
	\label{sec:associative-setup}
	
	\subsection{Associative memory and stochastic gradients}
	\label{sec:assoc-model}
	
	We study a linear associative memory \citep{wang2025tailend,li2026associative} that maps query embeddings to answer logits through a trainable matrix $W$. There are $K\ge2$ associations $(E_j,\widetilde E_j)$, sampled with probabilities $p_j>0$, where $\sum_jp_j=1$. Write $P=\operatorname{diag}(p_1,\ldots,p_K)$. Let $E=[E_1,\ldots,E_K]$ and $\widetilde E=[\widetilde E_1,\ldots,\widetilde E_K]$, with $E^\top E=\widetilde E^\top\widetilde E=I_K$. For query $E_j$, answer $i$ has logit $\widetilde E_i^\top WE_j=\overline W_{ij}$, where $\overline W=\widetilde E^\top WE$. The prediction probabilities are $S_{ij}=\exp(\overline W_{ij})/\sum_{k=1}^K\exp(\overline W_{kj})$, and the deterministic-label population loss is $\mathcal L(W)=-\sum_{j=1}^Kp_j\log S_{jj}$.
	
	For a batch of size $B$, let $q_j=N_j/B$ be the empirical task frequency and $Q=\operatorname{diag}(q)$, where $Bq\sim\operatorname{Multinomial}(B,p)$. The negative stochastic gradient is $\widetilde E(I-S)QE^\top$ (Appendix~\ref{app:associative-gradient}). Orthogonal equivariance of the partial polar map gives the exact task-coordinate Muon update without momentum:
	\begin{equation}
		\overline W_{t+1}
		=\overline W_t+\eta\,\operatorname{msgn}\bigl((I-S_t)Q_t\bigr).
		\label{eq:assoc-muon-update}
	\end{equation}
	If task $j$ is absent, $q_j=0$, its gradient column in task coordinates is zero. Nullspace preservation then makes the corresponding polar-update column zero. This absence guarantee applies to the momentum-free update; historical momentum can remain nonzero when a task is absent from the current batch.
	
	\subsection{Two tasks: an exact finite-batch balancing transition}
	\label{sec:two-task-main}
	
	With $K=2$, the stochastic gradient is rank one, which makes the finite-batch sampling effect exactly solvable. Let $p_1\ge p_2=:p_{\min}>0$ and let $\Delta_1,\Delta_2$ be the two target-vs.-nontarget margins.
	
		\begin{theorem}[Two-task finite-batch balancing transition]
			\label{thm:two-task-coverage}
			Assume $\Delta_1(0)=\Delta_2(0)=0$ and fix $\eta>0$. Let $B_{\rm bal}=\log2/[-\log(1-p_{\min})]$.
			
			\textbf{(i) Unequal growth.} If $B<B_{\rm bal}$, then $(\Delta_1(t),\Delta_2(t))/(\eta t)\xrightarrow{\mathrm{a.s.}}\sqrt2\,(p_1^B,1-p_1^B)$.
			
			\textbf{(ii) Controlled gap.} If $B>B_{\rm bal}$, then $\sup_t\mathbb E e^{a|\Delta_1(t)-\Delta_2(t)|}<\infty$ for some $a>0$, and $(\Delta_1(t)-\Delta_2(t))/t\xrightarrow{\mathrm{a.s.}}0$.
		\end{theorem}
	The threshold is controlled by the probability $p_1^B$ of a batch containing only the frequent task. As its margin pulls far ahead, a pure task-1 batch advances only $\Delta_1$, while every batch containing task~2 asymptotically favors the rare task. The boundary $p_1^B=1/2$ separates outward from restoring tail drift. For $p_{\min}\ll1$, $B_{\rm bal}\sim(\log2)/p_{\min}$: batches of order the inverse tail frequency suffice for equal growth to leading linear order. The critical case $B=B_{\rm bal}$ is not covered by the theorem.
	
	\subsection{Higher rank: frequency-dependent dynamical regimes}
	\label{sec:hierarchical-memory}
	
	We next ask how the two-task results extend when many associations occur at different population frequencies. To represent multiple frequency scales while retaining a tractable matrix structure, we group the \(K=MC\) associations into \(M\) blocks of size \(C\). Block \(\mu\) has total probability \(\widetilde p_\mu\), distributed uniformly among its \(C\) tasks, so that
	$p_j=\widetilde p_\mu/C$ for $j\in\mu$, with $\sum_{\mu=1}^M\widetilde p_\mu=1$.
	This block model can be viewed as a coarse-grained representation of a heterogeneous frequency distribution: different blocks correspond to different frequency levels, while tasks within the same block are statistically exchangeable. We take $\overline W_0=0$. Full-batch dynamics preserve this block symmetry, whereas mini-batch fluctuations generally break it. For $i\ne j$ in the same block, the margin $\overline W_{jj}-\overline W_{ij}$ compares the correct answer with an incorrect answer from that block. The first-order expansion is in within-block count fluctuations at fixed block totals (Appendix~\ref{app:block-minibatch}).
	
		\begin{proposition}[Sampling and within-block margin growth]
			\label{prop:hierarchical-regimes}
			Consider momentum-free Muon and $i\ne j$ in the same block.
			
			\textbf{(i) Full batch.} From $\overline W_0=0$, each margin $\overline W_{jj}-\overline W_{ij}$ increases by exactly $\eta$ per step, independently of task frequency.
			
			\textbf{(ii) Sparse sampling.} This margin is unchanged when task $j$ is absent; the task appears with probability $1-(1-p_j)^B\sim Bp_j$ as $Bp_j\to0$.
			
			\textbf{(iii) Well-sampled blocks.} Under the local, well-sampled assumptions of Appendix~\ref{app:block-coupling}, at a fixed block-symmetric state the first-order margin increment is
			\[
			\eta+O_{\mathbb P}\!\left(\frac{\eta}{\sqrt{Bp_j}}\right),
			\qquad Bp_j\to\infty.
			\]
		\end{proposition}
	Thus sparse sampling limits update opportunities, whereas well-sampled blocks recover the full-batch margin increment to first order. The underlying contrast-to-contrast perturbation cancels, leaving only coupling to block-average directions. These local estimates do not establish uniform long-time balancing of the stochastic hierarchy. Appendix~\ref{app:hierarchical-continuum-refinement} separately proves finite-horizon convergence to the full-batch flow under large-batch and small-step limits.
	
	Figure~\ref{fig:coverage-main}(b) shows corresponding trajectory-level crossovers: measured late margin-growth rates approximately align against $Bp_j$ and approach the full-batch response as expected task counts increase. Frequent and rare groups can therefore exhibit different empirical regimes at the same batch size.
	
	\begin{figure}[t]
		\centering
		\includegraphics[width=0.98\linewidth]{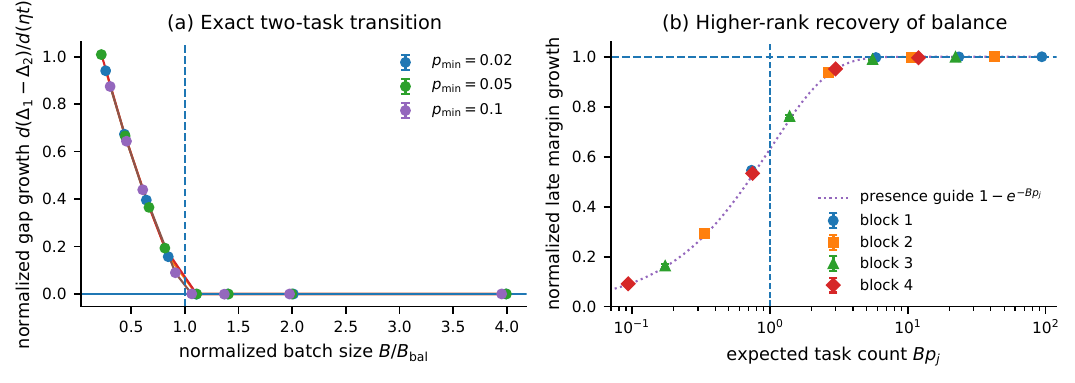}
		\caption{\textbf{Mini-batch sampling controls the transition from frequency-limited to balanced growth.} Left: the two-task gap-growth transition occurs at the exact threshold $B/B_{\rm bal}=1$. Right: in the higher-rank hierarchy, late margin-growth rates from different frequency blocks collapse when plotted against the expected task count $Bp_j$ and approach the full-batch unit response as that count increases. Error bars summarize 64 trajectories; the dotted curve is the task-presence guide $1-e^{-Bp_j}$.}
		\label{fig:coverage-main}
	\end{figure}
	
	Appendices~\ref{app:two-task-proof} and~\ref{app:hierarchical-memory} give the exact two-task proof and the higher-rank perturbation analysis underlying the approach to balanced dynamics.
	
	\subsection{Label noise: three stochastic regimes}
	\label{sec:label-noise}
	
	We next make the observed answer stochastic, with $0<\alpha<1$. Conditional on task $j$, symmetric label noise leaves the target label unchanged with probability $c_\alpha=1-\alpha+\alpha/K$ and returns a nontarget label with aggregate probability $\rho_\alpha=1-c_\alpha$. The crucial difference from task sampling is that label fluctuations remain nonzero at the finite population optimum.
	
	Let $d_j$ be the conditional label distribution, $\delta_j=Se_j-d_j$, and $\Gamma_j=\operatorname{diag}(d_j)-d_jd_j^\top$. For the stochastic gradient column $\widehat g_j$, define $\mathrm{SNR}_j^2=\|\mathbb E\widehat g_j\|_2^2/\operatorname{tr}\operatorname{Cov}(\widehat g_j)$.
	
		\begin{proposition}[Task frequency and gradient reliability decouple under label noise]
			\label{prop:label-gradient-statistics}
			\textbf{(i) Deterministic labels.} For $\delta_j\ne0$, $\mathrm{SNR}_j^2=Bp_j/(1-p_j)$.
			
			\textbf{(ii) Nondegenerate label noise.} As $\delta_j\to0$, $\mathrm{SNR}_j^2\sim Bp_j\|\delta_j\|_2^2/\operatorname{tr}\Gamma_j\to0$.
		\end{proposition}
	With deterministic labels, the stochastic fluctuation of a task shrinks together with its residual, so the number of task observations is enough to characterize directional reliability. Label noise breaks this one-scale picture: near the population optimum the mean residual vanishes, while the random observed label continues to fluctuate. A task can therefore be sampled many times ($Bp_j\gg1$) and still have a low-SNR gradient.
	
	Label noise introduces a second sampling scale, $Bp_j\rho_\alpha$, the expected number of corrupted task-$j$ observations. When $\rho_\alpha$ is small, it can be much smaller than the task count $Bp_j$, leaving a broad intermediate range in which the task appears regularly but corruptions are still rare.
	Write $m_{B,p}(\Delta)$ for the mean margin increment divided by $\eta$ in the scalar counting model of Appendix~\ref{app:label-coverage-dynamics}.
	
		\begin{proposition}[Projected regimes and local closure under label noise]
			\label{prop:label-regimes}
			Assume $\rho_\alpha<1/2$.
			
			\textbf{(i) Projected counting regimes.} With $p=p_j$,
			\[
			\begin{aligned}
				m_{B,p}(\Delta)&=Bp(1-2\rho_\alpha)+O((Bp)^2),
				&& Bp\ll1,\\
				\lim_{\Delta\to\infty}m_{B,p}(\Delta)&\approx1,
				&&1\ll Bp\ll\rho_\alpha^{-1},\\
				\lim_{\Delta\to\infty}m_{B,p}(\Delta)&<0,
				&&Bp\rho_\alpha\gg1.
			\end{aligned}
			\]
			\textbf{(ii) Local closure.} Near the noisy population optimum, the isotropic covariance closure with a locally homogeneous frequency scale $p$ predicts contraction coefficients proportional to $\sqrt{Bp}$ for Muon and $p$ for SGD.
		\end{proposition}
	The counting model separates occurrence-limited, corruption-sparse, and corruption-resolved behavior; a broad intermediate window requires small $\rho_\alpha$. Resolving corruptions reverses the projected large-margin drift, motivating a separate local analysis near the noisy optimum. The closure predicts weaker frequency dependence for Muon than for SGD. Both analyses use auxiliary models; Figure~\ref{fig:label-regimes-main} tests the closure prediction in the complete categorical model.
	
	\begin{figure}[t]
		\centering
		\includegraphics[width=0.98\linewidth]{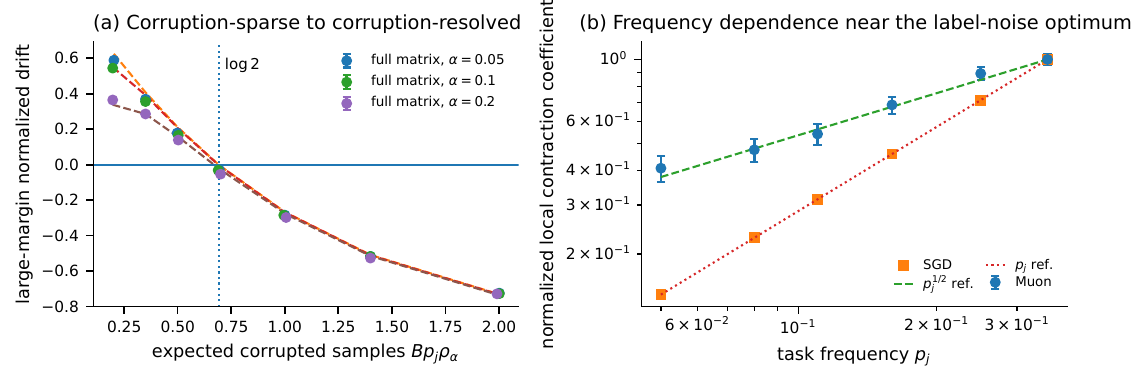}
		\caption{\textbf{Corrupted observations introduce a second sampling scale.} Left: full-matrix Muon changes from positive to negative large-margin drift as the expected number of corrupted task-$j$ observations, $Bp_j\rho_\alpha$, crosses an order-one scale; dashed curves show the exact projected counting formula, whose crossover approaches $\log2$ in the Poisson limit. Right: near the population optimum under label noise, the full-categorical Muon response scales empirically as $p_j^{0.494}$, close to the square-root prediction under the auxiliary covariance and frequency approximations, whereas the SGD coefficient scales linearly in $p_j$.}
		\label{fig:label-regimes-main}
	\end{figure}
	
	Appendix~\ref{app:label-noise} derives the complete categorical covariance and local Gaussian limit, then the projected counting dynamics and covariance closure used for the explicit regime formulas.
	
	\subsection{Sparse support: Muon versus SignSGD}
	\label{sec:signsgd-contrast}
	
	Without momentum, sparse mini-batch sampling also distinguishes polar normalization from coordinatewise sign normalization. Muon is orthogonally equivariant and preserves the missing-task nullspace; SignSGD is coordinate dependent, so a basis rotation can turn an update for an observed task into an update of an unobserved task.
	
	Consider momentum-free updates in the two-task model with query embedding $E=\bigl[\begin{smallmatrix}\cos\theta&\sin\theta\\-\sin\theta&\cos\theta\end{smallmatrix}\bigr]$, $0<\theta<\pi/2$, and let $\widetilde c=\|\widetilde E(1,-1)^\top\|_1$.
	
		\begin{proposition}[Muon avoids rotation-induced updates to a missing task]
			\label{prop:signsgd-contrast}
			A batch with no task-2 samples gives $\Delta_2^{\rm Muon}(t+1)-\Delta_2^{\rm Muon}(t)=0$, whereas $\Delta_2^{\rm Sign}(t+1)-\Delta_2^{\rm Sign}(t)=\eta\widetilde c(\cos\theta-\sin\theta)$.
			If $p_1^B>(1+\cot\theta)/2$, the rare-task SignSGD margin decays linearly almost surely.
		\end{proposition}
	For $\theta>\pi/4$, every frequent-only batch therefore decreases the rare-task SignSGD margin. The contrast is especially clear at $B=1$. Muon updates only the sampled task: the rare margin still grows, but on its slow occurrence clock. After rotation, the same stream of frequent-task observations produces a negative component in the rare SignSGD margin. Figure~\ref{fig:signsgd-main} shows the long-time consequence for $p_2=0.05$ and $\theta=1.15$ radians.
	
	\begin{figure}[t]
		\centering
		\includegraphics[width=0.98\linewidth]{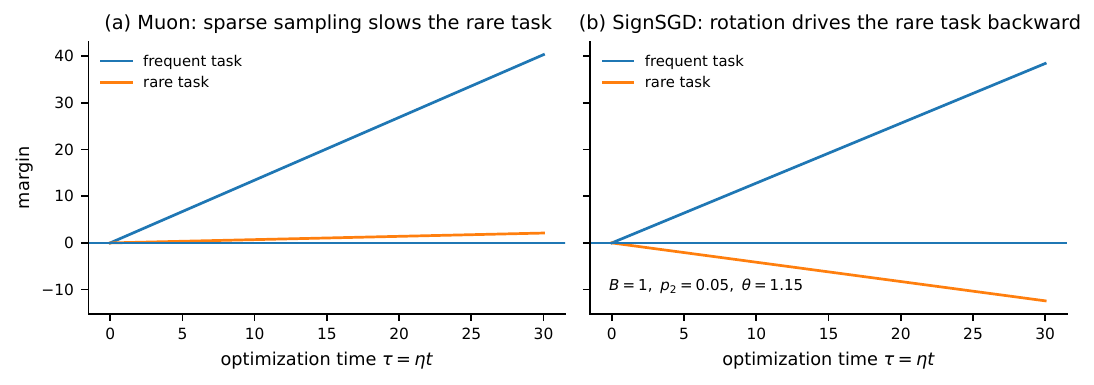}
		\caption{\textbf{Sparse support yields slow positive tail learning for Muon but negative transfer for SignSGD.} Both panels use the same $B=1$, $p_1=0.95$, $p_2=0.05$ mini-batch streams over 64 runs. Left: Muon preserves support, so the rare-task margin remains nonnegative and grows on its occurrence clock. Right: at $\theta=1.15>\pi/4$, rotated SignSGD drives the rare-task margin downward while the frequent task continues to improve. The positive answer-embedding scalar $\widetilde c$ is absorbed into the SignSGD learning rate. Curves show medians and ribbons show interquartile ranges.}
		\label{fig:signsgd-main}
	\end{figure}
	
	Appendix~\ref{app:two-task-signsgd} gives the exact rotated recursion and the long-time negative-transfer proof.
	
	\section{Momentum: Averaging and Tracking}
	\label{sec:direction-aware-momentum}
	
	Momentum changes the sampling regime seen by the polar map by averaging stochastic gradients across optimization steps.
	
		\begin{proposition}[Averaging scale and tracking under momentum]
			\label{prop:temporal-aggregation}
			\textbf{(i) Averaging.} The stationary EMA task frequency has variance-equivalent batch size $B_{\rm eff}(\beta)=B(1+\beta)/(1-\beta)$. Relative frequency fluctuations vanish as $B_{\rm eff}p_j\to\infty$.
			
			\textbf{(ii) Tracking.} If $\|S_t-S_{t-r}\|_F\le L_S\eta r$, the frozen-state buffer error is at most $L_S\eta\beta/(1-\beta)$ and vanishes when $\eta/(1-\beta)\to0$.
		\end{proposition}
	At a frozen state in the deterministic-label model, the stationary EMA $M_t=(1-\beta)\sum_{r\ge0}\beta^rG_{t-r}$ can be expressed using exponentially averaged task frequencies, reducing their sampling variance without changing the physical mini-batch size.
	
	At a frozen state, replacing $B$ by $B_{\rm eff}$ reproduces task-frequency variances, but not task-absence probabilities or the full batch distribution. For $\beta$ close to one, $B_{\rm eff}\approx2B/(1-\beta)$, so $1-\beta\ll Bp_j$ suppresses relative frequency fluctuations. EMA updates are also temporally correlated: reduced instantaneous variance need not reduce long-time diffusion (Appendix~\ref{app:two-task-momentum}).
	
	Under the local tracking assumptions, combining averaging with a small frozen-state buffer error gives the sufficient window $\eta\ll1-\beta\ll Bp_j$. Transferring a buffer-error bound to a polar-direction bound additionally requires local control of the nonzero singular values or of relative perturbations. The window is therefore a local design guide, not a global convergence guarantee. It motivates longer memory for rare, slowly varying directions and shorter memory where rapid tracking is needed. Under stochastic labels, averaging also reduces persistent label fluctuations.
	
	\begin{figure}[t]
		\centering
		\includegraphics[width=0.98\linewidth]{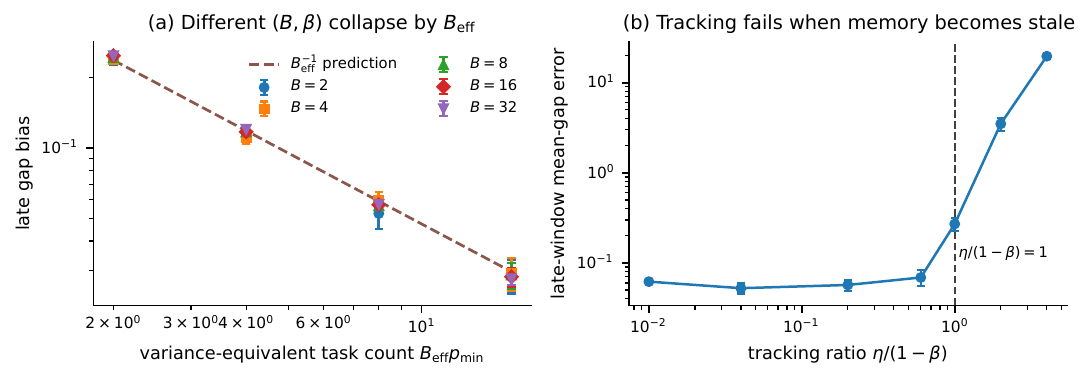}
		\caption{\textbf{Momentum trades variance reduction against tracking.} Left: different physical batch sizes and momentum coefficients nearly collapse when the two-task gap bias is plotted against the variance-equivalent task count $B_{\rm eff}p_{\min}$. Right: with $B$ and $\beta$ fixed and a common horizon $\tau_{\max}=50$, the late-window mean-gap error stays small while $\eta/(1-\beta)$ is small and grows once the EMA horizon becomes comparable to the state-evolution time scale.}
		\label{fig:momentum-main}
	\end{figure}
	
	Appendix~\ref{app:block-momentum} derives the EMA variance and tracking decomposition, and Appendix~\ref{app:noise-lag} gives a complementary scalar bias--variance calculation.
	
	\section{Fast--Slow Momentum in Language-Model Pretraining}
	\label{sec:fast-slow-pretraining}
	
	\paragraph{From theory to practice.}
	Our analysis motivates longer memory for rare, slowly varying directions and shorter memory where rapid tracking is needed. To examine this principle in language-model pretraining, we use the curvature-conditioned fast--slow momentum construction of \citet{zhu2026multiscale}. It retains the fast Muon update and adds a slow-momentum contribution in estimated flat directions, identified through the lower singular subspaces of the fast momentum. This provides a practical counterpart to the averaging--tracking trade-off in our theory, without requiring explicit task-frequency estimates.
	
	\begin{figure}[htbp]
		\centering
		\begin{minipage}[c]{0.49\linewidth}
			\begin{tcolorbox}[
				title={Fast--slow Muon (schematic)},
				colback=white,colbacktitle=black!5,coltitle=black,
				colframe=black!50,fonttitle=\small\bfseries,
				boxrule=0.4pt,arc=0pt,
				left=4pt,right=4pt,top=4pt,bottom=4pt]
				\footnotesize
				Decay rates $0\leq\beta_{\rm f}<\beta_{\rm s}<1$; mixing weight $\lambda>0$; step size $\eta_t$.
				Initialize $M^{\rm f}=M^{\rm s}=0$.
				\begin{algorithmic}[1]
					\For{each step $t$}
					\State Compute gradient $G_t$
					\State $M^{\rm f}\gets\beta_{\rm f}M^{\rm f}+G_t$
					\State $M^{\rm s}\gets\beta_{\rm s}M^{\rm s}+G_t$
					\State $H_t\gets\beta_{\rm f}M^{\rm f}+G_t$
					\State Estimate flat projector $\mathcal Q_t$ from $H_t$
					\State $D^{\rm f}_t\gets\operatorname{msgn}(H_t)$
					\State $D^{\rm s}_t\gets\mathcal Q_t(\operatorname{RowScale}(M^{\rm s}))$
					\State $W_{t+1}\gets W_t-\eta_t(D^{\rm f}_t+\lambda D^{\rm s}_t)$
					\EndFor
				\end{algorithmic}
			\end{tcolorbox}
		\end{minipage}\hfill
		\begin{minipage}[c]{0.49\linewidth}
			\includegraphics[width=\linewidth]{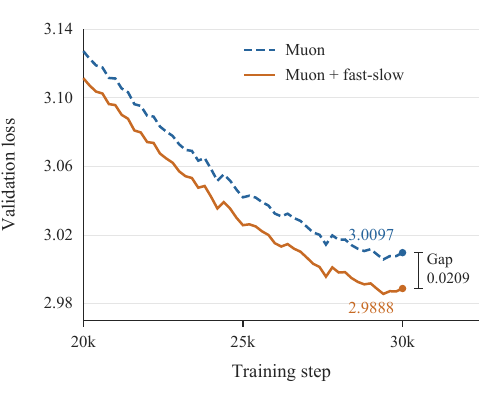}
		\end{minipage}
		\caption{\textbf{Fast--slow momentum: core update and pretraining results.}
			Left: schematic momentum mixing following \citet{zhu2026multiscale};
			$\operatorname{RowScale}$ normalizes rows, and parameter and update scaling
			follow the cited implementation.
			Right: validation loss from steps 20k to 30k for a 0.12B dense language model,
			with a final reduction of $0.0209$.
			Other training settings are fixed; curves are unsmoothed, with one run per method.}
		\label{fig:fast-slow-pretraining}
	\end{figure}
	
	\paragraph{Pretraining results.}
	We compare standard Muon with this fast--slow variant on a 0.12B-parameter dense language model trained for 30,000 steps, keeping the remaining training configuration fixed. Across all 138 shared checkpoints from step 2,600 to 30,000, the fast--slow variant has lower validation loss. Figure~\ref{fig:fast-slow-pretraining} shows the final 10,000 steps. At step 30,000, validation loss decreases from $3.0097$ to $2.9888$, a reduction of $0.0209$. These single-run results illustrate the practical relevance of heterogeneous momentum; variability across seeds and the effects of individual components remain to be evaluated. Appendix~\ref{app:llm-trajectory} documents the trajectory comparison.
	
	\section{Conclusion}
	\label{sec:conclusion}
	
	We study when Muon's preconditioning benefits survive stochastic gradient noise. Under isotropic low-SNR noise, its first-order conditional mean update matches learning-rate-rescaled SGD, showing that samplewise spectral normalization alone does not guarantee directional preconditioning. Structured sampling yields richer behavior: we establish an exact two-task balancing threshold and analyze higher-rank sampling effects governed by $Bp_j$. Label noise further separates task coverage from gradient reliability; an isotropic covariance closure with a locally homogeneous frequency approximation predicts square-root frequency dependence near the noisy optimum. A separate momentum-free SignSGD comparison highlights Muon's protection against rotation-induced negative transfer under sparse sampling.
	
	Our momentum analysis connects frozen-state variance reduction with local tracking error, motivating longer memory for rare, slowly varying directions. This provides a sampling-based interpretation of heterogeneous momentum designs, supported by a preliminary 0.12B pretraining experiment in which fast--slow momentum lowers validation loss relative to Muon. Extending the analysis to evolving, non-orthogonal representations and validating these benefits across training scales remain important directions.
	
	\subsection*{AI use statement}
	
	In this work, we used generative AI tools to assist with technical aspects of some mathematical proofs and derivations, and to assist with translation of the manuscript. We did not use generative AI tools to formulate the research questions, conceptual framework, or theorem statements. Additionally, we used generative AI tools to generate and edit code for controlled synthetic numerical experiments and figure generation. We reviewed all AI-assisted work. All AI-assisted mathematical arguments were checked and verified by the authors; AI-generated numerical code was inspected and tested for correctness, and all translated or edited text and generated figures were reviewed by the authors. We take responsibility for the final content of this work, including text, claims, code-derived results, and other artifacts produced with the aid of generative AI.
	
	\bibliography{iclr2027_conference}
	\bibliographystyle{plainnat}
	
	\clearpage
\appendix
\part{Appendix}
\parttoc
\clearpage
	
	\noindent\textbf{Guide to the appendices.}
	Appendix~A proves the isotropic Gaussian response result. Appendix~B records the associative-memory setup used throughout the structured analysis. Appendix~C gives the exact two-task dynamics and their large-batch/momentum refinements. Appendix~D supports the higher-rank regime map through sparse-task gating, full-batch contrast dynamics, well-sampled perturbation analysis, momentum averaging, and a final technical continuum refinement. Appendix~E separates exact categorical label-noise results from the projected counting model and the additional covariance closure. Appendix~F proves the Muon--SignSGD sparse-support contrast. Appendix~G contains an optional scalar EMA bias--variance calculation, and Appendix~H collects numerical protocols and supplementary diagnostics.
	
	\section{Isotropic Gaussian Mean Response: Proof of Proposition~\ref{prop:gaussian-isotropization}}
	\label{app:isotropic-proof}
	
	Throughout this appendix, we condition on the current iterate and suppress the time index. We write
	\begin{equation*}
		G = \bar G + \sigma \Xi,
		\qquad
		\bar G = \nabla \mathcal L(W),
	\end{equation*}
	where \(\Xi \in \mathbb{R}^{d\times d}\) has i.i.d.\ \(\mathcal N(0,1)\) entries and \(\sigma>0\).
	
	\subsection{Gaussian-smoothing proof}
	\label{app:isotropic-muon}
	
	Define
	\[
	F(H):=\mathbb E[\operatorname{msgn}(\Xi+H)],
	\qquad H\in\mathbb R^{d\times d},
	\]
	and write
	\[
	\langle A,B\rangle_F:=\operatorname{tr}(A^\top B).
	\]
	
	\paragraph{Smoothness by Gaussian convolution.}
	Let
	\[
	\phi(A)=(2\pi)^{-d^2/2}
	\exp\!\left(-\frac12\|A\|_F^2\right)
	\]
	be the standard Gaussian density on \(\mathbb R^{d\times d}\).
	Changing variables yields
	\begin{align*}
		F(H)
		& =\int\operatorname{msgn}(A)\phi(A-H)\,\mathrm dA \\
		& =\mathbb E\!\left[
		\operatorname{msgn}(\Xi)
		\exp\!\left(\langle H,\Xi\rangle_F
		-\frac12\|H\|_F^2\right)
		\right].
	\end{align*}
	For \(\|H\|_F\le r\), derivatives of any fixed order of the exponential factor are bounded by a constant, depending only on the derivative order and \(r\), times a polynomial in \(\|\Xi\|_F\) multiplied by \(e^{r\|\Xi\|_F}\). Since \(\|\operatorname{msgn}(\Xi)\|_F\le\sqrt d\) and
	\[
	\mathbb E\!\left[(1+\|\Xi\|_F^k)e^{r\|\Xi\|_F}\right]<\infty
	\]
	for every fixed \(k\ge0\) and \(r>0\), differentiation under the expectation is justified locally to every order. The same domination also implies continuity of the resulting derivatives. Hence \(F\) is smooth, and
	\begin{equation}
		\label{eq:gaussian-first-derivative}
		DF(0)[X] =\mathbb E\!\left[
		\operatorname{msgn}(\Xi)\langle X,\Xi\rangle_F
		\right]
	\end{equation}
	for every \(X\in\mathbb R^{d\times d}\). This argument differentiates the shifted Gaussian density, not the nonsmooth map \(\operatorname{msgn}\).
	
	\paragraph{Orthogonal equivariance of the derivative.}
	We use the following elementary lemma.
	
	\begin{lemma}
		\label{lem:orth-equivariant-map}
		Let \(T:\mathbb R^{d\times d}\to\mathbb R^{d\times d}\) be linear and satisfy
		\[
		T(UXV^\top)=UT(X)V^\top
		\]
		for every \(X\in\mathbb R^{d\times d}\) and every \(U,V\in O(d)\).
		Then \(T(X)=cX\) for some scalar \(c\in\mathbb R\) and all \(X\).
	\end{lemma}
	
	\begin{proof}
		Fix \(i,j\), and let \(D_1,D_2\) be arbitrary diagonal sign matrices satisfying \((D_1)_{ii}=(D_2)_{jj}=1\). Then
		\[
		D_1E_{ij}D_2^\top=E_{ij}.
		\]
		Hence equivariance gives
		\[
		T(E_{ij})=D_1T(E_{ij})D_2^\top.
		\]
		Varying all remaining diagonal signs independently forces every entry of \(T(E_{ij})\) except entry \((i,j)\) to vanish. Thus \(T(E_{ij})=c_{ij}E_{ij}\). Independent row and column permutations map any \(E_{ij}\) to any other matrix unit, so equivariance forces all \(c_{ij}\) to equal a common scalar \(c\). Linearity completes the proof.
	\end{proof}
	
	The identities
	\[
	\operatorname{msgn}(UAV^\top) =U\operatorname{msgn}(A)V^\top,
	\qquad
	U^\top\Xi V\stackrel{\mathrm d}{=}\Xi
	\]
	hold for all \(U,V\in O(d)\). Therefore
	\begin{align*}
		F(UHV^\top)
		& =U\,\mathbb E\!\left[
		\operatorname{msgn}(U^\top\Xi V+H)
		\right]V^\top            \\
		& =UF(H)V^\top.
	\end{align*}
	Differentiating at zero gives
	\[
	DF(0)[UXV^\top]=UDF(0)[X]V^\top.
	\]
	By Lemma~\ref{lem:orth-equivariant-map}, there is a scalar \(c_d\) such that
	\[
	DF(0)[X]=c_dX.
	\]
	
	\paragraph{Identification of the coefficient.}
	Taking the trace of this linear operator in the orthonormal basis \(\{E_{ij}\}_{i,j=1}^d\) and using \eqref{eq:gaussian-first-derivative}, we obtain
	\begin{align*}
		d^2c_d
		& =\sum_{i,j}\langle E_{ij},DF(0)[E_{ij}]\rangle_F        \\
		& =\sum_{i,j}\mathbb E[
		\operatorname{msgn}(\Xi)_{ij}\Xi_{ij}]                     \\
		& =\mathbb E\langle\operatorname{msgn}(\Xi),\Xi\rangle_F.
	\end{align*}
	For an SVD \(\Xi=U\Sigma V^\top\),
	\[
	\langle\operatorname{msgn}(\Xi),\Xi\rangle_F =
	\operatorname{tr}(\Sigma) =
	\|\Xi\|_*.
	\]
	Consequently,
	\[
	c_d=\frac{1}{d^2}\mathbb E\|\Xi\|_*>0.
	\]
	The expectation is finite because \(\|\Xi\|_*\le\sqrt d\,\|\Xi\|_F\), and it is strictly positive because \(\Xi\) is nonzero almost surely. For \(d=1\), \(c_1=\mathbb E|\xi|=\sqrt{2/\pi}\), recovering the SignSGD coefficient.
	
	\paragraph{Taylor theorem with remainder.}
	Since \(\operatorname{msgn}(-A)=-\operatorname{msgn}(A)\) and \(\Xi\stackrel{\mathrm d}{=}-\Xi\), we have \(F(-H)=-F(H)\). Consequently,
	\[
	F(0)=0,
	\qquad
	D^2F(0)=0.
	\]
	Because \(F\) is smooth and the dimension is fixed,
	\[
	M_d
	:=
	\sup_{\|Z\|_F\le1}
	\|D^3F(Z)\|_{\mathrm{op}} <\infty,
	\]
	where \(\|\cdot\|_{\mathrm{op}}\) denotes the norm of a trilinear map induced by the Frobenius norm on its arguments and output. Taylor's theorem with integral remainder gives
	\[
	F(H)-DF(0)[H] =
	\frac12\int_0^1(1-s)^2
	D^3F(sH)[H,H,H]\,\mathrm ds.
	\]
	Hence, for every \(\|H\|_F\le1\),
	\[
	\|F(H)-c_dH\|_F
	\le
	\frac{M_d}{6}\|H\|_F^3.
	\]
	Finally, since \(\sigma>0\) and the matrix sign is invariant under positive rescaling,
	\[
	\operatorname{msgn}(\bar G+\sigma\Xi) =
	\operatorname{msgn}\!\left(
	\Xi+\frac{\bar G}{\sigma}
	\right).
	\]
	Taking \(H=\bar G/\sigma\), for \(\|\bar G\|_F/\sigma\le1\) we obtain
	\[
	\left\|
	\mathbb E[\operatorname{msgn}(\bar G+\sigma\Xi)]
	-\frac{c_d}{\sigma}\bar G
	\right\|_F
	\le
	\frac{M_d}{6}
	\frac{\|\bar G\|_F^3}{\sigma^3}.
	\]
	This proves the claimed \(O_d(\|\bar G\|_F^3/\sigma^3)\) remainder as \(\|\bar G\|_F/\sigma\to0\).
	
	\section{Associative-Memory Setup and Gradient Derivation}
	\label{app:associative-gradient}
	
	\subsection{Gradient of the Associative Memory Model}
	
	Recall that
	\[
	\overline W=\widetilde E^\top W E,
	\qquad
	S_{ij} =
	\frac{\exp(\overline W_{ij})}
	{\sum_{k=1}^K\exp(\overline W_{kj})}.
	\]
	For task \(j\), the cross-entropy loss is
	\[
	\ell_j(W) =
	-\log S_{jj} =
	-\widetilde E_j^\top W E_j
	+
	\log\!\left(
	\sum_{k=1}^K
	e^{\widetilde E_k^\top W E_j}
	\right).
	\]
	Differentiating gives
	\begin{align*}
		\nabla_W\ell_j(W)
		& =
		-\widetilde E_jE_j^\top
		+
		\sum_{k=1}^K
		S_{kj}\widetilde E_kE_j^\top \\
		& =
		-\left(
		\widetilde E_j
		-
		\sum_{k=1}^K
		S_{kj}\widetilde E_k
		\right)E_j^\top.
	\end{align*}
	Hence, with \(P=\operatorname{diag}(p_1,\ldots,p_K)\),
	\begin{align*}
		\nabla\mathcal L(W)
		& =
		-\sum_{j=1}^K
		\sum_{k=1}^K
		p_j(\delta_{kj}-S_{kj})
		\widetilde E_kE_j^\top \\
		& =
		-\widetilde E(I-S)PE^\top.
	\end{align*}
	Similarly, for a mini-batch with empirical frequencies
	\(q(t)=(q_1(t),\ldots,q_K(t))\) and
	\(Q_t=\operatorname{diag}(q(t))\),
	\[
	\nabla\mathcal L(W_t;\mathcal B_t) =
	-\widetilde E(I-S_t)Q_tE^\top.
	\]
	Since the partial polar factor is equivariant under left and right
	orthogonal transformations,
	\[
	\operatorname{msgn}\!\left(
	\widetilde EAE^\top
	\right) =
	\widetilde E\operatorname{msgn}(A)E^\top,
	\]
	the task-coordinate Muon update is
	\[
	\overline W_{t+1} =
	\overline W_t
	+
	\eta\operatorname{msgn}\!\left(
	(I-S_t)Q_t
	\right).
	\]
	
	\section{Two-Task Finite-Batch Dynamics and Refinements}
	\label{app:two-task-proof}
	
	The two-task model is retained as an exactly solvable specialization of the finite-batch sampling mechanism. Because its stochastic gradient is rank one, it should not be interpreted as evidence for genuinely higher-rank spectral effects. Throughout this appendix, we assume \(K=2\) and \(p_1\ge p_2=p_{\min}>0\). Each batch consists of \(B\) independent samples from \((p_1,p_2)\), independently across iterations. Subsections~C.1--C.4 establish the exact finite-\(B\) statements used in Theorem~\ref{thm:two-task-coverage}; the final subsection records large-batch and momentum-averaging refinements rather than a second phase-transition claim. Thus
	\[
	Bq_1(t)\sim\operatorname{Binom}(B,p_1),
	\qquad
	q_2(t)=1-q_1(t).
	\]
	
	\subsection{Exact Reduction to the Margin Dynamics}
	\label{app:two-task-gradient}
	
	Define
	\[
	\Delta_1(t) =
	(\overline W_t)_{11}-(\overline W_t)_{21},
	\qquad
	\Delta_2(t) =
	(\overline W_t)_{22}-(\overline W_t)_{12},
	\]
	and
	\[
	u_i(t)=\frac{1}{1+e^{\Delta_i(t)}}.
	\]
	For two classes,
	\[
	1-S_{11}(t)=S_{21}(t)=u_1(t),
	\qquad
	1-S_{22}(t)=S_{12}(t)=u_2(t).
	\]
	Therefore
	\begin{align*}
		(I-S_t)Q_t
		& =
		\begin{bmatrix}
			q_1(t)(1-S_{11}(t)) & -q_2(t)S_{12}(t) \\-q_1(t)S_{21}(t)&q_2(t)(1-S_{22}(t))
		\end{bmatrix}
		\\
		& =
		\begin{bmatrix}
			1 \\-1
		\end{bmatrix}
		\begin{bmatrix}
			q_1(t)u_1(t) & -q_2(t)u_2(t)
		\end{bmatrix}.
	\end{align*}
	This matrix has rank one whenever at least one \(q_i(t)\) is nonzero, which always holds. For a nonzero rank-one matrix \(ab^\top\),
	\[
	\operatorname{msgn}(ab^\top) =
	\frac{a}{\|a\|_2}
	\frac{b^\top}{\|b\|_2}.
	\]
	Since \(\|(1,-1)^\top\|_2=\sqrt{2}\), we obtain
	\[
	\operatorname{msgn}\!\left((I-S_t)Q_t\right) =
	\begin{bmatrix}
		1/\sqrt{2} \\
		-1/\sqrt{2}
	\end{bmatrix}
	\frac{
		\begin{bmatrix}
			q_1(t)u_1(t) & -q_2(t)u_2(t)
		\end{bmatrix}
	}{
		\sqrt{q_1^2(t)u_1^2(t)+q_2^2(t)u_2^2(t)}
	}.
	\]
	Taking the corresponding column differences gives
	\begin{equation}
		\Delta_i(t+1)-\Delta_i(t) =
		\sqrt{2}\eta
		\frac{q_i(t)u_i(t)}
		{\sqrt{q_1^2(t)u_1^2(t)+q_2^2(t)u_2^2(t)}}.
		\label{eq:app-two-task-margins}
	\end{equation}
	
	\subsection{Full-batch Dynamics}
	
	Set \(q_i(t)=p_i\) and take the continuous-time limit
	\(\tau=\eta t\). Then
	\[
	\dot\Delta_i =
	\sqrt{2}\,
	\frac{p_i u_i}
	{\sqrt{p_1^2u_1^2+p_2^2u_2^2}}.
	\]
	Dividing the two equations gives
	\[
	\frac{d\Delta_1}{d\Delta_2} =
	\frac{p_1u_1}{p_2u_2} =
	\frac{p_1(1+e^{\Delta_2})}
	{p_2(1+e^{\Delta_1})}.
	\]
	Hence
	\[
	p_2(1+e^{\Delta_1})\,d\Delta_1 =
	p_1(1+e^{\Delta_2})\,d\Delta_2.
	\]
	Using \(\Delta_1(0)=\Delta_2(0)=0\),
	\begin{equation}
		p_2(\Delta_1+e^{\Delta_1}-1) =
		p_1(\Delta_2+e^{\Delta_2}-1).
		\label{eq:app-fullbatch-invariant}
	\end{equation}
	Moreover,
	\[
	\sqrt{2}
	\le
	\frac{d}{d\tau}(\Delta_1+\Delta_2) =
	\sqrt{2}\,
	\frac{p_1u_1+p_2u_2}
	{\sqrt{p_1^2u_1^2+p_2^2u_2^2}}
	\le2.
	\]
	Thus \(\Delta_1+\Delta_2\to\infty\). Therefore, \eqref{eq:app-fullbatch-invariant} implies that both margins diverge, and the exponential terms dominate:
	\[
	p_2e^{\Delta_1}
	\sim
	p_1e^{\Delta_2}.
	\]
	Therefore
	\[
	\Delta_1-\Delta_2
	\longrightarrow
	\log\frac{p_1}{p_2},
	\qquad
	p_1u_1\sim p_2u_2.
	\]
	Substituting back into the ODE yields
	\[
	\dot\Delta_i\longrightarrow1,
	\qquad
	\frac{\Delta_i(\tau)}{\tau}\longrightarrow1.
	\]
	
	\subsection{Small Mini-batches: Occurrence-limited Imbalance}
	
	Let
	\[
	A_t=\{q_1(t)=1\},
	\qquad
	\mathbb P(A_t)=p_1^B.
	\]
	On \(A_t\),
	\[
	\Delta_1(t+1)-\Delta_1(t)=\sqrt{2}\eta,
	\qquad
	\Delta_2(t+1)-\Delta_2(t)=0.
	\]
	On \(A_t^c\), both increments are nonnegative and their difference is at
	least \(-\sqrt{2}\eta\). Hence
	\[
	D(t+1)-D(t)
	\ge
	\sqrt{2}\eta\bigl(2\mathbf{1}_{A_t}-1\bigr),
	\qquad
	D(t):=\Delta_1(t)-\Delta_2(t).
	\]
	By the strong law of large numbers,
	\[
	\liminf_{t\to\infty}\frac{D(t)}{\eta t}
	\ge
	\sqrt{2}(2p_1^B-1)
	\qquad\text{a.s.}
	\]
	Thus \(D(t)\to+\infty\) when \(p_1^B>1/2\).
	
	We can then identify the exact asymptotic rates. Since
	\[
	\frac{u_1(t)}{u_2(t)} =
	\frac{1+e^{\Delta_2(t)}}{1+e^{\Delta_1(t)}}
	\longrightarrow0,
	\]
	every batch containing task-2 asymptotically allocates the normalized
	update to task-2, while a pure task-1 batch allocates it to task-1.
	Consequently,
	\[
	\frac{\Delta_1(t)}{\eta t}
	\xrightarrow{\mathrm{a.s.}}
	\sqrt{2}p_1^B,
	\qquad
	\frac{\Delta_2(t)}{\eta t}
	\xrightarrow{\mathrm{a.s.}}
	\sqrt{2}(1-p_1^B).
	\]
	The transition \(p_1^B=1/2\) is equivalent to
	\[
	B_{\mathrm{bal}} =
	\frac{\log2}{-\log p_1} =
	\frac{\log2}{-\log(1-p_{\min})}.
	\]
	
	\subsection{Large Mini-batches: Recovery of Balanced Growth}
	
	Assume
	\[
	p_1^B<\frac{1}{2}.
	\]
	We first establish an exact stability result for the discrete-time gap
	\[
	D(t):=\Delta_1(t)-\Delta_2(t).
	\]
	
	\paragraph{Restoring drift in the tails.}
	
	From \eqref{eq:app-two-task-margins},
	\begin{equation}
		D(t+1)-D(t) =
		\sqrt{2}\eta
		\frac{
			q_1(t)u_1(t)-q_2(t)u_2(t)
		}{
			\sqrt{
				q_1^2(t)u_1^2(t)+q_2^2(t)u_2^2(t)
			}
		}.
		\label{eq:app-gap-increment}
	\end{equation}
	In particular,
	\[
	|D(t+1)-D(t)|
	\le
	\sqrt{2}\eta.
	\]
	
	Since both margins start from zero and have nonnegative increments,
	\[
	\Delta_i(t)\ge 0,
	\qquad
	\forall t.
	\]
	If \(D(t)\ge 0\), then
	\begin{align*}
		\frac{u_1(t)}{u_2(t)}
		& =
		\frac{1+e^{\Delta_2(t)}}{1+e^{\Delta_1(t)}}      \\
		& =
		\frac{1+e^{\Delta_2(t)}}{1+e^{\Delta_2(t)+D(t)}} \\
		& \le
		2e^{-D(t)}.
	\end{align*}
	Thus \(u_1/u_2\to 0\) uniformly over the common offset of the two margins as \(D\to+\infty\).
	
	If \(q_2(t)=0\), then \(q_1(t)=1\) and \(D(t+1)-D(t)=\sqrt{2}\eta\).
	
	If \(q_2(t)>0\), then \(q_2(t)\ge 1/B\). Since the batch proportions take values in a finite set, \eqref{eq:app-gap-increment} and the bound above imply, uniformly over all batches with \(q_2(t)>0\),
	\[
	D(t+1)-D(t)
	\longrightarrow
	-\sqrt{2}\eta
	\qquad
	\text{as }D(t)\to+\infty.
	\]
	Consequently,
	\begin{equation}
		\mathbb E\!\left[
		D(t+1)-D(t)
		\mid\mathcal{F}_t
		\right]
		\longrightarrow
		\sqrt{2}\eta\left(2p_1^B-1\right) <0
		\qquad
		\text{as }D(t)\to+\infty.
		\label{eq:app-positive-tail-drift}
	\end{equation}
	
	Similarly, if \(D(t)\le 0\), then
	\[
	\frac{u_2(t)}{u_1(t)}
	\le
	2e^{D(t)} =
	2e^{-|D(t)|}.
	\]
	Hence, as \(D(t)\to-\infty\), every batch containing task-1 satisfies
	\[
	D(t+1)-D(t)
	\longrightarrow
	\sqrt{2}\eta,
	\]
	whereas a pure task-2 batch satisfies
	\[
	D(t+1)-D(t) =
	-\sqrt{2}\eta.
	\]
	Therefore,
	\begin{equation}
		\mathbb E\!\left[
		D(t+1)-D(t)
		\mid\mathcal{F}_t
		\right]
		\longrightarrow
		\sqrt{2}\eta\left(1-2p_2^B\right) >0
		\qquad
		\text{as }D(t)\to-\infty,
		\label{eq:app-negative-tail-drift}
	\end{equation}
	where we used \(p_2^B\le p_1^B<1/2\).
	
	It follows that there exist constants \(R,\delta>0\) such that
	\[
	\mathbb E\!\left[
	D(t+1)-D(t)
	\mid\mathcal{F}_t
	\right]
	\le -\delta
	\qquad
	\text{when }D(t)\ge R,
	\]
	and
	\[
	\mathbb E\!\left[
	D(t+1)-D(t)
	\mid\mathcal{F}_t
	\right]
	\ge \delta
	\qquad
	\text{when }D(t)\le -R.
	\]
	
	\paragraph{Exponential moment bound.}
	
	Let
	\[
	X_t:=D(t+1)-D(t),
	\qquad
	L:=\sqrt{2}\eta,
	\]
	so that \(|X_t|\le L\). Increasing \(R\) if necessary, assume \(R>L\).
	
	When \(D(t)\ge R\), the sign of \(D(t)\) cannot change in one step, and hence
	\[
	|D(t+1)|=D(t)+X_t.
	\]
	For any \(\theta>0\),
	\[
	e^{\theta X_t}
	\le
	1+\theta X_t
	+
	\frac{\theta^2L^2}{2}e^{\theta L}.
	\]
	Therefore, for sufficiently small \(\theta>0\),
	\[
	\mathbb E\!\left[
	e^{\theta X_t}
	\mid\mathcal{F}_t
	\right]
	\le
	1-\frac{\theta\delta}{2}
	\qquad
	\text{when }D(t)\ge R.
	\]
	The same argument applied to \(-X_t\) gives
	\[
	\mathbb E\!\left[
	e^{-\theta X_t}
	\mid\mathcal{F}_t
	\right]
	\le
	1-\frac{\theta\delta}{2}
	\qquad
	\text{when }D(t)\le -R.
	\]
	
	Define
	\[
	V(D):=e^{\theta|D|}.
	\]
	There then exist constants \(\rho<1\) and \(b_0<\infty\) such that
	\[
	\mathbb E\!\left[
	V(D(t+1))
	\mid\mathcal{F}_t
	\right]
	\le
	\rho V(D(t))
	+
	b_0\mathbf{1}_{\{|D(t)|\le R\}}.
	\]
	Iterating this inequality yields
	\begin{equation}
		\sup_{t\ge 0}
		\mathbb Ee^{\theta|D(t)|} <\infty.
		\label{eq:app-gap-exponential-moment}
	\end{equation}
	
	Hence, for every \(\varepsilon>0\),
	\[
	\mathbb P\!\left(
	|D(t)|>\varepsilon t
	\right)
	\le
	b_\theta e^{-\theta\varepsilon t},
	\]
	and therefore
	\[
	\sum_{t=1}^{\infty}
	\mathbb P\!\left(
	|D(t)|>\varepsilon t
	\right) <\infty.
	\]
	By the Borel--Cantelli lemma,
	\begin{equation}
		\frac{D(t)}{t}
		\longrightarrow
		0
		\qquad
		\text{a.s.}
		\label{eq:app-balanced-gap}
	\end{equation}
	This proves vanishing asymptotic gap growth whenever \(p_1^B<1/2\).
	
	\paragraph{Late-time effective drift.}
	
	We next refine the description of the gap in the well-sampled regime. The remainder of this subsection concerns a late-time, small-step approximation rather than an exact invariant law of the discrete process.
	
	First note that both margins diverge almost surely. A pure task-\(i\) batch occurs independently with probability \(p_i^B>0\), and on every such batch,
	\[
	\Delta_i(t+1)-\Delta_i(t) =
	\sqrt{2}\eta.
	\]
	The strong law therefore implies
	\[
	\liminf_{t\to\infty}
	\frac{\Delta_i(t)}{\eta t}
	\ge
	\sqrt{2}p_i^B >0
	\qquad
	\text{a.s.}
	\]
	Consequently,
	\[
	u_i(t) =
	e^{-\Delta_i(t)}\left(1+o(1)\right).
	\]
	
	Define the centered gap
	\[
	\gamma(t)
	:=
	D(t)-\log\frac{p_1}{p_2}.
	\]
	Then
	\[
	\frac{p_1u_1(t)}{p_2u_2(t)} =
	e^{-\gamma(t)}\left(1+o(1)\right).
	\]
	
	For
	\[
	a=\frac{q_1}{p_1},
	\qquad
	b=\frac{q_2}{p_2},
	\]
	define
	\begin{equation}
		\varphi(\gamma;a,b)
		:=
		\sqrt{2}
		\frac{
			ae^{-\gamma}-b
		}{
			\sqrt{
				a^2e^{-2\gamma}+b^2
			}
		}.
		\label{eq:app-gap-drift-function}
	\end{equation}
	Then, uniformly for \(\gamma\) in compact sets,
	\begin{equation}
		\frac{\gamma(t+1)-\gamma(t)}{\eta} =
		\varphi\!\left(
		\gamma(t);
		\frac{q_1(t)}{p_1},
		\frac{q_2(t)}{p_2}
		\right)
		+
		o(1)
		\qquad
		\text{as }t\to\infty.
		\label{eq:app-late-gap-recursion}
	\end{equation}
	
	Define the corresponding late-time mean drift
	\begin{equation}
		\psi_B(\gamma)
		:=
		\mathbb E\!\left[
		\varphi\!\left(
		\gamma;
		\frac{q_1}{p_1},
		\frac{q_2}{p_2}
		\right)
		\right].
		\label{eq:app-late-mean-drift}
	\end{equation}
	
	For \(a,b>0\),
	\[
	\partial_\gamma\varphi(\gamma;a,b) =
	-\sqrt{2}
	\frac{
		ab e^{-\gamma}\left(ae^{-\gamma}+b\right)
	}{
		\left(a^2e^{-2\gamma}+b^2\right)^{3/2}
	} <0.
	\]
	If \(b=0\), then \(\varphi(\gamma;a,0)=\sqrt{2}\), whereas if \(a=0\), \(\varphi(\gamma;0,b)=-\sqrt{2}\). Thus \(\varphi\) is nonincreasing in \(\gamma\) for every batch and strictly decreasing for every mixed batch. Since mixed batches occur with positive probability, \(\psi_B\) is strictly decreasing.
	
	Moreover,
	\[
	\lim_{\gamma\to+\infty}
	\psi_B(\gamma) =
	\sqrt{2}\left(2p_1^B-1\right) <0,
	\]
	and
	\[
	\lim_{\gamma\to-\infty}
	\psi_B(\gamma) =
	\sqrt{2}\left(1-2p_2^B\right) >0.
	\]
	Hence \(\psi_B\) has a unique zero, denoted by
	\[
	\gamma_B^*.
	\]
	
	Suppose now that
	\[
	Bp_{\min}\to\infty.
	\]
	Since
	\[
	\operatorname{Var}\!\left(
	\frac{q_i}{p_i}
	\right) =
	\frac{1-p_i}{Bp_i},
	\]
	we have
	\[
	\frac{q_i}{p_i}
	\xrightarrow{\mathbb P}
	1.
	\]
	Since \(|\varphi|\le\sqrt{2}\), convergence in probability together with uniform boundedness implies
	\[
	\psi_B(\gamma)
	\longrightarrow
	\psi_\infty(\gamma)
	:=
	\sqrt{2}
	\frac{
		e^{-\gamma}-1
	}{
		\sqrt{
			e^{-2\gamma}+1
		}
	}
	\]
	for every fixed \(\gamma\).
	
	The limiting drift is strictly decreasing and has the unique zero \(\gamma=0\). Fix any \(\delta>0\). Since
	\[
	\psi_\infty(-\delta)>0,
	\qquad
	\psi_\infty(\delta)<0,
	\]
	the same inequalities hold for \(\psi_B\) when \(Bp_{\min}\) is sufficiently large. By monotonicity,
	\[
	|\gamma_B^*|<\delta.
	\]
	Therefore,
	\begin{equation}
		\gamma_B^*
		\longrightarrow
		0
		\qquad
		\text{as }Bp_{\min}\to\infty.
		\label{eq:app-large-batch-equilibrium}
	\end{equation}
	
	For fixed \(p_1,p_2>0\), this convergence can be sharpened. Define
	\[
	h(q)
	:=
	\varphi\!\left(
	0;
	\frac{q}{p_1},
	\frac{1-q}{p_2}
	\right).
	\]
	A direct calculation gives
	\[
	h(p_1)=0,
	\qquad
	h'(p_1)=\frac{1}{p_1p_2},
	\qquad
	h''(p_1) =
	\frac{p_1-p_2}{p_1^2p_2^2}.
	\]
	Since
	\[
	\mathbb E[q_1-p_1]=0,
	\qquad
	\operatorname{Var}(q_1) =
	\frac{p_1p_2}{B},
	\]
	a second-order expansion gives
	\[
	\psi_B(0) =
	\frac{p_1-p_2}
	{2Bp_1p_2}
	+
	o\!\left(\frac1B\right).
	\]
	Furthermore,
	\[
	\psi_B'(0) =
	-1+o(1).
	\]
	Expanding the zero of \(\psi_B\) around zero therefore yields
	\begin{equation}
		\gamma_B^* =
		\frac{p_1-p_2}
		{2Bp_1p_2}
		+
		o\!\left(\frac1B\right).
		\label{eq:app-large-batch-bias}
	\end{equation}
	
	Thus mini-batch fluctuations induce a small positive shift of the effective gap when \(p_1>p_2\), and this shift vanishes in the large-batch limit.
	
	\paragraph{Local diffusion approximation.}
	
	Let
	\[
	\lambda_B
	:=
	-\psi_B'(\gamma_B^*) >0
	\]
	and
	\[
	v_B
	:=
	\operatorname{Var}\!\left[
	\varphi\!\left(
	\gamma_B^*;
	\frac{q_1}{p_1},
	\frac{q_2}{p_2}
	\right)
	\right].
	\]
	Writing
	\[
	x(t):=\gamma(t)-\gamma_B^*,
	\]
	the local recursion suggested by \eqref{eq:app-late-gap-recursion} is
	\[
	x(t+1)-x(t)
	\simeq
	-\eta\lambda_B x(t)
	+
	\eta\xi_t,
	\]
	where
	\[
	\mathbb E[\xi_t]\simeq 0,
	\qquad
	\operatorname{Var}(\xi_t)\simeq v_B.
	\]
	
	In the well-sampled regime,
	\[
	\lambda_B=1+o(1).
	\]
	Moreover, using
	\[
	h'(p_1)=\frac{1}{p_1p_2}
	\]
	and
	\[
	\operatorname{Var}(q_1)=\frac{p_1p_2}{B},
	\]
	the delta method gives
	\begin{equation}
		v_B =
		\frac{1+o(1)}
		{Bp_1p_2}.
		\label{eq:app-large-batch-noise}
	\end{equation}
	
	On the rescaled time
	\[
	\tau=\eta t,
	\]
	the corresponding local diffusion approximation is
	\begin{equation}
		dx
		\simeq
		-\lambda_Bx\,d\tau
		+
		\sqrt{\eta v_B}\,
		d\mathcal W_\tau.
		\label{eq:app-large-batch-diffusion}
	\end{equation}
	The associated Ornstein--Uhlenbeck stationary variance is
	\[
	\frac{\eta v_B}{2\lambda_B} =
	\frac{\eta}{2Bp_1p_2}
	\left(1+o(1)\right).
	\]
	Thus, within the late-time local diffusion approximation,
	\begin{equation}
		D =
		\log\frac{p_1}{p_2}
		+
		\gamma_B^*
		+
		O_{\mathbb P}\!\left(
		\sqrt{
			\frac{\eta}{Bp_1p_2}
		}
		\right).
		\label{eq:app-large-batch-concentration}
	\end{equation}
	Together with \eqref{eq:app-large-batch-equilibrium}, this recovers the full-batch gap in the well-sampled, small-step limit.
	
	\subsection{Momentum: Averaging Across Steps and Tracking}
	\label{app:two-task-momentum}
	
	We now consider
	\[
	g_t =
	\left(
	q_1(t)u_1(t),
	q_2(t)u_2(t)
	\right)^\top,
	\qquad
	m_t =
	\beta m_{t-1}
	+
	(1-\beta)g_t,
	\]
	with
	\[
	0<\beta<1,
	\qquad
	\epsilon:=1-\beta.
	\]
	The normalized margin update is
	\[
	\Delta(t+1)-\Delta(t) =
	\sqrt{2}\eta
	\frac{m_t}{\|m_t\|_2}.
	\]
	
	We first freeze the margins and characterize how momentum averages task frequencies across iterations. We then restore the slow evolution of the margins.
	
	\paragraph{Frozen margins and the EMA task frequency.}
	
	Fix \(u_1,u_2>0\). After initialization transients, the stationary momentum admits the representation
	\[
	m_t =
	(1-\beta)
	\sum_{k=0}^{\infty}
	\beta^k g_{t-k}.
	\]
	Define the exponentially averaged task-1 frequency
	\begin{equation}
		Z_t
		:=
		(1-\beta)
		\sum_{k=0}^{\infty}
		\beta^k q_1(t-k).
		\label{eq:app-ema-frequency}
	\end{equation}
	Since \(q_2(t)=1-q_1(t)\) and
	\[
	(1-\beta)
	\sum_{k=0}^{\infty}
	\beta^k =
	1,
	\]
	we obtain the exact frozen-margin representation
	\begin{equation}
		m_t =
		\begin{pmatrix}
			u_1Z_t \\
			u_2(1-Z_t)
		\end{pmatrix}.
		\label{eq:app-momentum-frequency-representation}
	\end{equation}
	
	The first two moments of \(Z_t\) are
	\begin{align}
		\mathbb E[Z_t]
		& =
		p_1,
		\label{eq:app-ema-frequency-mean} \\
		\operatorname{Var}(Z_t)
		& =
		(1-\beta)^2
		\sum_{k=0}^{\infty}
		\beta^{2k}
		\frac{p_1p_2}{B}
		\notag                            \\
		& =
		\frac{1-\beta}{1+\beta}
		\frac{p_1p_2}{B}.
		\label{eq:app-ema-frequency-var}
	\end{align}
	
	It is therefore natural to define the variance-equivalent batch size
	\begin{equation}
		B_{\mathrm{eff}}
		:=
		B\frac{1+\beta}{1-\beta},
		\label{eq:app-effective-batch-size}
	\end{equation}
	for which
	\[
	\operatorname{Var}(Z_t) =
	\frac{p_1p_2}{B_{\mathrm{eff}}}.
	\]
	
	Since both tasks have positive probability and \(\beta>0\),
	\[
	0<Z_t<1
	\qquad
	\text{a.s.}
	\]
	under the stationary representation. Thus EMA averaging assigns positive weight to both tasks almost surely even when individual tasks rarely appear in a mini-batch.
	
	\paragraph{Momentum-induced effective gap.}
	
	Define again
	\[
	\gamma =
	D-\log\frac{p_1}{p_2}.
	\]
	At late time,
	\[
	\frac{p_1u_1}{p_2u_2} =
	e^{-\gamma}\left(1+o(1)\right).
	\]
	Using \eqref{eq:app-momentum-frequency-representation}, the normalized frozen-buffer gap increment is therefore asymptotically
	\[
	\eta
	\varphi\!\left(
	\gamma;
	\frac{Z_t}{p_1},
	\frac{1-Z_t}{p_2}
	\right),
	\]
	where \(\varphi\) is defined in \eqref{eq:app-gap-drift-function}.
	
	Define the frozen-momentum mean drift
	\begin{equation}
		\psi_{B,\beta}(\gamma)
		:=
		\mathbb E\!\left[
		\varphi\!\left(
		\gamma;
		\frac{Z}{p_1},
		\frac{1-Z}{p_2}
		\right)
		\right],
		\label{eq:app-momentum-mean-drift}
	\end{equation}
	where \(Z\) has the stationary marginal distribution defined by \eqref{eq:app-ema-frequency}.
	
	Since \(0<Z<1\) almost surely,
	\[
	\partial_\gamma
	\varphi\!\left(
	\gamma;
	\frac{Z}{p_1},
	\frac{1-Z}{p_2}
	\right) <0
	\qquad
	\text{a.s.}
	\]
	Hence \(\psi_{B,\beta}\) is strictly decreasing. Moreover,
	\[
	\lim_{\gamma\to-\infty}
	\psi_{B,\beta}(\gamma) =
	\sqrt{2},
	\]
	whereas
	\[
	\lim_{\gamma\to+\infty}
	\psi_{B,\beta}(\gamma) =
	-\sqrt{2}.
	\]
	Therefore there exists a unique \(\gamma_{B,\beta}^*\) such that
	\[
	\psi_{B,\beta}(\gamma_{B,\beta}^*)=0.
	\]
	
	Unlike the single-batch mean drift, the stationary frozen-momentum drift has no analogue of the threshold \(p_1^B=1/2\): for every \(\beta>0\), both tasks receive positive weight almost surely in the stationary EMA, producing restoring drift in both tails. This is the basic mechanism by which momentum suppresses task-frequency fluctuations across steps.
	
	\paragraph{Recovery of the full-batch gap.}
	
	From \eqref{eq:app-ema-frequency-var},
	\[
	\operatorname{Var}\!\left(
	\frac{Z}{p_1}
	\right) =
	\frac{p_2}{B_{\mathrm{eff}}p_1},
	\]
	and
	\[
	\operatorname{Var}\!\left(
	\frac{1-Z}{p_2}
	\right) =
	\frac{p_1}{B_{\mathrm{eff}}p_2}.
	\]
	Hence, if
	\begin{equation}
		B_{\mathrm{eff}}p_{\min}
		\longrightarrow
		\infty,
		\label{eq:app-temporal-coverage-condition}
	\end{equation}
	then
	\[
	\frac{Z}{p_1}
	\xrightarrow{\mathbb P}
	1,
	\qquad
	\frac{1-Z}{p_2}
	\xrightarrow{\mathbb P}
	1.
	\]
	Since \(|\varphi|\le\sqrt{2}\),
	\[
	\psi_{B,\beta}(\gamma)
	\longrightarrow
	\psi_\infty(\gamma) =
	\sqrt{2}
	\frac{
		e^{-\gamma}-1
	}{
		\sqrt{
			e^{-2\gamma}+1
		}
	}
	\]
	for every fixed \(\gamma\).
	
	As before, monotonicity and uniqueness of the zero of \(\psi_\infty\) imply
	\begin{equation}
		\gamma_{B,\beta}^*
		\longrightarrow
		0
		\qquad
		\text{as }
		B_{\mathrm{eff}}p_{\min}\to\infty.
		\label{eq:app-momentum-equilibrium-limit}
	\end{equation}
	
	Since
	\[
	B_{\mathrm{eff}} =
	B\frac{1+\beta}{1-\beta},
	\]
	condition~(\ref{eq:app-temporal-coverage-condition}) is equivalent, up to the bounded factor \(1+\beta\), to
	\[
	\frac{1-\beta}{Bp_{\min}}
	\longrightarrow
	0.
	\]
	
	For fixed \(p_1,p_2>0\), the finite-sampling bias can be sharpened by a second-order expansion. Using the function \(h\) defined in the previous subsection,
	\[
	\mathbb E[h(Z)] =
	\frac{1}{2}
	h''(p_1)
	\operatorname{Var}(Z)
	+
	o\!\left(B_{\mathrm{eff}}^{-1}\right).
	\]
	Therefore,
	\[
	\psi_{B,\beta}(0) =
	\frac{p_1-p_2}
	{2B_{\mathrm{eff}}p_1p_2}
	+
	o\!\left(B_{\mathrm{eff}}^{-1}\right).
	\]
	Moreover,
	\[
	\psi_{B,\beta}'(0) =
	-1+o(1).
	\]
	Expanding the unique zero of \(\psi_{B,\beta}\) around zero gives
	\begin{equation}
		\gamma_{B,\beta}^* =
		\frac{p_1-p_2}
		{2B_{\mathrm{eff}}p_1p_2}
		+
		o\!\left(B_{\mathrm{eff}}^{-1}\right).
		\label{eq:app-momentum-bias}
	\end{equation}
	Equivalently,
	\[
	\gamma_{B,\beta}^* =
	\frac{1-\beta}{1+\beta}
	\frac{p_1-p_2}{2Bp_1p_2}
	+
	o\!\left(B_{\mathrm{eff}}^{-1}\right).
	\]
	Thus momentum suppresses the finite-sampling displacement of the effective gap according to \(B_{\mathrm{eff}}\).
	
	\paragraph{Tracking slowly evolving margins.}
	
	We now return to the actual dynamics, in which \(u_i(t)\) evolve over time. The EMA has a memory length of order
	\[
	\epsilon^{-1} =
	(1-\beta)^{-1}
	\]
	iterations. Over one iteration, the margins move by \(O(\eta)\). In the late-time balanced regime, the relative variation of the mean gradient over one iteration is therefore also \(O(\eta)\). Across one EMA window, this produces a relative tracking error of order
	\begin{equation}
		O\!\left(
		\frac{\eta}{\epsilon}
		\right).
		\label{eq:app-momentum-tracking-error}
	\end{equation}
	
	On the other hand, the frozen-buffer fluctuation scale follows from \eqref{eq:app-ema-frequency-var}. Near balanced growth, where
	\[
	p_1u_1\asymp p_2u_2,
	\]
	we obtain
	\begin{equation}
		\frac{
			\|m_t-\mathbb E[m_t]\|_2
		}{
			\|\mathbb E[m_t]\|_2
		} =
		O_{\mathbb P}\!\left(
		\sqrt{
			\frac{\epsilon}{Bp_{\min}}
		}
		\right).
		\label{eq:app-momentum-relative-fluctuation}
	\end{equation}
	
	Thus the two scale-separation requirements are
	\begin{equation}
		\frac{\eta}{\epsilon}\ll 1,
		\qquad
		\frac{\epsilon}{Bp_{\min}}\ll 1.
		\label{eq:app-momentum-two-conditions}
	\end{equation}
	Equivalently,
	\[
	\eta\ll 1-\beta\ll Bp_{\min}.
	\]
	The first condition ensures that the EMA tracks the slowly evolving gradient direction, while the second ensures that momentum averaging provides sufficiently small relative task-frequency fluctuations.
	
	\paragraph{Long-time diffusion with momentum.}
	
	Momentum reduces the instantaneous variance of the buffer but also introduces serial correlations. We now examine their combined effect on the leading long-time diffusion coefficient.
	
	Freeze the margins and write
	\[
	\bar g =
	\left(
	p_1u_1,
	p_2u_2
	\right)^\top,
	\qquad
	n_t:=m_t-\bar g.
	\]
	The fresh-gradient covariance is
	\[
	\Sigma =
	\frac{p_1p_2}{B}
	\begin{bmatrix}
		u_1^2   & -u_1u_2 \\
		-u_1u_2 & u_2^2
	\end{bmatrix}.
	\]
	The stationary EMA covariance satisfies
	\[
	\operatorname{Cov}(n_t) =
	\frac{1-\beta}{1+\beta}\Sigma.
	\]
	More generally,
	\[
	\operatorname{Cov}(n_t,n_{t+k}) =
	\beta^{|k|}
	\frac{1-\beta}{1+\beta}
	\Sigma.
	\]
	Consequently,
	\begin{align}
		\sum_{k\in\mathbb{Z}}
		\operatorname{Cov}(n_0,n_k)
		& =
		\frac{1-\beta}{1+\beta}
		\left(
		1+2\sum_{k=1}^{\infty}\beta^k
		\right)\Sigma
		\notag \\
		& =
		\Sigma.
		\label{eq:app-momentum-integrated-cov}
	\end{align}
	Thus the reduction in instantaneous variance is exactly compensated, at the covariance level, by the increased correlation time. To translate this into fluctuations of the gap, define
	\[
	\ell(m)
	:=
	\sqrt{2}
	\frac{m_1-m_2}{\|m\|_2}.
	\]
	Then
	\[
	D(t+1)-D(t) =
	\eta\,\ell(m_t).
	\]
	Linearizing around \(\bar g\),
	\[
	\ell(m_t)-\ell(\bar g)
	\simeq
	\nabla\ell(\bar g)^\top n_t.
	\]
	Hence, by \eqref{eq:app-momentum-integrated-cov},
	\[
	\sum_{k\in\mathbb{Z}}
	\operatorname{Cov}\!\left(
	\ell(m_0),
	\ell(m_k)
	\right)
	\simeq
	\nabla\ell(\bar g)^\top
	\Sigma
	\nabla\ell(\bar g).
	\]
	
	At the full-batch balanced point
	\[
	p_1u_1=p_2u_2,
	\]
	a direct calculation gives
	\begin{equation}
		\nabla\ell(\bar g)^\top
		\Sigma
		\nabla\ell(\bar g) =
		\frac{1}{Bp_1p_2}.
		\label{eq:app-momentum-gap-cov}
	\end{equation}
	Since \(\gamma_{B,\beta}^*\to0\) in the well-sampled regime, the same coefficient holds to leading order around the momentum-induced equilibrium.
	
	Therefore, on the slow time
	\[
	\tau=\eta t,
	\]
	the local gap dynamics are approximated by
	\begin{equation}
		d\gamma
		\simeq
		-\left(
		\gamma-\gamma_{B,\beta}^*
		\right)d\tau
		+
		\sqrt{
			\frac{\eta}{Bp_1p_2}
		}\,
		d\mathcal W_\tau,
		\label{eq:app-momentum-gap-diffusion}
	\end{equation}
	up to corrections associated with finite effective sample size and tracking error. In particular, the local stationary fluctuation scale is
	\[
	O_{\mathbb P}\!\left(
	\sqrt{
		\frac{\eta}{Bp_1p_2}
	}
	\right),
	\]
	which matches the leading fluctuation scale of the momentum-free well-sampled regime.
	
	Combining the momentum-induced equilibrium, the tracking error, and the local diffusion scale gives the late-time approximation
	\begin{equation}
		D =
		\log\frac{p_1}{p_2}
		+
		\gamma_{B,\beta}^*
		+
		O\!\left(
		\frac{\eta}{1-\beta}
		\right)
		+
		O_{\mathbb P}\!\left(
		\sqrt{
			\frac{\eta}{Bp_1p_2}
		}
		\right).
		\label{eq:app-momentum-gap-concentration}
	\end{equation}
	
	Finally, consider a scale-separated regime satisfying
	\[
	\frac{Bp_{\min}}{1-\beta}
	\longrightarrow
	\infty,
	\qquad
	\frac{\eta}{1-\beta}
	\longrightarrow
	0.
	\]
	The first condition implies
	\[
	B_{\mathrm{eff}}p_{\min}
	\longrightarrow
	\infty,
	\]
	and hence, by \eqref{eq:app-momentum-equilibrium-limit},
	\[
	\gamma_{B,\beta}^*
	\longrightarrow
	0.
	\]
	Moreover, since \(p_1\ge 1/2\) and \(p_2=p_{\min}\),
	\[
	\frac{\eta}{Bp_1p_2} =
	O\!\left(
	\frac{\eta}{1-\beta}
	\frac{1-\beta}{Bp_{\min}}
	\right)
	\longrightarrow
	0.
	\]
	Thus the late-time local approximation gives
	\begin{equation}
		\Delta_1-\Delta_2 =
		\log\frac{p_1}{p_2}
		+
		o_{\mathbb P}(1),
		\qquad
		\frac{Bp_{\min}}{1-\beta}\to\infty,
		\qquad
		\frac{\eta}{1-\beta}\to0.
		\label{eq:app-momentum-limit}
	\end{equation}
	
	Momentum therefore acts as cross-step task aggregation. Under frozen margins, it replaces the noisy within-batch frequency by an exponentially averaged frequency with variance-equivalent batch size \(B_{\mathrm{eff}}=B(1+\beta)/(1-\beta)\). When the EMA also tracks the slowly evolving margins, this variance reduction drives the gap toward the full-batch value. At the same time, the reduction in instantaneous variance is compensated by increased correlation across steps, leaving the leading local long-time diffusion coefficient unchanged.
	
	\section{Higher-Rank Sampling Dynamics and Momentum Averaging}
	\label{app:hierarchical-memory}
	
	Throughout this section, the \(K=MC\) tasks are partitioned into \(M\) blocks of size \(C\). Tasks in block \(\mu\) have common population frequency
	\[
	p_j=\frac{\tilde p_\mu}{C},
	\qquad
	\sum_{\mu=1}^M \tilde p_\mu=1.
	\]
	The organization mirrors Proposition~\ref{prop:hierarchical-regimes}. Appendix~D.1 gives the exact occurrence-gated clock in the frequency-limited regime, and Appendix~D.3 gives the balanced full-batch contrast response. Appendices~D.4--D.5 then establish how the well-sampled response approaches that full-batch behavior and quantify the residual $O_{\mathbb P}((Bp_j)^{-1/2})$ coupling; the cancellation removes the direct contrast perturbation, leaving contrast-to-average coupling as the leading first-order effect. Appendix~D.6 treats momentum as averaging across steps with a separate tracking condition. A finite-horizon continuum theorem is retained in Appendix~D.7 as a technical refinement.
	
	\subsection{Sparse regime: exact occurrence-gated clock and block-symmetrized reference}
	\label{app:block-trajectory}
	
	We begin with the part of the hierarchical dynamics that is exact for arbitrary batch size. The partial polar factor has the same right nullspace as its input: if $Gx=0$, then $\operatorname{msgn}(G)x=0$. Indeed, for an SVD $G=U\Sigma V^\top$, $Gx=0$ implies $\Sigma V^\top x=0$; the coordinates of $V^\top x$ corresponding to positive singular values therefore vanish, which also gives $\operatorname{msgn}(G)x=0$.
	
	For task $j$, if $q_j(t)=0$ then
	\[
	(I-S_t)Q_te_j=0,
	\]
	so
	\begin{equation}
		(\overline W_{t+1}-\overline W_t)e_j=0.
		\label{eq:app-task-absence-freezes-column}
	\end{equation}
	Since $\|\operatorname{msgn}(G)\|_{\mathrm{op}}\le1$,
	\begin{equation}
		\|\overline W_Te_j-\overline W_0e_j\|_2
		\le
		\eta N_j(T),
		\qquad
		N_j(T):=\sum_{t<T}\mathbf 1\{q_j(t)>0\}.
		\label{eq:app-task-observation-clock-bound}
	\end{equation}
	For independent mini-batches, the indicators are i.i.d. Bernoulli with parameter
	\[
	r_j=1-(1-p_j)^B,
	\]
	and hence $N_j(T)/T\to r_j$ almost surely. The expansion $r_j=Bp_j+O((Bp_j)^2)$ in the sparse regime gives the trajectory-level frequency clock used in the main text.
	
	Next define the empirical block masses
	\[
	\pi_{\mu,t}=\sum_{j\in\mu}q_j(t)
	\]
	and the block-averaged count matrix
	\[
	Q_t^\parallel
	=
	\operatorname{diag}\!\left(
	\frac{\pi_{1,t}}{C}I_C,\ldots,
	\frac{\pi_{M,t}}{C}I_C
	\right).
	\]
	Starting from $\widetilde W_0=0$, consider
	\[
	\widetilde W_{t+1}
	=
	\widetilde W_t
	+
	\eta\operatorname{msgn}((I-\widetilde S_t)Q_t^\parallel).
	\]
	Both the initialization and $Q_t^\parallel$ are invariant under permutations within each block, so the trajectory remains block symmetric. The decomposition in the next subsection therefore applies at every step. Whenever $\pi_{\mu,t}>0$, the contrast singular value on $U_\mu$ is positive and the partial polar factor equals the identity there; when $\pi_{\mu,t}=0$, that contrast block is in the right nullspace. Hence
	\begin{equation}
		U_\mu^\top
		(\widetilde W_{t+1}-\widetilde W_t)
		U_\mu
		=
		\eta\mathbf 1\{\pi_{\mu,t}>0\}I_{C-1},
		\label{eq:app-reference-contrast-increment}
	\end{equation}
	Summing over steps gives
	\begin{equation}
		U_\mu^\top\widetilde W_TU_\mu
		=
		\eta N_\mu^{\mathrm{blk}}(T)I_{C-1},
		\qquad
		N_\mu^{\mathrm{blk}}(T):=
		\sum_{t<T}\mathbf 1\{\pi_{\mu,t}>0\}.
		\label{eq:app-reference-contrast-trajectory}
	\end{equation}
	This sampled block-symmetric process is an exact finite-$B$ benchmark. Appendix~D.7 records the stronger finite-horizon continuum refinement obtained after taking large expected task counts and a small step size.
	
	\subsection{Block decomposition at a symmetric reference state}
	\label{app:block-decomposition}
	
	For each block \(\mu\), let
	\[
	X_\mu\in\mathbb R^{C\times(C-1)}
	\]
	have orthonormal columns spanning the zero-sum subspace
	\[
	\left\{
	x\in\mathbb R^C:\mathbf 1_C^\top x=0
	\right\}.
	\]
	Let \(U_\mu\in\mathbb R^{K\times(C-1)}\) denote the embedding of \(X_\mu\) into block \(\mu\). Define the normalized block-average vectors
	\[
	(y_\mu)_i =
	\begin{cases}
		1/\sqrt C, & i\in\mu,          \\
		0,         & \text{otherwise},
	\end{cases}
	\]
	and set
	\[
	Y=(y_1,\ldots,y_M),
	\qquad
	\mathcal U=(U_1,\ldots,U_M,Y).
	\]
	Then \(\mathcal U\) is orthogonal.
	
	Consider a block-symmetric softmax matrix \(\widetilde S\), and write
	\[
	A:=I-\widetilde S.
	\]
	Block symmetry implies that each within-block contrast subspace is invariant under \(A\). Hence there exist scalars \(a_\mu\) and a matrix
	\(L\in\mathbb R^{M\times M}\) such that
	\begin{equation}
		\mathcal U^\top A\mathcal U =
		\begin{bmatrix}
			\operatorname{diag}\!\left(a_\mu I_{C-1}\right) & 0 \\
			0                                               & L
		\end{bmatrix}.
		\label{eq:app-A-block-decomposition}
	\end{equation}
	Equivalently, if \(d_\mu\) denotes the eigenvalue of \(\widetilde S\) on the within-block zero-sum subspace of block \(\mu\), then
	\[
	a_\mu=1-d_\mu.
	\]
	Along the block-symmetric full-batch trajectory considered below, \(a_\mu>0\).
	
	Now let
	\[
	Q^\parallel =
	\operatorname{diag}
	\left(
	r_1 I_C,\ldots,r_M I_C
	\right),
	\qquad
	r_\mu>0.
	\]
	In the basis \(\mathcal U\),
	\[
	\mathcal U^\top Q^\parallel\mathcal U =
	\begin{bmatrix}
		\operatorname{diag}\!\left(r_\mu I_{C-1}\right) & 0 \\
		0                                               & R
	\end{bmatrix},
	\qquad
	R:=\operatorname{diag}(r_1,\ldots,r_M).
	\]
	Therefore the block-symmetric reference gradient
	\[
	\widetilde G
	:=
	A Q^\parallel
	\]
	takes the form
	\begin{equation}
		\mathcal U^\top\widetilde G\mathcal U =
		\begin{bmatrix}
			\operatorname{diag}\!\left(\sigma_\mu I_{C-1}\right) & 0 \\
			0                                                    & H
		\end{bmatrix},
		\qquad
		\sigma_\mu:=a_\mu r_\mu,
		\qquad
		H:=LR.
		\label{eq:app-reference-gradient}
	\end{equation}
	
	Since \(\widetilde S\) is a positive column-stochastic matrix, \(A=I-\widetilde S\) has rank \(K-1\). Because \(a_\mu>0\), it follows from \eqref{eq:app-A-block-decomposition} that \(L\), and hence \(H\), has rank \(M-1\).
	
	\subsection{Full-batch balanced contrast dynamics}
	\label{app:block-fullbatch}
	
	Under full-batch training,
	\[
	Q=P =
	\operatorname{diag}
	\left(
	\frac{\tilde p_1}{C}I_C,\ldots,
	\frac{\tilde p_M}{C}I_C
	\right).
	\]
	Starting from \(\overline W_0=0\), the full-batch dynamics preserve
	block symmetry. Thus \((I-S_t)P\) has the form
	\eqref{eq:app-reference-gradient} at every iteration, with
	\[
	r_\mu=\frac{\tilde p_\mu}{C}.
	\]
	
	Since every contrast singular value \(\sigma_\mu\) is positive, the matrix sign acts as the identity on every within-block contrast subspace:
	\begin{equation}
		\operatorname{msgn}\!\left((I-S_t)P\right) =
		\sum_{\mu=1}^M U_\mu U_\mu^\top
		+
		Y\operatorname{msgn}(H_t)Y^\top.
		\label{eq:app-fullbatch-polar}
	\end{equation}
	Using
	\[
	\sum_{\mu=1}^M U_\mu U_\mu^\top =
	I_K-YY^\top,
	\]
	we may equivalently write
	\begin{equation}
		\operatorname{msgn}\!\left((I-S_t)P\right) =
		I_K
		+
		Y\left(
		\operatorname{msgn}(H_t)-I_M
		\right)Y^\top.
		\label{eq:app-fullbatch-entrywise}
	\end{equation}
	Since
	\[
	\left\|
	\operatorname{msgn}(H_t)-I_M
	\right\|_{\mathrm{op}}
	\le 2,
	\]
	every entry of the correction term in \eqref{eq:app-fullbatch-entrywise} is \(O(1/C)\). Hence
	\[
	\operatorname{msgn}\!\left((I-S_t)P\right) =
	I_K+O_{\mathrm{entry}}\!\left(\frac1C\right).
	\]
	
	More strongly, the within-block statement is exact:
	\begin{equation}
		U_\mu^\top
		\operatorname{msgn}\!\left((I-S_t)P\right)
		U_\nu =
		\delta_{\mu\nu}I_{C-1}.
		\label{eq:app-fullbatch-exact-contrast}
	\end{equation}
	Therefore
	\[
	U_\mu^\top
	\left(
	\overline W_{t+1}-\overline W_t
	\right)
	U_\mu =
	\eta I_{C-1}.
	\]
	Thus full-batch Muon exactly equalizes the update on all within-block contrast directions; all remaining nontrivial structure is confined to the \(M\)-dimensional block-average subspace.
	
	\subsection{Mini-batch perturbation decomposition around the symmetric reference}
	\label{app:block-minibatch}
	
	We analyze mini-batch sampling fluctuations around a fixed block-symmetric reference state \(\widetilde S\). The calculation is local and one-step: it quantifies the instantaneous response of the polar map rather than the full departure of \(S_t\) from the block-symmetric manifold.
	
	Let
	\[
	Q=\operatorname{diag}(q_1,\ldots,q_K),
	\]
	where \(Bq\) is the empirical frequency vector of a mini-batch of size \(B\). Define the empirical block frequencies
	\[
	\pi_\mu
	:=
	\sum_{j\in\mu}q_j.
	\]
	Decompose
	\[
	Q=Q^\parallel+Q^\perp,
	\]
	where
	\begin{equation}
		Q^\parallel =
		\operatorname{diag}
		\left(
		\frac{\pi_1}{C}I_C,\ldots,
		\frac{\pi_M}{C}I_C
		\right).
		\label{eq:app-Q-parallel}
	\end{equation}
	For \(j\in\mu\), define the within-block residual
	\[
	(\epsilon_\mu)_j
	:=
	q_j-\frac{\pi_\mu}{C}.
	\]
	Then
	\[
	\mathbf 1_C^\top\epsilon_\mu=0,
	\]
	and the restriction of \(Q^\perp\) to block \(\mu\) is
	\[
	\operatorname{diag}(\epsilon_\mu).
	\]
	
	Conditioned on \(\pi_\mu\), the \(B\pi_\mu\) samples falling in block \(\mu\) are uniformly distributed among its \(C\) tasks. Therefore
	\begin{equation}
		\mathbb E[
		\epsilon_\mu
		\mid
		\pi_\mu
		] =
		0,
		\label{eq:app-epsilon-mean}
	\end{equation}
	and
	\begin{equation}
		\operatorname{Cov}\!\left(
		\epsilon_\mu
		\mid
		\pi_\mu
		\right) =
		\frac{\pi_\mu}{BC}
		\left(
		I_C-\frac1C J_C
		\right).
		\label{eq:app-epsilon-cov}
	\end{equation}
	Define
	\[
	z_\mu:=X_\mu^\top\epsilon_\mu.
	\]
	Since \(X_\mu^\top J_C=0\),
	\begin{equation}
		\operatorname{Cov}\!\left(
		z_\mu
		\mid
		\pi_\mu
		\right) =
		\frac{\pi_\mu}{BC}I_{C-1},
		\label{eq:app-z-cov}
	\end{equation}
	and hence
	\begin{equation}
		\mathbb E\!\left[
		\|z_\mu\|_2^2
		\mid
		\pi_\mu
		\right] =
		\frac{(C-1)\pi_\mu}{BC}.
		\label{eq:app-z-second-moment}
	\end{equation}
	
	Let
	\[
	A:=I-\widetilde S,
	\qquad
	\widetilde G:=AQ^\parallel,
	\qquad
	E:=AQ^\perp.
	\]
	For the fixed-rank calculation below, work on the event
	\[
	\mathcal E_B
	:=
	\{\pi_\mu>0\text{ for every }\mu\}.
	\]
	For fixed \(M\),
	\[
	\mathbb P(\mathcal E_B^c)
	\le
	\sum_{\mu=1}^M(1-\tilde p_\mu)^B
	\le
	M e^{-B\tilde p_{\min}},
	\]
	so \(\mathbb P(\mathcal E_B)\to1\) in the well-sampled regime considered below. On \(\mathcal E_B\), \(Q^\parallel\) is invertible. Moreover, for sufficiently small \(|s|\), the diagonal matrix
	\[
	Q^\parallel+sQ^\perp
	\]
	remains invertible, so the path
	\[
	G(s)
	:=
	A\left(
	Q^\parallel+sQ^\perp
	\right) =
	\widetilde G+sE
	\]
	has constant rank \(K-1\) in a neighborhood of \(s=0\).
	
	Because \(\operatorname{msgn}\) is not an ambient-smooth map at a rank-deficient matrix, we define its derivative here along this constant-rank path:
	\begin{equation}
		\mathcal D_{\widetilde G}[E]
		:=
		\left.
		\frac{d}{ds}
		\operatorname{msgn}\!\left(
		\widetilde G+sE
		\right)
		\right|_{s=0}.
		\label{eq:app-tangent-polar-derivative}
	\end{equation}
	
	We will use the following standard fixed-rank derivative formula.
	
	\begin{lemma}[Derivative of the partial polar factor along a fixed-rank path]
		\label{lem:app-fixed-rank-polar}
		Let \(G(s)\) be a differentiable square-matrix path of constant rank
		\(r\) near \(s=0\), and write
		\[
		G(0) =
		U
		\begin{bmatrix}
			\Sigma & 0 \\
			0      & 0
		\end{bmatrix}
		V^\top,
		\qquad
		\Sigma =
		\operatorname{diag}(\sigma_1,\ldots,\sigma_r),
		\qquad
		\sigma_i>0.
		\]
		Let
		\[
		E:=G'(0),
		\qquad
		\widehat E:=U^\top E V,
		\]
		and let
		\[
		\dot P
		:=
		\left.
		\frac{d}{ds}
		\operatorname{msgn}(G(s))
		\right|_{s=0}.
		\]
		Partition \(U=(U_+,U_0)\) and \(V=(V_+,V_0)\) according to the positive
		and zero singular subspaces. Then
		\begin{align}
			\left(
			U_+^\top\dot P V_+
			\right)_{ij}
			& =
			\frac{
				(\widehat E_{++})_{ij}
				-
				(\widehat E_{++})_{ji}
			}{
				\sigma_i+\sigma_j
			},
			\label{eq:app-polar-derivative-positive}   \\
			U_+^\top\dot P V_0
			& =
			\Sigma^{-1}\widehat E_{+0},
			\label{eq:app-polar-derivative-right-null} \\
			U_0^\top\dot P V_+
			& =
			\widehat E_{0+}\Sigma^{-1},
			\label{eq:app-polar-derivative-left-null}  \\
			U_0^\top\dot P V_0
			& =
			0.
			\label{eq:app-polar-derivative-null}
		\end{align}
	\end{lemma}
	
	\begin{proof}
		Because the rank is constant, the column and row spaces vary smoothly
		locally. Choose differentiable orthonormal frames \(U_+(s)\) and
		\(V_+(s)\) for these spaces and write
		\[
		G(s)=U_+(s)R(s)V_+(s)^\top,
		\]
		where \(R(s)\in\mathbb R^{r\times r}\) is invertible. We choose the
		frames so that \(R(0)=\Sigma\). If \(P_R(s)\) denotes the polar factor
		of \(R(s)\), then
		\[
		\operatorname{msgn}(G(s)) =
		U_+(s)P_R(s)V_+(s)^\top,
		\qquad
		P_R(0)=I_r.
		\]
		
		Let
		\[
		A:=U_+^\top\dot U_+,
		\qquad
		B:=V_+^\top\dot V_+,
		\qquad
		\Omega_R:=\dot P_R(0).
		\]
		Both \(A\) and \(B\) are skew-symmetric. Differentiating the polar
		decomposition of the invertible core \(R(s)\) at \(s=0\) gives the
		Sylvester equation
		\[
		\Sigma\Omega_R+\Omega_R\Sigma =
		\dot R-\dot R^\top.
		\]
		On the positive singular subspaces,
		\[
		\widehat E_{++} =
		A\Sigma+\dot R-\Sigma B,
		\]
		while
		\[
		U_+^\top\dot P V_+ =
		A+\Omega_R-B.
		\]
		Combining the last two displays yields
		\[
		\Sigma
		\left(U_+^\top\dot P V_+\right)
		+
		\left(U_+^\top\dot P V_+\right)\Sigma =
		\widehat E_{++}-\widehat E_{++}^\top,
		\]
		which gives \eqref{eq:app-polar-derivative-positive} entrywise.
		
		For the cross blocks, differentiating
		\(G(s)V_0(s)=0\) and
		\(U_0(s)^\top G(s)=0\), or equivalently projecting the derivative of
		\(G(s)=U_+(s)R(s)V_+(s)^\top\), gives
		\[
		U_+^\top\dot P V_0 =
		\Sigma^{-1}\widehat E_{+0},
		\qquad
		U_0^\top\dot P V_+ =
		\widehat E_{0+}\Sigma^{-1}.
		\]
		Finally, differentiating
		\(\operatorname{msgn}(G(s)) =
		U_+(s)P_R(s)V_+(s)^\top\)
		shows that its derivative has no null-to-null block, proving
		\eqref{eq:app-polar-derivative-null}.
	\end{proof}
	
	We next compute the perturbation \(E=AQ^\perp\) in the basis \(\mathcal U\). By \eqref{eq:app-A-block-decomposition},
	\[
	U_\mu^\top A U_\mu=a_\mu I_{C-1},
	\qquad
	U_\mu^\top AY=0,
	\qquad
	Y^\top AU_\mu=0,
	\qquad
	Y^\top AY=L.
	\]
	Moreover,
	\begin{align}
		U_\mu^\top Q^\perp U_\mu
		& =
		D_\mu,
		\qquad
		D_\mu
		:=
		X_\mu^\top
		\operatorname{diag}(\epsilon_\mu)
		X_\mu,
		\label{eq:app-Qperp-UU} \\
		U_\mu^\top Q^\perp Y
		& =
		\frac{1}{\sqrt C}
		z_\mu e_\mu^\top,
		\label{eq:app-Qperp-UY} \\
		Y^\top Q^\perp U_\mu
		& =
		\frac{1}{\sqrt C}
		e_\mu z_\mu^\top,
		\label{eq:app-Qperp-YU} \\
		Y^\top Q^\perp Y
		& =
		0.
		\label{eq:app-Qperp-YY}
	\end{align}
	Here \(e_\mu\) is the \(\mu\)-th standard basis vector of \(\mathbb R^M\). Hence
	\begin{align}
		U_\mu^\top E U_\mu
		& =
		a_\mu D_\mu,
		\label{eq:app-E-UU} \\
		U_\mu^\top E Y
		& =
		\frac{a_\mu}{\sqrt C}
		z_\mu e_\mu^\top,
		\label{eq:app-E-UY} \\
		Y^\top E U_\mu
		& =
		\frac{1}{\sqrt C}
		L e_\mu z_\mu^\top,
		\label{eq:app-E-YU} \\
		Y^\top E Y
		& =
		0.
		\label{eq:app-E-YY}
	\end{align}
	For \(\mu\neq\nu\),
	\[
	U_\mu^\top E U_\nu=0.
	\]
	
	\subsection{Well-sampled regime: first-order cancellation and residual coupling}
	\label{app:block-coupling}
	
	We now apply Lemma~\ref{lem:app-fixed-rank-polar} to the reference gradient \(\widetilde G\).
	
	\paragraph{Cancellation inside the contrast subspaces.}
	
	Recall from \eqref{eq:app-reference-gradient} that the singular value of \(\widetilde G\) on \(U_\mu\) is
	\[
	\sigma_\mu =
	a_\mu r_\mu,
	\qquad
	r_\mu =
	\frac{\pi_\mu}{C}.
	\]
	The matrix \(D_\mu\) in \eqref{eq:app-Qperp-UU} is symmetric. Therefore
	\[
	U_\mu^\top E U_\mu =
	\left(
	U_\mu^\top E U_\mu
	\right)^\top.
	\]
	By \eqref{eq:app-polar-derivative-positive},
	\[
	U_\mu^\top
	\mathcal D_{\widetilde G}[E]
	U_\mu =
	0.
	\]
	For \(\mu\neq\nu\), both cross blocks \(U_\mu^\top E U_\nu\) and \(U_\nu^\top E U_\mu\) vanish. Hence
	\begin{equation}
		U_\mu^\top
		\mathcal D_{\widetilde G}[E]
		U_\nu =
		0,
		\qquad
		\forall\,\mu,\nu.
		\label{eq:app-hierarchical-cancellation}
	\end{equation}
	Thus the direct first-order effect of within-block sampling fluctuations vanishes on every contrast-to-contrast block.
	
	\paragraph{Coupling to block-average directions.}
	
	Let
	\[
	H =
	\widehat U
	\operatorname{diag}(s_1,\ldots,s_M)
	\widehat V^\top
	\]
	be an SVD of the block-average matrix \(H=LR\), with
	\[
	s_1,\ldots,s_{M-1}>0,
	\qquad
	s_M=0.
	\]
	Write \(\widehat u_k\) and \(\widehat v_k\) for the corresponding left and right singular vectors.
	
	Since
	\[
	H=LR,
	\]
	the singular-vector identity
	\[
	\widehat u_k^\top H =
	s_k\widehat v_k^\top
	\]
	gives, for every \(k\) and \(\mu\),
	\begin{equation}
		\widehat u_k^\top L e_\mu =
		\frac{s_k}{r_\mu}
		(\widehat v_k)_\mu.
		\label{eq:app-H-singular-relation}
	\end{equation}
	For \(s_k=0\), this follows because \(R\) is invertible and \(\widehat u_k^\top LR=0\).
	
	Using equations~\ref{eq:app-E-UY} and~\ref{eq:app-E-YU}, and the derivative formulas in Lemma~\ref{lem:app-fixed-rank-polar}, we obtain
	\begin{equation}
		U_\mu^\top
		\mathcal D_{\widetilde G}[E]
		Y\widehat v_k =
		\frac{(\widehat v_k)_\mu}{\sqrt C\,r_\mu}
		\frac{\sigma_\mu-s_k}{\sigma_\mu+s_k}
		z_\mu,
		\qquad
		k=1,\ldots,M.
		\label{eq:app-coupling-mode}
	\end{equation}
	When \(s_k=0\), the fraction in \eqref{eq:app-coupling-mode} equals one, and the formula agrees with the positive-to-null derivative in equation~\ref{eq:app-polar-derivative-right-null}.
	
	Since
	\[
	\left|
	\frac{\sigma_\mu-s_k}{\sigma_\mu+s_k}
	\right|
	\le 1
	\]
	and \(\widehat V\) is orthogonal,
	\begin{align}
		\left\|
		U_\mu^\top
		\mathcal D_{\widetilde G}[E]
		Y
		\right\|_F^2
		& \le
		\frac{\|z_\mu\|_2^2}
		{C r_\mu^2}
		\sum_{k=1}^M
		(\widehat v_k)_\mu^2
		\notag \\
		& =
		\frac{\|z_\mu\|_2^2}
		{C r_\mu^2}.
		\label{eq:app-coupling-deterministic-bound}
	\end{align}
	Conditioning on \(\pi_\mu\) and using \eqref{eq:app-z-second-moment} together with \(r_\mu=\pi_\mu/C\),
	\begin{align}
		\mathbb E\!\left[
		\left\|
		U_\mu^\top
		\mathcal D_{\widetilde G}[E]
		Y
		\right\|_F^2
		\biggm|
		\pi_\mu
		\right]
		& \le
		\frac{1}{C r_\mu^2}
		\frac{(C-1)\pi_\mu}{BC}
		\notag \\
		& =
		\frac{C-1}{B\pi_\mu}.
		\label{eq:app-coupling-second-moment}
	\end{align}
	Therefore, conditionally on \(\pi_\mu>0\),
	\begin{equation}
		\left\|
		U_\mu^\top
		\mathcal D_{\widetilde G}[E]
		Y
		\right\|_F =
		O_{\mathbb P}\!\left(
		\sqrt{
			\frac{C}{B\pi_\mu}
		}
		\right).
		\label{eq:app-coupling-pi-scale}
	\end{equation}
	
	The empirical block frequency satisfies
	\[
	B\pi_\mu
	\sim
	\operatorname{Binom}(B,\tilde p_\mu).
	\]
	Hence, whenever
	\[
	B\tilde p_\mu\to\infty,
	\]
	we have
	\[
	\frac{\pi_\mu}{\tilde p_\mu}
	\xrightarrow{\mathbb P}
	1.
	\]
	It follows from \eqref{eq:app-coupling-pi-scale} that
	\begin{equation}
		\left\|
		U_\mu^\top
		\mathcal D_{\widetilde G}[E]
		Y
		\right\|_F =
		O_{\mathbb P}\!\left(
		\sqrt{
			\frac{C}{B\tilde p_\mu}
		}
		\right).
		\label{eq:app-coupling-population-scale}
	\end{equation}
	Thus the first-order contrast-to-average coupling vanishes when
	\begin{equation}
		\frac{B\tilde p_\mu}{C}\to\infty.
		\label{eq:app-spatial-coverage-condition}
	\end{equation}
	
	For fixed \(M\), the sufficient uniform condition
	\[
	\frac{B\tilde p_{\min}}{C}\to\infty
	\]
	makes the first-order coupling small in every block simultaneously. This condition should be interpreted as a large expected sample count per task. It does not require every individual task to appear in every mini-batch.
	
	\subsection{Momentum: averaging across steps and tracking}
	\label{app:block-momentum}
	
	We next analyze momentum as EMA averaging of the empirical frequency matrix. This subsection also proves Proposition~\ref{prop:temporal-aggregation}. The variance-equivalent batch-size notation is retained as a shorthand, but the underlying object is the exponentially weighted observation process derived explicitly below. Let the matrix momentum obey
	\begin{equation}
		\mathsf M_t =
		\beta\mathsf M_{t-1}
		+
		(1-\beta)G_t,
		\qquad
		G_t=(I-S_t)Q_t,
		\label{eq:app-matrix-momentum}
	\end{equation}
	with fresh independent mini-batches and $0\le\beta<1$. Here $G_t$ denotes the negative gradient in task coordinates, so the parameter update adds its normalized momentum.
	
	\paragraph{Taskwise stationary EMA and finite-initialization correction.}
	For a single task $j$, the empirical mini-batch frequency satisfies
	\[
	\mathbb E q_{j,t}=p_j,
	\qquad
	\operatorname{Var}(q_{j,t})=\frac{p_j(1-p_j)}{B},
	\]
	and the variables are independent across time. For the stationary infinite-past EMA
	\[
	\overline q_{j,t}
	=(1-\beta)\sum_{r=0}^{\infty}\beta^r q_{j,t-r},
	\]
	independence gives
	\begin{align}
		\mathbb E\overline q_{j,t}&=p_j,\notag\\
		\operatorname{Var}(\overline q_{j,t})
		&=(1-\beta)^2\sum_{r\ge0}\beta^{2r}\frac{p_j(1-p_j)}{B}
		=\frac{1-\beta}{1+\beta}\frac{p_j(1-p_j)}{B}.
		\label{eq:app-taskwise-ema-variance}
	\end{align}
	Thus, with
	\[
	B_{\mathrm{eff}}(\beta)=B\frac{1+\beta}{1-\beta},
	\]
	we have
	\[
	\frac{\operatorname{Var}(\overline q_{j,t})}{p_j^2}
	=\frac{1-p_j}{B_{\mathrm{eff}}(\beta)p_j},
	\]
	which proves the frozen-state variance identity in Proposition~\ref{prop:temporal-aggregation}.
	
	For comparison with an implementation initialized at $\mathsf M_{-1}=0$, define the bias-corrected finite-history frequency average
	\[
	\widehat q_{j,t}
	:=
	\frac{1-\beta}{1-\beta^{t+1}}
	\sum_{r=0}^{t}\beta^r q_{j,t-r}.
	\]
	Then $\mathbb E\widehat q_{j,t}=p_j$ and
	\begin{equation}
		\frac{\operatorname{Var}(\widehat q_{j,t})}{p_j^2}
		=
		\frac{1-p_j}{Bp_j}
		\frac{1-\beta}{1+\beta}
		\frac{1+\beta^{t+1}}{1-\beta^{t+1}}.
		\label{eq:app-finite-ema-variance}
	\end{equation}
	The last factor tends to one as $t\to\infty$, recovering the stationary $B_{\mathrm{eff}}$. Without bias correction, the zero-initialized buffer is exactly the truncated stationary sum and has mean $(1-\beta^{t+1})p_j$.
	
	For an evolving state, the stationary infinite-past form of \eqref{eq:app-matrix-momentum} is
	\[
	\mathsf M_t
	=(1-\beta)\sum_{r\ge0}\beta^r(I-S_{t-r})Q_{t-r}.
	\]
	Adding and subtracting $(I-S_t)Q_{t-r}$ term by term gives the exact decomposition
	\begin{equation}
		\mathsf M_t=(I-S_t)\overline Q_t+R_t,
		\qquad
		R_t=(1-\beta)\sum_{r\ge0}\beta^r(S_t-S_{t-r})Q_{t-r},
		\label{eq:app-general-tracking-decomposition}
	\end{equation}
	where $\overline Q_t=(1-\beta)\sum_{r\ge0}\beta^rQ_{t-r}$. If
	$\|S_t-S_{t-r}\|_F\le L_S\eta r$ and $\|Q_{t-r}\|_{\mathrm{op}}\le1$, then
	\begin{align}
		\|R_t\|_F
		&\le
		(1-\beta)\sum_{r\ge0}\beta^r
		\|S_t-S_{t-r}\|_F\,\|Q_{t-r}\|_{\mathrm{op}}\notag\\
		&\le
		L_S\eta(1-\beta)\sum_{r\ge0}r\beta^r
		=
		L_S\eta\frac{\beta}{1-\beta},
		\label{eq:app-general-tracking-bound}
	\end{align}
	proving the tracking part of Proposition~\ref{prop:temporal-aggregation}. A zero-initialized finite-history buffer has the same decomposition with the sums truncated at $r=t$; its bound is no larger, while bias correction multiplies both the frozen term and the remainder by $(1-\beta^{t+1})^{-1}$.
	
	\paragraph{Frozen-state block representation.}
	
	First freeze the state at a block-symmetric matrix \(\widetilde S\), and let
	\[
	A:=I-\widetilde S.
	\]
	After initialization transients,
	\[
	\mathsf M_t =
	(1-\beta)
	\sum_{k=0}^\infty
	\beta^k A Q_{t-k} =
	A\overline Q_t,
	\]
	where
	\begin{equation}
		\overline Q_t
		:=
		(1-\beta)
		\sum_{k=0}^\infty
		\beta^k Q_{t-k}.
		\label{eq:app-ema-Q}
	\end{equation}
	Thus, under a frozen state, momentum simply replaces the single-batch frequency matrix \(Q_t\) by its exponentially weighted average across steps.
	
	Decompose
	\[
	\overline Q_t =
	\overline Q_t^\parallel
	+
	\overline Q_t^\perp
	\]
	exactly as before. Define
	\[
	\overline\pi_\mu(t)
	:=
	(1-\beta)
	\sum_{k=0}^\infty
	\beta^k\pi_\mu(t-k),
	\]
	and let \(\overline\epsilon_\mu(t)\) be the corresponding within-block residual vector.
	
	The batch variables are independent across time. Since
	\[
	\mathbb E[\epsilon_\mu(t)]=0
	\]
	and, by averaging \eqref{eq:app-epsilon-cov} over \(\pi_\mu\),
	\[
	\operatorname{Cov}(\epsilon_\mu(t)) =
	\frac{\tilde p_\mu}{BC}
	\left(
	I_C-\frac1C J_C
	\right),
	\]
	we obtain
	\begin{equation}
		\operatorname{Cov}(\overline\epsilon_\mu) =
		\frac{1-\beta}{1+\beta}
		\frac{\tilde p_\mu}{BC}
		\left(
		I_C-\frac1C J_C
		\right).
		\label{eq:app-ema-epsilon-cov}
	\end{equation}
	Similarly,
	\begin{equation}
		\operatorname{Var}(\overline\pi_\mu) =
		\frac{1-\beta}{1+\beta}
		\frac{\tilde p_\mu(1-\tilde p_\mu)}{B}.
		\label{eq:app-ema-block-frequency-var}
	\end{equation}
	
	Define
	\begin{equation}
		B_{\mathrm{eff}}
		:=
		B\frac{1+\beta}{1-\beta}.
		\label{eq:app-hierarchical-effective-batch}
	\end{equation}
	Then equations~\ref{eq:app-ema-epsilon-cov} and~\ref{eq:app-ema-block-frequency-var} are exactly the corresponding single-batch \emph{covariance} formulas with \(B\) replaced by \(B_{\mathrm{eff}}\). This covariance identity is not a distributional identity: \(\overline Q_t\) is a weighted average of multinomial batches rather than an empirical frequency vector from \(B_{\mathrm{eff}}\) independent samples.
	
	Let
	\[
	\overline z_\mu
	:=
	X_\mu^\top\overline\epsilon_\mu.
	\]
	Thus \eqref{eq:app-ema-epsilon-cov} gives
	\begin{equation}
		\mathbb E\|\overline z_\mu\|_2^2 =
		\frac{C-1}{C}
		\frac{\tilde p_\mu}{B_{\mathrm{eff}}}.
		\label{eq:app-ema-z-second-moment}
	\end{equation}
	Moreover, if
	\[
	B_{\mathrm{eff}}\tilde p_\mu\to\infty,
	\]
	then
	\[
	\frac{\overline\pi_\mu}{\tilde p_\mu}
	\xrightarrow{\mathbb P}
	1
	\]
	by \eqref{eq:app-ema-block-frequency-var}.
	
	Combining the deterministic bound in equation~\ref{eq:app-coupling-deterministic-bound} with equation~\ref{eq:app-ema-z-second-moment} and the concentration of \(\overline\pi_\mu\) therefore yields
	\begin{equation}
		\left\|
		U_\mu^\top
		\mathcal D_{\overline G^\parallel}
		[\overline E]
		Y
		\right\|_F =
		O_{\mathbb P}\!\left(
		\sqrt{
			\frac{C}
			{B_{\mathrm{eff}}\tilde p_\mu}
		}
		\right),
		\label{eq:app-momentum-coupling-scale}
	\end{equation}
	where
	\[
	\overline G^\parallel =
	A\overline Q^\parallel,
	\qquad
	\overline E =
	A\overline Q^\perp.
	\]
	
	Hence, in probability at first order around the frozen symmetric state, momentum averaging suppresses the sampling distortion whenever
	\begin{equation}
		\frac{B_{\mathrm{eff}}\tilde p_{\min}}{C}
		\gg 1.
		\label{eq:app-temporal-coverage}
	\end{equation}
	Using \eqref{eq:app-hierarchical-effective-batch}, this can also be written, up to the bounded factor \(1+\beta\), as
	\begin{equation}
		\frac{B\tilde p_{\min}}{(1-\beta)C}
		\gg1.
		\label{eq:app-temporal-coverage-beta}
	\end{equation}
	
	\paragraph{Tracking a slowly evolving state.}
	We now return to the actual dynamics, where \(S_t\) evolves. The exact stationary-form expansion of the momentum is
	\[
	\mathsf M_t =
	(1-\beta)
	\sum_{k=0}^\infty
	\beta^k
	(I-S_{t-k})Q_{t-k}.
	\]
	If the reference state is frozen at its current value \(S_t\), the corresponding approximation is
	\[
	(I-S_t)\overline Q_t.
	\]
	Their difference is
	\begin{equation}
		\mathsf M_t
		-
		(I-S_t)\overline Q_t =
		(1-\beta)
		\sum_{k=0}^\infty
		\beta^k
		(S_t-S_{t-k})Q_{t-k}.
		\label{eq:app-momentum-tracking-difference}
	\end{equation}
	
	In the late-time regime, the block-symmetric reference quantities evolve on the parameter timescale \(1/\eta\). Locally, if their relative change per iteration is \(O(\eta)\), then over \(k\) iterations the mismatch is \(O(k\eta)\). Since
	\[
	(1-\beta)
	\sum_{k=0}^\infty
	k\beta^k =
	\frac{\beta}{1-\beta},
	\]
	\eqref{eq:app-momentum-tracking-difference} gives the adiabatic tracking scale
	\begin{equation}
		O\!\left(
		\frac{\eta}{1-\beta}
		\right).
		\label{eq:app-hierarchical-tracking-scale}
	\end{equation}
	Under the stated local \(O(\eta)\)-variation assumption, the frozen-state mismatch is therefore small on the adiabatic scale
	\begin{equation}
		\eta\ll 1-\beta.
		\label{eq:app-hierarchical-tracking-condition}
	\end{equation}
	This is a tracking estimate, not a global stability theorem for the stochastic trajectory.
	
	Thus the local momentum regime is described by two separate requirements:
	\begin{equation}
		\frac{B\tilde p_{\min}}{(1-\beta)C}\gg1,
		\label{eq:app-hierarchical-momentum-observability}
	\end{equation}
	for small relative sampling fluctuations under momentum, and
	\begin{equation}
		\frac{\eta}{1-\beta}\ll1,
		\label{eq:app-hierarchical-momentum-tracking}
	\end{equation}
	for adiabatic tracking. The first improves the frozen-state sampling estimate relative to the no-momentum condition \(B\tilde p_{\min}/C\gg1\); the second prevents the averaging window from becoming so long that the underlying gradient changes substantially across it.
	
	\subsection{Technical refinement: finite-horizon well-sampled continuum limit}
	\label{app:hierarchical-continuum-refinement}
	
	The main text does not use convergence to the full-batch continuum flow as the substantive higher-rank claim; large expected sample counts making mini-batch frequencies concentrate is expected. We retain the following result because it formalizes how the local sampling analysis connects to the full-batch trajectory on any fixed continuum-time horizon.
	
	\begin{proposition}[Finite-horizon well-sampled continuum limit]
		\label{prop:hierarchical-trajectory}
		Fix $M,C$ and positive task frequencies. Let $\overline W^{\eta,B}(\tau)$ be the piecewise-linear interpolation of the stochastic hierarchy on slow time $\tau=\eta t$, and let $W^\star$ solve the full-batch continuum flow $\dot W^\star=\operatorname{msgn}((I-S(W^\star))P)$ from $W^\star(0)=0$. For every finite $\tau_0$ and $\varepsilon>0$,
		\begin{equation}
			\lim_{\eta\downarrow0}\;\lim_{Bp_{\min}\to\infty}
			\mathbb P\!\left(
			\sup_{0\le\tau\le\tau_0}
			\|\overline W^{\eta,B}(\tau)-W^\star(\tau)\|_F>\varepsilon
			\right)=0.
			\label{eq:hierarchical-fluid-limit}
		\end{equation}
		The same holds along any joint sequence $\eta_n\downarrow0$, $B_np_{\min}\to\infty$ satisfying $\log(1/\eta_n)=o(B_np_{\min})$. Moreover, along the limiting flow,
		\begin{equation}
			U_\mu^\top W^\star(\tau)U_\nu
			=\delta_{\mu\nu}\,\tau I_{C-1},
			\qquad \tau\ge0.
			\label{eq:hierarchical-continuum-contrast}
		\end{equation}
	\end{proposition}
	
	\paragraph{From large expected task count to the continuum trajectory.}
	We now prove Proposition~\ref{prop:hierarchical-trajectory}. Define
	\[
	F(W):=\operatorname{msgn}((I-S(W))P).
	\]
	For every finite $W$, the softmax matrix $S(W)$ has strictly positive entries and is column stochastic. Perron--Frobenius therefore implies that its eigenvalue $1$ is simple, so $I-S(W)$, and hence $(I-S(W))P$, has rank $K-1$. The partial polar factor is smooth on the fixed-rank manifold in a neighborhood of each such matrix. Moreover
	\[
	\|F(W)\|_F\le \sqrt{K-1},
	\]
	so the ODE
	\[
	\dot W^\star(\tau)=F(W^\star(\tau)),
	\qquad W^\star(0)=0,
	\]
	has a unique global solution. On every compact time interval $[0,\tau_0]$, the image of $W^\star$ is compact and the smallest nonzero singular value of $(I-S(W^\star(\tau)))P$ has a positive minimum. Thus $F$ is uniformly Lipschitz on a sufficiently small tube around that trajectory.
	
	The ODE preserves permutations within every frequency block because $P$ and the initialization share those symmetries and the polar map is orthogonally equivariant. Therefore the block decomposition of Appendix~\ref{app:block-fullbatch} applies for every $\tau$. In particular,
	\begin{equation}
		U_\mu^\top \dot W^\star(\tau)U_\nu
		=
		\delta_{\mu\nu}I_{C-1}.
		\label{eq:app-continuum-exact-contrast-rate}
	\end{equation}
	Integrating from $W^\star(0)=0$ proves \eqref{eq:hierarchical-continuum-contrast} for every $\tau\ge0$.
	
	We first prove the sequential limit in \eqref{eq:hierarchical-fluid-limit}. Fix $\eta>0$ and $\tau_0<\infty$. Then only
	$N=\lceil\tau_0/\eta\rceil$ mini-batches are involved. As $Bp_{\min}\to\infty$, multinomial concentration and a union bound over these finitely many steps give
	\[
	\max_{t<N}\|Q_tP^{-1}-I\|_{\mathrm{op}}
	\xrightarrow{\mathbb P}0.
	\]
	Uniform fixed-rank smoothness of the polar map on a tube around the corresponding full-batch Euler trajectory then implies that the stochastic recursion converges uniformly over these $N$ steps to
	\[
	V_{t+1}=V_t+\eta F(V_t),
	\qquad V_0=0.
	\]
	Finally, sending $\eta\downarrow0$, the piecewise-linear Euler interpolation converges uniformly on $[0,\tau_0]$ to $W^\star$. This proves the iterated well-sampled/small-step limit without any coupling between $B$ and $\eta$.
	
	For completeness, we now prove the convenient joint sufficient scaling stated in the proposition. Let
	\[
	N_n:=\left\lceil\frac{\tau_0}{\eta_n}\right\rceil,
	\qquad
	\kappa_n=B_np_{\min}.
	\]
	For a single mini-batch, multiplicative Chernoff bounds give, for $0<\delta<1$,
	\begin{equation}
		\mathbb P\!\left(
		\max_j\frac{|q_j-p_j|}{p_j}>\delta
		\right)
		\le
		2K\exp(-c\delta^2\kappa_n)
		\label{eq:app-uniform-multinomial-one-step}
	\end{equation}
	for a universal constant $c>0$. Set
	\[
	\delta_n:=
	\left(
	\frac{\log(2KN_n)}{\kappa_n}
	\right)^{1/4}.
	\]
	The assumption $\log(1/\eta_n)=o(\kappa_n)$ implies $\delta_n\to0$ and
	\[
	N_n\,2K\exp(-c\delta_n^2\kappa_n)\to0.
	\]
	Hence, with probability tending to one,
	\begin{equation}
		\max_{0\le t<N_n}
		\|Q_tP^{-1}-I\|_{\mathrm{op}}
		\le\delta_n.
		\label{eq:app-uniform-multinomial-window}
	\end{equation}
	In particular all task counts are positive on this event for all sufficiently large $n$, so every sampled gradient remains on the same rank-$(K-1)$ manifold as the full-batch gradient.
	
	Let $V_{t+1}=V_t+\eta_nF(V_t)$, $V_0=0$, denote the deterministic Euler discretization of the continuum flow. Uniform fixed-rank smoothness on the tube around $W^\star([0,\tau_0])$ gives constants $L_{\mathrm{dyn}},a_{\mathrm{samp}}<\infty$ such that, on the event in equation~\ref{eq:app-uniform-multinomial-window},
	\begin{align}
		&\|\operatorname{msgn}((I-S(W))Q_t)-F(W)\|_F
		\le a_{\mathrm{samp}}\delta_n,\label{eq:app-polar-frequency-uniform}\\
		&\|F(W)-F(V)\|_F\le L_{\mathrm{dyn}}\|W-V\|_F.\notag
	\end{align}
	A standard bootstrap then keeps both discrete trajectories inside the same tube and yields
	\[
	\max_{t\le N_n}
	\|\overline W_t-V_t\|_F
	\le
	a_{\mathrm{samp}}\tau_0e^{L_{\mathrm{dyn}}\tau_0}\delta_n
	=o_{\mathbb P}(1).
	\]
	Because $F$ is Lipschitz on the same compact tube, the usual Euler estimate gives
	\[
	\sup_{0\le\tau\le\tau_0}
	\|V^{\eta_n}(\tau)-W^\star(\tau)\|_F
	=O(\eta_n),
	\]
	where $V^{\eta_n}$ denotes the piecewise-linear interpolation. Combining the two estimates proves \eqref{eq:hierarchical-fluid-limit}.
	
	Finally, \eqref{eq:hierarchical-continuum-contrast} implies for every fixed $\tau>0$,
	\[
	\frac1\tau
	U_\mu^\top\overline W^{\eta_n,B_n}(\tau)U_\mu
	\xrightarrow{\mathbb P}I_{C-1}.
	\]
	Hence also
	\begin{equation}
		\lim_{\tau\to\infty}\;\lim_{n\to\infty}
		\frac1\tau
		U_\mu^\top\overline W^{\eta_n,B_n}(\tau)U_\mu
		=I_{C-1}
		\quad\text{in probability}.
		\label{eq:hierarchical-long-time-rate-app}
	\end{equation}
	This is an iterated long-time statement rather than a claim of uniform stochastic approximation on the entire half-line: the nonzero singular gap of the gradient may shrink as training proceeds. The important distinction from the previous finite-step bound is that the number of discrete optimizer steps already diverges as $1/\eta_n$ on every fixed continuum interval, while the limiting dynamics themselves are global in $\tau$ and have an exact asymptotic contrast rate.
	
	\section{Stochastic Labels: Exact Statistics, Projected Regimes, and Local Closure}
	\label{app:label-noise}
	
	We now derive the stochastic-gradient structure under symmetric label noise, assuming \(0<\alpha<1\).
	
	\paragraph{Roadmap and status of the results.}
	Appendices~E.1--E.3 concern the complete categorical model and are exact: they give the gradient decomposition, covariance/reliability formulas, and the fixed-dimensional local Gaussian limit. Appendix~E.4 introduces the longitudinal target-versus-average-nontarget projection; the counting formulas are exact \emph{within that reduced model} and describe the occurrence-limited and corruption-sparse regimes. Appendices~E.5--E.6 add a moment-matched isotropic covariance closure (and a local frequency-homogeneous sector) to obtain the explicit \(\sqrt p\) coefficient in the corruption-resolved local regime. The exact categorical covariance is anisotropic, so the closure result should not be read as an exact formula for the full matrix process.
	
	Define
	\[
	c_\alpha
	:=
	1-\alpha+\frac{\alpha}{K},
	\qquad
	\rho_\alpha
	:=
	1-c_\alpha
	=
	\alpha\frac{K-1}{K}.
	\]
	Thus \(c_\alpha\) is the probability that the observed label equals the target label and
	\(\rho_\alpha\) is the probability that the observed label differs from the target label. The
	transition matrix is
	\[
	D_\alpha
	=
	(1-\alpha)I_K
	+
	\frac{\alpha}{K}\mathbf 1\mathbf 1^\top.
	\]
	
	\subsection{Complete categorical model: gradient dynamics}
	\label{app:label-gradient}
	
	Let \(J\) denote the sampled task and \(Y\) the observed label. We have
	\[
	\mathbb P(J=j)=p_j,
	\qquad
	\mathbb P(Y=i\mid J=j)=p_{i|j},
	\]
	with
	\[
	p_{i|j}
	=
	\begin{cases}
		c_\alpha, & i=j,     \\
		\alpha/K, & i\neq j.
	\end{cases}
	\]
	The population loss under label noise is
	\[
	\mathcal L(W)
	=
	-\sum_{j=1}^Kp_j
	\sum_{i=1}^Kp_{i|j}
	\log S_{ij}(W).
	\]
	In task coordinates, differentiation gives
	\begin{equation}
		\nabla\mathcal L(W)
		=
		-\widetilde E(D_\alpha-S)PE^\top,
		\qquad
		P:=\operatorname{diag}(p_1,\ldots,p_K).
		\label{eq:app-label-pop-gradient}
	\end{equation}
	
	For a mini-batch of size \(B\), let \(N_{ij}\) be the number of samples
	with task \(j\) and observed label \(i\), and define the empirical joint
	frequency matrix
	\[
	\widehat T_{ij}:=\frac{N_{ij}}{B}.
	\]
	The empirical task frequencies are
	\[
	q_j:=\sum_i\widehat T_{ij},
	\qquad
	Q:=\operatorname{diag}(q_1,\ldots,q_K).
	\]
	The negative stochastic gradient in task coordinates is therefore
	\begin{equation}
		\widehat T-SQ.
		\label{eq:app-label-stochastic-gradient}
	\end{equation}
	If \(q_j>0\), one may equivalently write
	\(\widehat T_{\cdot j}=q_jR_{\cdot j}\), where \(R_{\cdot j}\) is the
	empirical conditional label distribution. When \(q_j=0\), the column of
	\(R\) is immaterial because only \(RQ=\widehat T\) enters the gradient.
	
	A useful decomposition is
	\begin{equation}
		\widehat T-SQ
		=
		(D_\alpha-S)P
		+
		L
		+
		T,
		\label{eq:app-label-gradient-decomposition}
	\end{equation}
	where
	\begin{align}
		L
		& :=
		\widehat T-D_\alpha Q,
		\label{eq:app-label-noise-term} \\
		T
		& :=
		(D_\alpha-S)(Q-P).
		\label{eq:app-task-sampling-term}
	\end{align}
	Here \(L\) is the conditional label-sampling fluctuation and \(T\) is
	the task-frequency fluctuation. At the noisy population optimum
	\[
	S^*=D_\alpha,
	\]
	the task-frequency term $T$ vanishes exactly, so the persistent local stochasticity
	comes from \(L\).
	
	\subsection{Complete categorical model: exact covariance and reliability}
	\label{app:label-covariance}
	
	From this point through the local-response analysis we work in task coordinates; the orthogonal embedding transforms preserve all covariance, norm, and polar-equivariance statements used below.
	
	Let
	\[
	d_j:=D_\alpha e_j,
	\qquad
	\delta_j:=S e_j-d_j,
	\qquad
	\Gamma_j:=\operatorname{diag}(d_j)-d_jd_j^\top.
	\]
	For one observation \((J,Y)\), define the task-\(j\) positive-gradient contribution
	\[
	X_j:=\mathbf 1\{J=j\}(S e_j-e_Y).
	\]
	Then
	\[
	\mathbb E X_j=p_j\delta_j,
	\qquad
	\mathbb E[X_jX_j^\top]
	=p_j(\Gamma_j+\delta_j\delta_j^\top).
	\]
	Averaging \(B\) independent observations gives
	\begin{equation}
		\operatorname{Cov}(\widehat g_j)
		=\frac1B\left[p_j\Gamma_j+p_j(1-p_j)\delta_j\delta_j^\top\right],
		\label{eq:app-full-gradient-column-cov}
	\end{equation}
	and for \(j\neq k\),
	\begin{equation}
		\operatorname{Cov}(\widehat g_j,\widehat g_k)
		=-\frac{p_jp_k}{B}\delta_j\delta_k^\top.
		\label{eq:app-full-gradient-cross-cov}
	\end{equation}
	Therefore
	\begin{equation}
		\mathrm{SNR}_j^2
		=\frac{Bp_j\|\delta_j\|_2^2}
		{\operatorname{tr}\Gamma_j+(1-p_j)\|\delta_j\|_2^2}.
		\label{eq:app-full-gradient-snr}
	\end{equation}
	At \(S^*=D_\alpha\), all \(\delta_j=0\). Writing \(\widehat G:=SQ-\widehat T\) for the positive stochastic gradient, the population Hessian is block diagonal across task columns with blocks \(p_j\Gamma_j\), so
	\begin{equation}
		\operatorname{Cov}\!\left(\operatorname{vec}\widehat G(W^*)\right)
		=\frac1B H^*
		\label{eq:app-cov-hessian-identity}
	\end{equation}
	on the columnwise zero-sum subspace. This likelihood identity motivates a local structured-noise viewpoint, but it does not imply that applying a polar map implements \((H^*)^{-1/2}\).
	
	We now isolate the conditional label-sampling term. Let \(\ell_j:=Le_j\). From \eqref{eq:app-label-noise-term},
	\[
	\ell_j
	=
	\frac{1}{B}\sum_{b=1}^B
	\mathbf 1\{J_b=j\}
	\bigl(e_{Y_b}-d_j\bigr).
	\]
	Hence
	\[
	\mathbb E[\ell_j]=0.
	\]
	Different task columns are uncorrelated, and
	\begin{equation}
		\operatorname{Cov}(\ell_j)
		=
		\frac{p_j}{B}\Gamma_j.
		\label{eq:app-label-column-cov}
	\end{equation}
	In particular,
	\begin{align}
		\operatorname{Var}(L_{jj})
		& =
		\frac{p_j}{B}c_\alpha\rho_\alpha,
		\label{eq:app-label-diagonal-var} \\
		\operatorname{Var}(L_{ij})
		& =
		\frac{p_j}{B}\frac{\alpha}{K}
		\left(1-\frac{\alpha}{K}\right),
		\qquad i\neq j,
		\label{eq:app-label-offdiag-var}
	\end{align}
	and, for \(i\neq k\),
	\[
	\operatorname{Cov}(L_{ij},L_{kj})
	=
	-\frac{p_j}{B}p_{i|j}p_{k|j}.
	\]
	
	The geometry of \(\Gamma_j\) is especially simple on the zero-sum subspace.
	Let
	\[
	P_0:=I_K-\frac{1}{K}\mathbf 1\mathbf 1^\top.
	\]
	A direct calculation gives
	\begin{equation}
		\Gamma_j
		=
		\frac{\alpha}{K}P_0
		+
		\alpha(1-\alpha)
		(P_0e_j)(P_0e_j)^\top.
		\label{eq:app-label-cov-decomposition}
	\end{equation}
	Define the correct-versus-average-incorrect contrast
	\[
	v_j
	:=
	e_j-
	\frac{1}{K-1}\sum_{i\neq j}e_i.
	\]
	The covariance eigenvalue along the normalized version of \(v_j\) is
	\[
	\lambda_{\parallel}=\alpha c_\alpha,
	\]
	whereas the other \(K-2\) zero-sum directions have eigenvalue
	\[
	\lambda_{\perp}=\frac{\alpha}{K}.
	\]
	Equivalently, if \(\zeta_j\) is the coefficient of \(\ell_j\) along the
	unnormalized vector \(v_j\), then
	\begin{equation}
		\operatorname{Var}(\zeta_j)
		=
		\frac{p_j}{B}c_\alpha\rho_\alpha.
		\label{eq:app-label-longitudinal-var}
	\end{equation}
	
	These formulas identify the two basic expected-count scales. A batch contains
	\(Bp_j\) task-\(j\) examples on average and
	\[
	Bp_j\rho_\alpha
	=
	Bp_j\alpha\frac{K-1}{K}
	\]
	incorrectly labeled task-\(j\) examples on average. The latter scale
	resolves the aggregate correct-versus-incorrect fluctuation. If one
	instead requires each individual incorrect class to be well sampled,
	the stronger scale \(Bp_j\alpha/K\) is relevant.
	
	\subsection{Complete categorical model: local Gaussian limit and structured response}
	\label{app:label-fullmatrix-gaussian}
	The covariance calculation above determines the complete local sampling limit before any isotropic closure. The complete categorical model admits a fixed-dimensional limit that retains every zero-sum fluctuation mode; this separates exact local sampling behavior from the auxiliary closures introduced later only to obtain explicit scalar rates.
	Let
	\[
	\mathcal S=\{A\in\mathbb R^{K\times K}:\mathbf 1^\top A=0\}
	\]
	with the Frobenius inner product. Choose the zero-column-mean representative $W^*$ of a population minimizer and let the Hessian operator satisfy
	\[
	\mathcal H^*[A]e_j=p_j\Gamma_jAe_j.
	\]
	\begin{theorem}[Local sampling limit with complete categorical noise]
		\label{thm:fullmatrix-local-gaussian}
		Fix $K$, $p_j>0$, and $0<\alpha<1$. For deterministic $A\in\mathcal S$, let $W_B=W^*+A/\sqrt B$. Then the positive mini-batch gradient obeys
		\begin{equation}
			\sqrt B\,\widehat G_B(W_B)
			\xrightarrow{d}
			\mathcal H^*[A]+Z,
			\label{eq:fullmatrix-local-clt}
		\end{equation}
		where $Z\in\mathcal S$ is centered Gaussian with vectorized covariance $\mathcal H^*$. Moreover,
		\begin{equation}
			\lim_{B\to\infty}
			\mathbb E\,\operatorname{msgn}(\widehat G_B(W_B))
			=
			\mathbb E\,\operatorname{msgn}(\mathcal H^*[A]+Z).
			\label{eq:fullmatrix-local-polar-limit}
		\end{equation}
		The limit keeps $K,p,\alpha$ fixed and is not uniform over vanishing task frequencies or vanishing label noise.
	\end{theorem}
	\begin{proof}
		Smoothness of the finite-dimensional softmax loss gives
		$\sqrt B\,\nabla\mathcal L(W_B)=\mathcal H^*[A]+o(1)$.
		Per-sample gradients are uniformly bounded near $W^*$ and their covariance converges to $\mathcal H^*$, so the triangular-array central limit theorem gives \eqref{eq:fullmatrix-local-clt}. Choose an orthonormal basis $X\in\mathbb R^{K\times(K-1)}$ for the zero-sum subspace. In these coordinates the Gaussian covariance is positive definite, so $X^\top(\mathcal H^*[A]+Z)$ has full row rank almost surely; equivalently, $\mathcal H^*[A]+Z$ has rank $K-1$ almost surely. Every stochastic gradient also lies in $\mathcal S$. Writing any $G\in\mathcal S$ as $G=XB$ gives $\operatorname{msgn}(G)=X\operatorname{msgn}(B)$; the rectangular partial polar map is continuous on the open set of full-row-rank $(K-1)\times K$ matrices. Hence the square partial polar map is continuous at every rank-$(K-1)$ point relative to $\mathcal S$, and the continuous mapping theorem applies on this support. Since every entry of a partial polar factor is bounded by one, convergence in distribution upgrades to convergence of expectations, and positive scale invariance removes the factor $\sqrt B$.
	\end{proof}
	
	The theorem isolates the remaining question: how does the polar map respond to an anisotropic Gaussian law? In orthonormal coordinates on $\mathcal S$, let $Z$ be centered nondegenerate Gaussian with covariance $\Sigma$ and define
	\[
	F_\Sigma(A)=\mathbb E\,\operatorname{msgn}(Z+A).
	\]
	Gaussian density differentiation gives the exact linear-response operator
	\begin{equation}
		DF_\Sigma(0)[A]
		=
		\mathbb E\!\left[
		\operatorname{msgn}(Z)
		\,\bigl\langle
		\Sigma^{-1}\operatorname{vec}_{\mathcal S}A,
		\operatorname{vec}_{\mathcal S}Z
		\bigr\rangle
		\right].
		\label{eq:structured-gaussian-response}
	\end{equation}
	The inverse is taken only on the zero-sum support, not on softmax gauge directions. The same Gaussian-smoothing argument as Appendix~\ref{app:isotropic-muon} implies a cubic fixed-dimensional remainder after this linear term. Unlike the fully isotropic setting, however, no left--right orthogonal symmetry forces \eqref{eq:structured-gaussian-response} to be a scalar multiple of $A$. Anisotropy is therefore necessary for nontrivial directional response in this local Gaussian picture, but is not by itself sufficient to guarantee improved conditioning.
	
	The exact Gaussian limit above retains the full anisotropic response but does not simplify to a scalar closed form. The following subsections therefore study complementary reductions: first a longitudinal counting model that is exact within its projection, and then an isotropic covariance closure that preserves total label-noise energy while sacrificing anisotropy.
	
	\subsection{Projected counting model: task and corruption counts}
	\label{app:label-coverage-dynamics}
	
	We quantify the occurrence-limited and corruption-sparse regimes in a projected target-versus-average-nontarget model. The projection retains the longitudinal component along \(v_j\) and removes the transverse categorical fluctuations in \eqref{eq:app-label-cov-decomposition}. Within this reduced model, the counting argument below is exact; the Gaussian approximation enters only later.
	
	Consider a symmetry-equivalent task direction with frequency \(p\) and margin \(\Delta\). Write
	\[
	u(\Delta)
	:=
	1-S_{jj}
	=
	\frac{K-1}{e^\Delta+K-1}.
	\]
	Let \(N\) be the number of samples of this task in a batch and let \(M\) be the number of those samples whose observed label is incorrect. For each of the \(B\) draws, the three possibilities have probabilities
	\[
	pc_\alpha,
	\qquad
	p\rho_\alpha,
	\qquad
	1-p,
	\]
	corresponding to a correctly labeled occurrence of the task, an incorrectly labeled occurrence, and another task, respectively. Hence
	\begin{equation}
		\mathbb P(N=n,M=m)
		=
		\frac{B!}{(n-m)!\,m!\,(B-n)!}
		(pc_\alpha)^{n-m}
		(p\rho_\alpha)^m
		(1-p)^{B-n},
		\label{eq:app-label-count-law}
	\end{equation}
	for \(0\le m\le n\le B\).
	
	The correct-coordinate residual in the stochastic gradient column is
	\[
	\widehat T_{jj}-S_{jj}q_j
	=
	\frac{N-M-NS_{jj}}{B}.
	\]
	Since the residual column has zero sum, this is also its coefficient along the unnormalized contrast \(v_j\). Therefore the longitudinal random coefficient is
	\begin{equation}
		A(\Delta)
		:=
		\frac{Nu(\Delta)-M}{B}.
		\label{eq:app-label-longitudinal-coefficient}
	\end{equation}
	The corresponding scalar sign response is
	\begin{equation}
		m_{B,p}(\Delta)
		:=
		\mathbb E\!\left[\operatorname{sgn}(A(\Delta))\right]
		=
		\sum_{n=0}^B\sum_{m=0}^n
		\operatorname{sgn}\!\left(nu(\Delta)-m\right)
		\mathbb P(N=n,M=m),
		\label{eq:app-label-exact-sign-response}
	\end{equation}
	with the convention \(\operatorname{sgn}(0)=0\).
	
	To connect this scalar response to the projected Muon dynamics, consider \(s\) exchangeable directions with coefficients \(A_1,\ldots,A_s\), and let
	\[
	X\in\mathbb R^{s\times(s-1)}
	\]
	have orthonormal columns spanning the zero-sum subspace. Define
	\[
	\mathcal A_s(\Delta)
	:=
	X^\top\operatorname{diag}(A_1(\Delta),\ldots,A_s(\Delta))X.
	\]
	Permutation symmetry implies
	\begin{equation}
		\mathbb E\!\left[
		\operatorname{msgn}(\mathcal A_s(\Delta))
		\right]
		=
		r_s(\Delta)I_{s-1}
		\label{eq:app-label-count-projected-response}
	\end{equation}
	for a scalar \(r_s(\Delta)\). By Cauchy interlacing between the eigenvalues of \(\mathcal A_s\) and the diagonal entries \(A_1,\ldots,A_s\),
	\begin{equation}
		\left|
		r_s(\Delta)-m_{B,p}(\Delta)
		\right|
		\le
		\frac{2}{s-1}.
		\label{eq:app-label-count-interlacing-bound}
	\end{equation}
	Thus \(m_{B,p}\) gives the large-sector mean Muon response in the projected model, with an explicit \(O(s^{-1})\) correction.
	
	\paragraph{Sparse-task limit.}
	Suppose \(Bp\to0\). With probability \(1-O((Bp)^2)\), the task appears at most once. Whenever it appears once, an observation matching the target label gives \(A(\Delta)>0\), whereas a mismatched observation gives \(A(\Delta)<0\), for every finite \(\Delta\). Therefore
	\begin{equation}
		m_{B,p}(\Delta)
		=
		Bp(c_\alpha-\rho_\alpha)
		+O\!\left((Bp)^2\right)
		=
		Bp(1-2\rho_\alpha)
		+O\!\left((Bp)^2\right).
		\label{eq:app-label-sparse-response}
	\end{equation}
	This expansion is uniform for \(\Delta\) in compact sets. Hence, for $\rho_\alpha<1/2$, the scalar counting dynamics have positive mean margin drift proportional to $Bp$. For the finite-sector matrix response $r_s$, the additive bound in \eqref{eq:app-label-count-interlacing-bound} alone does not establish the same leading rate as $Bp\to0$; transferring this asymptotic requires additional control, such as $s^{-1}=o(Bp)$.
	
	\paragraph{Large-margin response.}
	We next let \(\Delta\to\infty\), so that \(u(\Delta)\downarrow0\). For each fixed batch realization,
	\[
	\operatorname{sgn}(A(\Delta))
	\longrightarrow
	\mathbf 1\{N>0,M=0\}
	-
	\mathbf 1\{M>0\}.
	\]
	Now
	\[
	\mathbb P(M=0)
	=
	(1-p\rho_\alpha)^B,
	\qquad
	\mathbb P(N=0)
	=
	(1-p)^B.
	\]
	Therefore dominated convergence gives
	\begin{equation}
		\lim_{\Delta\to\infty}
		m_{B,p}(\Delta)
		=
		\chi_{B,p}
		:=
		2(1-p\rho_\alpha)^B
		-(1-p)^B
		-1.
		\label{eq:app-label-large-margin-response}
	\end{equation}
	Since \(u(\Delta)\) decreases with \(\Delta\), every summand in \eqref{eq:app-label-exact-sign-response} is nonincreasing in \(\Delta\). Thus \(m_{B,p}(\Delta)\) is nonincreasing and \(\chi_{B,p}\) is its minimum tail value.
	
	Consider the large-sector projected mean-field recursion
	\begin{equation}
		\Delta^{\mathrm{mf}}_{t+1}
		-
		\Delta^{\mathrm{mf}}_t
		=
		\eta m_{B,p}(\Delta^{\mathrm{mf}}_t).
		\label{eq:app-label-mean-field-recursion}
	\end{equation}
	If \(\chi_{B,p}>0\), then \(m_{B,p}(\Delta)\ge\chi_{B,p}>0\) for every finite \(\Delta\), so the limiting mean-field recursion in equation~\ref{eq:app-label-mean-field-recursion} satisfies \(\Delta_t^{\mathrm{mf}}\to\infty\). Since \(m_{B,p}(\Delta)\to\chi_{B,p}\) along that trajectory,
	\begin{equation}
		\frac{\Delta_t^{\mathrm{mf}}}{\eta t}
		\longrightarrow
		\chi_{B,p}.
		\label{eq:app-label-tail-linear-rate}
	\end{equation}
	For a finite projected sector, \eqref{eq:app-label-count-interlacing-bound} additionally guarantees positive mean drift whenever $2/(s-1)<\chi_{B,p}$. Consequently,
	\[
	1-S_{jj}(t)
	\asymp
	(K-1)e^{-\eta\chi_{B,p}t}
	\]
	along this target-interpolating mean-field trajectory.
	
	\paragraph{Corruption-sparse window and corruption threshold.}
	We use \(\ll\) and \(\gg\) in this paragraph as scale-separation notation. For fixed nonzero \(\rho_\alpha\), the corruption-sparse case is the intermediate window
	\[
	1\ll Bp\ll\rho_\alpha^{-1}.
	\]
	Within this window, \((1-p)^B\le e^{-Bp}\) and a binomial expansion gives
	\begin{equation}
		\chi_{B,p}
		=
		1-2Bp\rho_\alpha
		+O\!\left(
		(Bp\rho_\alpha)^2+e^{-Bp}
		\right).
		\label{eq:app-label-corruption-sparse-response}
	\end{equation}
	Thus the projected large-margin growth rate is close to the unit deterministic-label full-batch rate and is no longer occurrence-limited.
	
	A sharper crossover is obtained in the Poisson scaling
	\[
	Bp\to\infty,
	\qquad
	p\rho_\alpha\to0,
	\qquad
	Bp\rho_\alpha\to\lambda.
	\]
	Then
	\begin{equation}
		\chi_{B,p}
		\longrightarrow
		2e^{-\lambda}-1.
		\label{eq:app-label-poisson-tail-response}
	\end{equation}
	Hence the large-margin regime has positive drift for \(\lambda<\log2\) and negative drift for \(\lambda>\log2\). Equivalently, once task absence is negligible, the crossover batch scale is
	\begin{equation}
		B_{\mathrm{noise}}
		\simeq
		\frac{\log2}{p\rho_\alpha}.
		\label{eq:app-label-noise-threshold}
	\end{equation}
	The regime \(Bp\gg1\) but \(Bp\rho_\alpha\ll1\) therefore has a quantitatively stable large-margin regime in the reduced model, while resolved corruption reverses the projected large-margin drift. The finite population optimum exists for every $0<\alpha<1$, independently of batch size; the local response around it is analyzed separately below.
	
	\subsection{Moment-matched isotropic covariance closure}
	\label{app:random-spectrum}
	
	A Gaussian proxy with the exact column covariance of \(L\) can be represented as
	\begin{equation}
		L_{\mathrm G}
		=
		\frac{1}{\sqrt B}
		P_0
		\left[
		\sqrt{\frac{\alpha}{K}}Z
		+
		\sqrt{\alpha(1-\alpha)}\operatorname{diag}(g)
		\right]
		P^{1/2},
		\label{eq:app-label-gaussian-representation}
	\end{equation}
	where \(Z\) has i.i.d. standard Gaussian entries and \(g\) has i.i.d.
	standard Gaussian coordinates, independently. This proxy makes the anisotropic categorical covariance explicit: it contains both a projected dense component and a projected diagonal component.
	
	For a closed-form polar response, we now replace this anisotropic covariance by a moment-matched isotropic covariance on the zero-sum subspace. The exact covariance in \eqref{eq:app-label-cov-decomposition} has one longitudinal eigenvalue \(\alpha c_\alpha\) and \(K-2\) transverse eigenvalues \(\alpha/K\); the closure preserves their total covariance energy while removing the directional anisotropy. This is the additional approximation that allows us to reuse the Gaussian polar response from Section~\ref{sec:isotropic}.
	
	From \eqref{eq:app-label-cov-decomposition},
	\begin{equation}
		\operatorname{tr}(\Gamma_j)
		=
		\alpha(2-\alpha)\frac{K-1}{K}.
		\label{eq:app-label-cov-trace}
	\end{equation}
	The zero-sum subspace has dimension \(K-1\). We therefore define
	\begin{equation}
		\tau_\alpha^2
		:=
		\frac{\operatorname{tr}(\Gamma_j)}{K-1}
		=
		\frac{\alpha(2-\alpha)}{K},
		\label{eq:app-label-isotropic-variance}
	\end{equation}
	and make the moment-matched covariance approximation
	\begin{equation}
		\Gamma_j
		\approx
		\tau_\alpha^2 P_0.
		\label{eq:app-label-cov-isotropization}
	\end{equation}
	This approximation preserves the total label-noise variance on the
	zero-sum subspace but removes the longitudinal anisotropy.
	
	Let
	\[
	X\in\mathbb R^{K\times(K-1)}
	\]
	have orthonormal columns spanning the zero-sum subspace, so that
	\[
	X^\top X=I_{K-1},
	\qquad
	XX^\top=P_0.
	\]
	Since every stochastic-gradient column sums to zero, we may write
	\[
	G=X\widehat G,
	\qquad
	\widehat G:=X^\top G\in\mathbb R^{(K-1)\times K}.
	\]
	The partial polar factors are related by
	\begin{equation}
		\operatorname{msgn}(G)
		=
		X\operatorname{msgn}(\widehat G).
		\label{eq:app-label-polar-projection}
	\end{equation}
	
	For an explicit closed-form response, we additionally use a local
	frequency-homogeneous approximation: the symmetry-equivalent columns
	under consideration are assigned a common frequency scale \(p\). This is a local spectral proxy, not an exact reduction of a heterogeneous matrix: the polar factor couples columns, and their unequal noise scales generally break right-orthogonal symmetry. Varying \(p\) compares proxies at different local frequency scales; it does not vary the common probability of all \(K\) tasks in one normalized uniform distribution.
	Under \eqref{eq:app-label-cov-isotropization}, the projected Gaussian
	label noise is then approximated by
	\begin{equation}
		\widehat L
		\approx
		\sigma_{p,\alpha}\Xi,
		\qquad
		\sigma_{p,\alpha}^2
		:=
		\frac{p\tau_\alpha^2}{B}
		=
		\frac{p\alpha(2-\alpha)}{BK},
		\label{eq:app-label-isotropic-noise-scale}
	\end{equation}
	where
	\[
	\Xi\in\mathbb R^{(K-1)\times K}
	\]
	has i.i.d.\ standard Gaussian entries.
	
	We use the rectangular analogue of
	Proposition~\ref{prop:gaussian-isotropization}. For a rectangular matrix \(A=U\Sigma V^\top\) with singular values \(\sigma_i\), we use the same partial-polar definition
	\[
	\operatorname{msgn}(A)
	:=
	U\operatorname{diag}(\mathbf 1\{\sigma_i>0\})V^\top.
	\]
	Thus positive singular values are mapped to one and zero singular values remain zero. If
	\(\Xi\in\mathbb R^{m\times n}\) has i.i.d.\ standard Gaussian entries,
	define
	\begin{equation}
		c_{m,n}
		:=
		\frac{1}{mn}
		\mathbb E\|\Xi\|_*.
		\label{eq:app-rectangular-cmn}
	\end{equation}
	For fixed \(m,n\), Gaussian convolution and
	\(O(m)\times O(n)\) equivariance give
	\begin{equation}
		\mathbb E[
		\operatorname{msgn}(H+\sigma\Xi)
		]
		=
		\frac{c_{m,n}}{\sigma}H
		+
		O_{m,n}\!\left(
		\frac{\|H\|_F^3}{\sigma^3}
		\right)
		\label{eq:app-rectangular-isotropic-response}
	\end{equation}
	as \(\|H\|_F/\sigma\to0\). The coefficient follows by taking the
	trace of the derivative:
	\[
	mn\,c_{m,n}
	=
	\mathbb E
	\langle
	\operatorname{msgn}(\Xi),\Xi
	\rangle_F
	=
	\mathbb E\|\Xi\|_*.
	\]
	Thus \eqref{eq:app-rectangular-isotropic-response} is the direct
	rectangular counterpart of the square Gaussian response derived in
	Appendix~\ref{app:isotropic-muon}.
	
	\subsection{Local Muon rate under the covariance closure}
	\label{app:label-dynamics}
	
	Under the equal-incorrect-logit reduction in a symmetry-equivalent sector, define
	\[
	\Delta_j
	:=
	\overline W_{jj}
	-
	\frac{1}{K-1}
	\sum_{i\neq j}\overline W_{ij}.
	\]
	Then
	\[
	S_{jj}
	=
	\frac{e^{\Delta_j}}{e^{\Delta_j}+K-1}.
	\]
	The noisy population optimum satisfies
	\[
	S_{jj}^*
	=
	c_\alpha,
	\]
	which gives
	\begin{equation}
		\Delta_j^*
		=
		\log\left(
		\frac{K}{\alpha}-(K-1)
		\right).
		\label{eq:app-label-optimal-margin}
	\end{equation}
	
	Define
	\[
	r_j
	:=
	c_\alpha-S_{jj}.
	\]
	Under the equal-incorrect-logit reduction,
	\begin{equation}
		(D_\alpha-S)e_j
		=
		r_jv_j,
		\qquad
		v_j
		=
		e_j-\frac{1}{K-1}\sum_{i\neq j}e_i.
		\label{eq:app-label-signal-direction}
	\end{equation}
	For a task of local frequency scale \(p\), the projected deterministic
	signal is therefore
	\begin{equation}
		\widehat H_j
		=
		pr_j
		x_j e_j^\top,
		\qquad
		x_j:=X^\top v_j.
		\label{eq:app-label-projected-signal}
	\end{equation}
	Because \(v_j\) lies in the zero-sum subspace,
	\begin{equation}
		\|x_j\|_2^2
		=
		\|v_j\|_2^2
		=
		\frac{K}{K-1}.
		\label{eq:app-label-vnorm}
	\end{equation}
	
	Combining equations~\ref{eq:app-label-isotropic-noise-scale} and~\ref{eq:app-rectangular-isotropic-response} with
	\(m=K-1\) and \(n=K\) gives
	\begin{equation}
		\mathbb E[
		\operatorname{msgn}(\widehat H_j+\widehat L)
		]
		\approx
		\frac{c_{K-1,K}}{\sigma_{p,\alpha}}
		\widehat H_j
		\label{eq:app-label-local-polar-response}
	\end{equation}
	in the small-signal regime near the population optimum under label noise.
	
	The margin functional is
	\[
	\Delta_j
	=
	v_j^\top\overline W e_j.
	\]
	Using equations~\ref{eq:app-label-polar-projection}, \ref{eq:app-label-projected-signal}, and~\ref{eq:app-label-vnorm}, the expected Muon margin increment becomes
	\begin{align}
		\mathbb E[
		\Delta_j(t+1)-\Delta_j(t)
		\mid\Delta_j(t)
		]
		&\approx
		\eta
		\frac{c_{K-1,K}}{\sigma_{p,\alpha}}
		pr_j
		\frac{K}{K-1}.
		\label{eq:app-label-margin-before-linearization}
	\end{align}
	
	Finally,
	\[
	\frac{\partial S_{jj}}{\partial\Delta_j}
	=
	S_{jj}(1-S_{jj}),
	\]
	and at the optimum,
	\[
	S_{jj}^*(1-S_{jj}^*)
	=
	c_\alpha\rho_\alpha.
	\]
	Hence
	\begin{equation}
		r_j
		=
		c_\alpha\rho_\alpha
		\bigl(\Delta_j^*-\Delta_j\bigr)
		+
		O\!\left(
		(\Delta_j-\Delta_j^*)^2
		\right).
		\label{eq:app-label-delta-linearization}
	\end{equation}
	Since
	\[
	\rho_\alpha
	=
	\alpha\frac{K-1}{K},
	\]
	substituting equations~\ref{eq:app-label-isotropic-noise-scale} and~\ref{eq:app-label-delta-linearization} into equation~\ref{eq:app-label-margin-before-linearization} yields
	\begin{equation}
		\mathbb E[
		\Delta_j(t+1)-\Delta_j(t)
		\mid\Delta_j(t)
		]
		\approx
		\eta\,
		c_{K-1,K}
		\sqrt{BKp}\,
		\frac{c_\alpha\sqrt{\alpha}}{\sqrt{2-\alpha}}
		\bigl(
		\Delta_j^*-\Delta_j(t)
		\bigr).
		\label{eq:app-label-explicit-rate-finite-K}
	\end{equation}
	Thus the covariance and locally homogeneous frequency approximations predict a local Muon conditional-mean contraction coefficient proportional to \(\sqrt p\).
	
	For comparison, under the same symmetric-logit reduction, the conditional mean SGD update obeys
	\[
	\mathbb E[\Delta_j(t+1)-\Delta_j(t)\mid W_t]
	=
	\eta\frac{K}{K-1}p_jr_j(t),
	\]
	so its exact local conditional-mean contraction coefficient is
	\begin{equation}
		\eta\frac{K}{K-1}
		p_jc_\alpha\rho_\alpha,
		\label{eq:app-label-sgd-rate}
	\end{equation}
	which scales linearly with \(p_j\).
	
	For large \(K\), the quarter-circle law gives
	\begin{equation}
		c_{K-1,K}\sqrt K
		\longrightarrow
		\frac{8}{3\pi},
		\label{eq:app-label-cmn-largeK}
	\end{equation}
	while \(c_\alpha\to1-\alpha\). Therefore
	\eqref{eq:app-label-explicit-rate-finite-K} reduces to
	\begin{equation}
		\mathbb E[
		\Delta_j(t+1)-\Delta_j(t)
		\mid\Delta_j(t)
		]
		\approx
		\eta\,
		\frac{8}{3\pi}
		\frac{(1-\alpha)\sqrt{\alpha}}{\sqrt{2-\alpha}}
		\sqrt{Bp}\,
		\bigl(
		\Delta_j^*-\Delta_j(t)
		\bigr).
		\label{eq:app-label-explicit-rate-large-K}
	\end{equation}
	
	\paragraph{Scope of the approximation.}
	The covariance formulas in equations~\ref{eq:app-label-column-cov}--\ref{eq:app-label-cov-decomposition} are exact. The explicit rate
	\eqref{eq:app-label-explicit-rate-finite-K} uses two additional
	approximations: the anisotropic categorical covariance is replaced by
	the isotropic covariance
	\eqref{eq:app-label-cov-isotropization} with matching trace, and the
	columns in the local spectral sector are assigned a common frequency
	scale \(p\). These approximations restore the left-right isotropy needed
	for a closed-form polar response while retaining the total label-noise
	energy. They are intended for the corruption-resolved local regime; the
	occurrence-limited and corruption-sparse regimes are instead described by the
	counting analysis in Appendix~\ref{app:label-coverage-dynamics}.
	
	\section{Muon--SignSGD Sparse-Support Contrast: Complete Derivation}
	\label{app:two-task-signsgd}
	
	This appendix derives the two-task SignSGD dynamics and makes explicit
	how the result depends on rotations of the embedding space. We use the
	elementwise convention \(\operatorname{sgn}(0)=0\).
	
	\subsection{Rank-one reduction}
	
	In task coordinates, the two-task negative stochastic gradient is
	\[
	vg_t^\top,
	\qquad
	v=
	\begin{bmatrix}
		1 \\
		-1
	\end{bmatrix},
	\qquad
	g_t=
	\begin{bmatrix}
		q_1(t)u_1(t) \\
		-q_2(t)u_2(t)
	\end{bmatrix}.
	\]
	Hence, in the original parameter coordinates, the negative gradient is
	\[
	\widetilde Evg_t^\top E^\top =
	(\widetilde Ev)(Eg_t)^\top.
	\]
	For arbitrary vectors \(a,b\),
	\[
	\operatorname{sgn}(ab^\top) =
	\operatorname{sgn}(a)\operatorname{sgn}(b)^\top,
	\]
	where the sign is applied elementwise. Therefore the SignSGD update in
	task coordinates is
	\[
	\overline W_{t+1}-\overline W_t =
	\eta
	\widetilde E^\top
	\operatorname{sgn}(\widetilde Ev)
	\operatorname{sgn}(Eg_t)^\top
	E.
	\]
	
	Define the margin vector
	\[
	h_t
	:=
	\begin{bmatrix}
		\Delta_1(t) \\
		-\Delta_2(t)
	\end{bmatrix} =
	\overline W_t^\top v.
	\]
	Then
	\begin{align*}
		h_{t+1}-h_t
		& =
		\eta
		E^\top
		\operatorname{sgn}(Eg_t)
		\operatorname{sgn}(\widetilde Ev)^\top
		\widetilde Ev \\
		& =
		\eta\widetilde c\,
		E^\top\operatorname{sgn}(Eg_t),
	\end{align*}
	where
	\begin{equation}
		\widetilde c
		:=
		\operatorname{sgn}(\widetilde Ev)^\top\widetilde Ev =
		\|\widetilde Ev\|_1 >0.
		\label{eq:app-signsgd-ctilde}
	\end{equation}
	Thus the answer embedding contributes only the positive scalar
	\(\widetilde c\) to the two margins, whereas the query embedding
	determines their coupling.
	
	\subsection{Rotation-dependent margin recursion}
	
	Consider the representative rotation
	\[
	E=
	\begin{bmatrix}
		c  & s \\
		-s & c
	\end{bmatrix},
	\qquad
	c:=\cos\theta,
	\qquad
	s:=\sin\theta,
	\qquad
	0<\theta<\frac{\pi}{2}.
	\]
	Then
	\[
	Eg_t =
	\begin{bmatrix}
		cq_1u_1-sq_2u_2 \\
		-sq_1u_1-cq_2u_2
	\end{bmatrix}.
	\]
	The second coordinate is strictly negative because
	\(q_1+q_2=1\) and \(u_1,u_2>0\). Define
	\begin{equation}
		\xi_t
		:=
		\operatorname{sgn}
		\left(
		cq_1(t)u_1(t)-sq_2(t)u_2(t)
		\right).
		\label{eq:app-signsgd-xi}
	\end{equation}
	Hence
	\[
	\operatorname{sgn}(Eg_t) =
	\begin{bmatrix}
		\xi_t \\
		-1
	\end{bmatrix},
	\]
	and
	\[
	E^\top
	\operatorname{sgn}(Eg_t) =
	\begin{bmatrix}
		c\xi_t+s \\
		s\xi_t-c
	\end{bmatrix}.
	\]
	Since the second entry of \(h_t\) is \(-\Delta_2(t)\), the exact margin
	recursion is
	\begin{align}
		\Delta_1(t+1)-\Delta_1(t)
		& =
		\eta\widetilde c(c\xi_t+s),
		\label{eq:signsgd-margin-1} \\
		\Delta_2(t+1)-\Delta_2(t)
		& =
		\eta\widetilde c(c-s\xi_t).
		\label{eq:signsgd-margin-2}
	\end{align}
	Equivalently, for
	\[
	D_t:=\Delta_1(t)-\Delta_2(t),
	\]
	we have
	\begin{equation}
		D_{t+1}-D_t =
		\eta\widetilde c
		\left(
		(c+s)\xi_t+s-c
		\right).
		\label{eq:signsgd-gap}
	\end{equation}
	In particular,
	\[
	\xi_t=1
	\quad\Longrightarrow\quad
	D_{t+1}-D_t =
	2\eta\widetilde c\,s,
	\]
	and
	\[
	\xi_t=-1
	\quad\Longrightarrow\quad
	D_{t+1}-D_t =
	-2\eta\widetilde c\,c.
	\]
	If \(\xi_t=0\), then
	\[
	D_{t+1}-D_t =
	\eta\widetilde c(s-c).
	\]
	
	\subsection{Full-batch reference dynamics}
	
	Set \(q_i(t)=p_i\). Define the sum and gap
	\[
	A_t:=\Delta_1(t)+\Delta_2(t),
	\qquad
	D_t:=\Delta_1(t)-\Delta_2(t).
	\]
	From equations~\ref{eq:signsgd-margin-1}--\ref{eq:signsgd-margin-2},
	\[
	A_{t+1}-A_t =
	\begin{cases}
		2\eta\widetilde c\,c,  & \xi_t=1,  \\
		\eta\widetilde c(c+s), & \xi_t=0,  \\
		2\eta\widetilde c\,s,  & \xi_t=-1.
	\end{cases}
	\]
	Therefore
	\begin{equation}
		A_t
		\ge
		2\eta\widetilde c
		\min\{c,s\}\,t,
		\label{eq:app-signsgd-sum-growth}
	\end{equation}
	so \(A_t\to\infty\).
	
	For fixed \(A\), write
	\[
	\Delta_1=\frac{A+D}{2},
	\qquad
	\Delta_2=\frac{A-D}{2}.
	\]
	The switching variable in \eqref{eq:app-signsgd-xi} has the same sign
	as
	\begin{equation}
		\Lambda(A,D)
		:=
		\log\frac{p_1c}{p_2s}
		+
		\log\left(1+e^{(A-D)/2}\right)
		-
		\log\left(1+e^{(A+D)/2}\right).
		\label{eq:app-signsgd-switching-function}
	\end{equation}
	For each fixed \(A\), \(\Lambda(A,D)\) is strictly decreasing in \(D\),
	with limits \(+\infty\) and \(-\infty\) as \(D\to-\infty\) and
	\(D\to+\infty\), respectively. Hence there is a unique switching
	threshold \(T(A)\) satisfying
	\[
	\Lambda(A,T(A))=0.
	\]
	Moreover,
	\begin{equation}
		T(A)
		\longrightarrow
		D_\theta
		:=
		\log\frac{p_1}{p_2}
		+
		\log\cot\theta
		\qquad
		\text{as }A\to\infty.
		\label{eq:app-signsgd-threshold-limit}
	\end{equation}
	
	On either strict side of the switching threshold, the recursion moves the
	gap toward that threshold:
	\[
	D_t<T(A_t)
	\quad\Longrightarrow\quad
	\xi_t=1,
	\]
	while
	\[
	D_t>T(A_t)
	\quad\Longrightarrow\quad
	\xi_t=-1.
	\]
	At equality, \(\xi_t=0\) and
	\[
	D_{t+1}-D_t=\eta\widetilde c(s-c),
	\]
	so the one-step displacement is uniformly bounded. Since \(T(A)\) is
	bounded for \(A\ge0\), these one-step bounds imply that \(D_t\) is
	bounded. Combining this with \eqref{eq:app-signsgd-sum-growth} gives
	\[
	\Delta_1(t)\to\infty,
	\qquad
	\Delta_2(t)\to\infty.
	\]
	Consequently, \eqref{eq:app-signsgd-threshold-limit} applies along the
	trajectory.
	
	More precisely, for every \(\varepsilon>0\), for all sufficiently large
	\(t\),
	\[
	D_\theta-\varepsilon
	\le
	T(A_t)
	\le
	D_\theta+\varepsilon.
	\]
	An upward crossing of the threshold changes \(D_t\) by at most
	\(2\eta\widetilde c\,s\), while a downward crossing changes it by at
	most \(2\eta\widetilde c\,c\). Therefore
	\begin{equation}
		D_\theta-2\eta\widetilde c\,c
		\le
		\liminf_{t\to\infty}D_t
		\le
		\limsup_{t\to\infty}D_t
		\le
		D_\theta+2\eta\widetilde c\,s.
		\label{eq:app-signsgd-switching-band}
	\end{equation}
	Thus the discrete full-batch gap need not converge to a single limit;
	rather, it remains in an \(O(\eta)\) switching band around the
	embedding-dependent location \(D_\theta\). In the small-step limit this
	band collapses to \(D_\theta\).
	
	We next compute the asymptotic growth rates. Rearranging
	\eqref{eq:signsgd-gap} gives
	\[
	\frac{1}{t}\sum_{k=0}^{t-1}\xi_k =
	\frac{c-s}{c+s}
	+
	\frac{D_t-D_0}
	{\eta\widetilde c(c+s)t}.
	\]
	Since \(D_t\) is bounded,
	\begin{equation}
		\frac{1}{t}\sum_{k=0}^{t-1}\xi_k
		\longrightarrow
		\frac{c-s}{c+s}.
		\label{eq:app-signsgd-average-xi}
	\end{equation}
	Averaging equations~\ref{eq:signsgd-margin-1} and~\ref{eq:signsgd-margin-2} then yields
	\begin{equation}
		\frac{\Delta_1(t)}{\eta t}
		\longrightarrow
		\frac{\widetilde c}{c+s},
		\qquad
		\frac{\Delta_2(t)}{\eta t}
		\longrightarrow
		\frac{\widetilde c}{c+s}.
		\label{eq:app-signsgd-fullbatch-rates}
	\end{equation}
	Hence full-batch SignSGD balances the two tasks at leading order, but
	both the switching location and the common growth rate depend on the
	embedding rotation.
	
	\subsection{Sparse-support negative transfer}
	
	Consider the event
	\[
	\mathcal A_t
	:=
	\{q_2(t)=0\},
	\]
	meaning that the batch contains no examples from task~2. Since every
	sample is then task~1,
	\[
	\mathbb P(\mathcal A_t)=p_1^B.
	\]
	On \(\mathcal A_t\),
	\[
	\xi_t=1,
	\]
	and hence
	\begin{equation}
		\Delta_2(t+1)-\Delta_2(t) =
		\eta\widetilde c(c-s).
		\label{eq:signsgd-unseen}
	\end{equation}
	Therefore, whenever
	\[
	\theta>\frac{\pi}{4},
	\]
	a batch containing no task~2 examples decreases the task~2 margin.
	
	A stronger statement follows by bounding all remaining updates. For
	every batch,
	\[
	\Delta_2(t+1)-\Delta_2(t)
	\le
	\eta\widetilde c(c+s).
	\]
	Consequently,
	\[
	\Delta_2(t)
	\le
	\Delta_2(0)
	+
	\eta\widetilde c
	\left[
	(c-s)\sum_{k<t}\mathbf 1_{\mathcal A_k}
	+
	(c+s)\sum_{k<t}\mathbf 1_{\mathcal A_k^c}
	\right].
	\]
	By the strong law of large numbers,
	\[
	\frac{1}{t}
	\sum_{k<t}\mathbf 1_{\mathcal A_k}
	\longrightarrow
	p_1^B
	\qquad
	\text{a.s.}
	\]
	Therefore
	\begin{equation}
		\limsup_{t\to\infty}
		\frac{\Delta_2(t)}
		{\eta\widetilde c\,t}
		\le
		c+s\left(1-2p_1^B\right)
		\qquad
		\text{a.s.}
		\label{eq:signsgd-negative}
	\end{equation}
	If
	\begin{equation}
		p_1^B >
		\frac{1+\cot\theta}{2},
		\label{eq:app-signsgd-linear-degradation-condition}
	\end{equation}
	then the right-hand side of \eqref{eq:signsgd-negative} is strictly
	negative, and therefore
	\[
	\Delta_2(t)\to-\infty
	\]
	linearly almost surely. Condition~(\ref{eq:app-signsgd-linear-degradation-condition}) can only
	hold when \(\theta>\pi/4\). In this regime,
	\eqref{eq:signsgd-margin-1} also gives
	\[
	\Delta_1(t+1)-\Delta_1(t)
	\ge
	\eta\widetilde c(s-c)>0.
	\]
	Thus insufficiently small relative task-frequency fluctuations can be converted by the embedding
	rotation into persistent negative transfer from the frequent task to
	the rare task.
	
	For Muon, orthogonal equivariance allows the dynamics to be analyzed
	directly in task coordinates. When \(q_2(t)=0\), the exact two-task
	Muon recursion in equation~\ref{eq:app-two-task-margins} gives
	\[
	\Delta_2(t+1)-\Delta_2(t)=0.
	\]
	The rotation-induced negative-transfer mechanism above is therefore
	specific to the coordinate-sensitive elementwise sign operation.
	
	\section{Supplementary EMA Bias--Variance Calculation}
	\label{app:noise-lag}
	
	This optional scalar calculation makes the familiar EMA bias--variance trade-off explicit for an exogenous affine drifting signal. It is included only to interpret the tracking side of Proposition~\ref{prop:temporal-aggregation}; it is not a separate Muon convergence result or a method-level contribution. For reference, write
	\begin{equation}
		R(\beta;v,b)
		:=
		v\frac{1-\beta}{1+\beta}
		+
		b^2\frac{\beta^2}{(1-\beta)^2},
		\label{eq:noise-lag-risk}
	\end{equation}
	and denote the first-order optimality condition by
	\begin{equation}
		b^2\beta(1+\beta)^2=v(1-\beta)^3.
		\label{eq:noise-lag-optimality}
	\end{equation}
	
	\subsection{Affine-drift calculation}
	Let $y_t=a_t+\varepsilon_t$, where $a_t=a_0+bt$ and the noise variables are independent with mean zero and variance $v>0$. Define $m_t=(1-\beta)\sum_{r\ge0}\beta^r y_{t-r}$ for $0\le\beta<1$. The deterministic weighted sum converges absolutely, and the noise weights are square summable, so the infinite-past EMA is well defined in mean square. Its mean is
	\[
	\mathbb E m_t
	=(1-\beta)\sum_{r\ge0}\beta^r(a_t-br)
	=a_t-\frac{\beta}{1-\beta}b,
	\]
	while independence gives
	\[
	\operatorname{Var}(m_t)
	=v(1-\beta)^2\sum_{r\ge0}\beta^{2r}
	=v\frac{1-\beta}{1+\beta}.
	\]
	Adding squared bias and variance proves \eqref{eq:noise-lag-risk}. For $0<\beta<1$,
	\[
	\partial_\beta R
	=-\frac{2v}{(1+\beta)^2}
	+\frac{2b^2\beta}{(1-\beta)^3}.
	\]
	The function $\beta(1+\beta)^2/(1-\beta)^3$ increases strictly from zero to infinity, so for $b\neq0$ the derivative has exactly one root and changes sign from negative to positive. This gives \eqref{eq:noise-lag-optimality}. Writing $\epsilon=1-\beta^*$ gives
	\[
	\epsilon^3=(b^2/v)(1-\epsilon)(2-\epsilon)^2,
	\]
	which yields $1-\beta^*\sim(4b^2/v)^{1/3}$ as $b^2/v\to0$.
	
	\begin{corollary}[Strict advantage over a shared constant memory]
		\label{cor:noise-lag-shared}
		For finitely many directions, let $v_j>0$, $b_j\neq0$, and $w_j>0$, and let $\beta_j^*$ minimize $R(\beta;v_j,b_j)$. If two ratios $b_j^2/v_j$ differ, then
		\[
		\sum_jw_jR(\beta_j^*;v_j,b_j)
		<
		\min_{0\le\beta<1}\sum_jw_jR(\beta;v_j,b_j).
		\]
	\end{corollary}
	\begin{proof}
		The strict monotonicity in the optimality equation gives distinct unique minimizers whenever the ratios $b_j^2/v_j$ differ. Every shared coefficient is therefore strictly suboptimal for at least one positive-weight direction. The shared objective is continuous and diverges as $\beta\uparrow1$, so it attains its minimum on a compact subset of $[0,1)$. Evaluating the strict pointwise inequality at that minimizer proves the claim.
	\end{proof}
	The corollary compares constant coefficients in an exogenous affine-signal estimation problem. It is not a lower bound against every adaptive scalar-momentum method and does not establish end-to-end acceleration for Muon.
	
	\section{Numerical Experiments and Reproducibility Details}
	The synthetic experiments in this appendix diagnose the displayed theoretical mechanisms using the exact partial polar factor. The preliminary language-model pretraining comparison is documented separately in Appendix~\ref{app:llm-trajectory}.
	
	\label{app:numerical-validation}
	
	\subsection{Common protocol and map from main-text figures}\label{app:num-common}
	
	\paragraph{Where each main-text figure is documented.}
	\begin{center}
		\small
		\begin{tabular}{p{0.14\linewidth} p{0.28\linewidth} p{0.48\linewidth}}
			\hline
			Main-text figure & Appendix protocol & Role in the paper \\
			\hline
			Fig.~\ref{fig:isotropic-main} & App.~\ref{app:num-isotropic} & Direct check of Proposition~\ref{prop:gaussian-isotropization}. \\
			Fig.~\ref{fig:coverage-main}(a) & App.~\ref{app:num-two-task} & Direct check of Theorem~\ref{thm:two-task-coverage}. \\
			Fig.~\ref{fig:coverage-main}(b) & App.~\ref{app:num-common}, \ref{app:num-hierarchical} & Trajectory-level regime summary plus a frozen-state check of the perturbation rate underlying the well-sampled regime in Proposition~\ref{prop:hierarchical-regimes}. \\
			Fig.~\ref{fig:label-regimes-main} & App.~\ref{app:num-label-drift}, \ref{app:num-label-local} & Full-matrix transfer/stress tests kept separate from the projected counting law and covariance closure used in Proposition~\ref{prop:label-regimes}. \\
			Fig.~\ref{fig:signsgd-main} & App.~\ref{app:num-signsgd} & Paired long-time trajectory check of the sparse-support negative-transfer mechanism in Proposition~\ref{prop:signsgd-contrast}. \\
			Fig.~\ref{fig:momentum-main} & App.~\ref{app:num-two-task} & Direct check of the variance-equivalent sampling and tracking scales in Proposition~\ref{prop:temporal-aggregation}. \\
			Fig.~\ref{fig:fast-slow-pretraining} & App.~\ref{app:llm-trajectory} & Single-run validation-loss comparison in language-model pretraining. \\
			\hline
		\end{tabular}
	\end{center}
	
	\paragraph{Two replication notions.}
	Unless specified otherwise, synthetic trajectory experiments use $R_{\rm traj}=64$ independent mini-batch streams per configuration, and frozen-state checks use $N_{\rm MC}=8192$ independent Monte Carlo replicates. Experiment-specific sample counts are given below, including larger samples for Gaussian coefficient estimates. Trajectory runs measure variability of an entire optimization path, whereas frozen-state replicates estimate a conditional expectation or coupling norm. For antithetic estimators, uncertainty should be computed across independent pairs, because the two evaluations within each pair are dependent. Unless stated otherwise, curves show means, trajectory ribbons show interquartile ranges, and confidence intervals are $1.96$ standard errors over independent replicates. Log--log exponents are ordinary least-squares fits to the displayed grid points. The pretraining comparison contains one run per method and has no repeated-run uncertainty estimate.
	
	\paragraph{Exact Muon map and numerical invariances.}
	Every synthetic Muon simulation uses the exact SVD partial polar factor rather than an approximate Newton--Schulz implementation. If $A=U\operatorname{diag}(\sigma_1,\ldots)V^\top$ and $\sigma_{\max}>0$, singular values are treated as numerical zeros only relative to the matrix scale, using a cutoff proportional to $\sigma_{\max}$; the zero matrix maps to zero. This preserves positive-scale invariance. Deterministic regression tests verify positive-scale invariance, left--right orthogonal equivariance, the rank-one polar formula, and agreement between the full $2\times2$ matrix update and the reduced two-task recursion, all to errors below $4\times10^{-15}$.
	
	\paragraph{Hierarchical trajectory experiment in Figure~\ref{fig:coverage-main}.}
	We simulate the deterministic-label associative-memory recursion directly in task coordinates,
	\[
	W_{t+1}=W_t+\eta\,\operatorname{msgn}\!\big((I-S(W_t))Q_t\big),
	\qquad
	Bq_t\sim\operatorname{Multinomial}(B,p),
	\qquad W_0=0.
	\]
	There are $M=4$ frequency blocks of size $C=6$ ($K=24$ tasks) with block masses
	$(0.55,0.25,0.13,0.07)$, so each task in block $\mu$ has probability $\widetilde p_\mu/C$. We use $B\in\{8,64,256,1024\}$, step size $\eta=0.02$, and $T=800$ updates. For block $\mu$ with index set $\mathcal I_\mu$, the plotted margin is
	\[
	m_\mu(t)
	=\frac1C\sum_{j\in\mathcal I_\mu}
	\left(W_{jj}(t)-\frac{1}{C-1}
	\sum_{\substack{i\in\mathcal I_\mu\\ i\neq j}}W_{ij}(t)\right).
	\]
	The late growth rate is the OLS slope of $m_\mu$ against continuous time $\tau=\eta t$ over the last $40\%$ of the recorded trajectory. The main text summarizes the resulting sampling collapse in panel~(b) of Figure~\ref{fig:coverage-main}; the raw small- and large-expected-count trajectories are retained below for reference.
	
	\begin{figure}[t]
		\centering
		\includegraphics[width=0.96\linewidth]{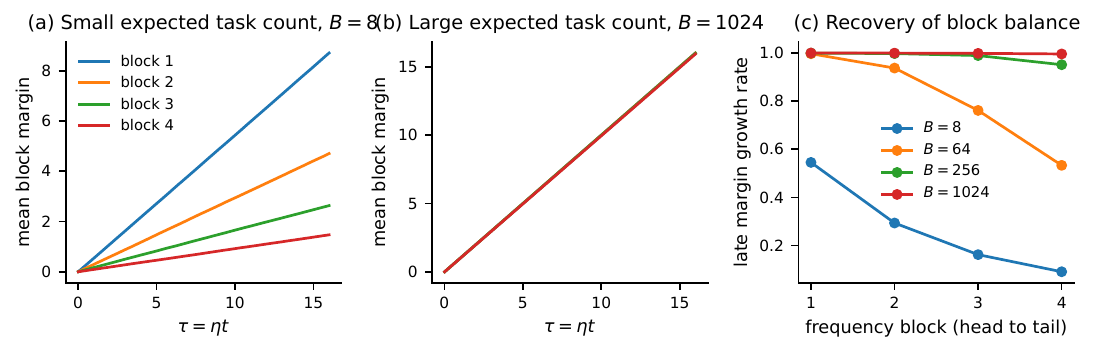}
		\caption{\textbf{Higher-rank trajectory view underlying Figure~\ref{fig:coverage-main}(b).} At small expected task count ($B=8$), block margins grow at visibly different rates; at large expected task count ($B=1024$), all four blocks approach the unit full-batch rate. The main text uses the rate-versus-$Bp_j$ collapse because it exposes the controlling expected-count variable more directly.}
		\label{fig:num-hierarchical-trajectories-app}
	\end{figure}
	
	\begin{table}[ht]
		\centering
		\small
		\begin{tabular}{p{0.53\linewidth}cc}
			\hline
			Diagnostic & measured & prediction \\
			\hline
			Isotropic anisotropic-signal median matrix error & $3.05\%$ & $\to0$ as $\varepsilon\to0$ \\
			Two-task equilibrium-gap bias exponent & $-1.011$ & $-1$ \\
			Two-task gap-fluctuation exponent & $-0.531$ & $-1/2$ \\
			Hierarchical spatial-coupling exponent & $-0.512$ & $-1/2$ \\
			Hierarchical momentum-coupling exponent & $-0.509$ & $-1/2$ \\
			Projected large-margin drift RMSE & $0.00618$ & $0$ \\
			Direct isotropic-closure frequency exponent & $0.501$ & $1/2$ \\
			Direct isotropic-closure mean coefficient error & $0.61\%$ & $\to0$ in closure limit \\
			Full-categorical stress-test exponent & $0.494$ & $1/2$ under auxiliary closure \\
			\hline
		\end{tabular}
		\caption{\textbf{Quantitative summary of the controlled numerical checks.} Exact and reduced-model validations are kept separate from transfer tests. The isotropic and closure errors are finite-signal/finite-Monte-Carlo quantities whose limiting prediction is zero; the full-categorical label-noise exponent is a stress test rather than a direct validation of the isotropic closure.}
		\label{tab:core-checks-app}
	\end{table}
	
	\subsection{Isotropic Gaussian mean response with an anisotropic signal}\label{app:num-isotropic}
	
	The first-order theorem is directional, so we use a deliberately non-isotropic signal rather than a multiple of the identity. For each $d\in\{4,8\}$, define
	\[
	H_d=\frac{\operatorname{diag}(s_1,\ldots,s_d)}{\|\operatorname{diag}(s_1,\ldots,s_d)\|_F},
	\qquad
	(s_1,\ldots,s_d)=\operatorname{geomspace}(1,0.08;d).
	\]
	For $\varepsilon\in\{0.04,0.07,0.12,0.20\}$ we estimate
	\[
	F_d(\varepsilon)=\mathbb E\,\operatorname{msgn}(\varepsilon H_d+\Xi)
	\]
	using $32{,}768$ antithetic Gaussian pairs. The coefficient $c_d=d^{-2}\mathbb E\|\Xi\|_*$ is estimated independently from $65{,}536$ Gaussian matrices, so agreement does not arise from sharing Monte Carlo noise between response and theory. We report both the scalar projection $\langle \widehat F_d(\varepsilon),H_d\rangle/\varepsilon$ and the full matrix residual
	\[
	\frac{\|\widehat F_d(\varepsilon)-c_d\varepsilon H_d\|_F}{\|c_d\varepsilon H_d\|_F}.
	\]
	Across the eight settings, the median full-matrix relative error is $3.05\%$. For the directional-gain panel in Figure~\ref{fig:isotropic-main}(a), we additionally repeat the $d=8$ experiment with $131{,}072$ antithetic pairs and an independent $131{,}072$-sample estimate of $c_d$, then report the gain along each singular direction divided by $c_d$.
	
	The corresponding response curves are shown in Figure~\ref{fig:isotropic-main} in the main text.
	
	\subsection{Two-task finite-batch transition, asymptotics, and momentum}\label{app:num-two-task}
	
	\paragraph{Direct finite-batch transition.}
	To test Theorem~\ref{thm:two-task-coverage} itself, we use $p_{\min}\in\{0.02,0.05,0.10\}$ and choose batch sizes spanning both sides of $B_{\rm bal}$. For each configuration we estimate the late-time growth rate of the gap $D(t)=\Delta_1(t)-\Delta_2(t)$ across 64 trajectories and plot it against $B/B_{\rm bal}$. In the small-batch phase,
	\[
	\frac{D(t)}{\eta t}\longrightarrow \sqrt2\bigl(2p_1^B-1\bigr),
	\]
	whereas the balanced phase has zero asymptotic gap slope. The three frequency settings collapse around $B/B_{\rm bal}=1$ and agree with the piecewise theoretical curve.
	
	The direct finite-batch transition plot is panel~(a) of Figure~\ref{fig:coverage-main} in the main text.
	
	\paragraph{Large-batch asymptotics.}
	We use the exact two-task recursion of Appendix~\ref{app:two-task-proof} with $p_2=0.05$, $p_1=0.95$, $\eta=0.003$, $T=30{,}000$, and
	\[
	B\in\{50,80,120,200,320,500,800\}.
	\]
	For each of the 64 trajectories we discard the first $60\%$ of updates. Over the remaining segment we estimate the time-averaged gap $\Delta_1-\Delta_2$ and its stationary standard deviation (sampling every five iterations). The equilibrium target is $\log(p_1/p_2)$. Regressing the mean equilibrium bias and the stationary standard deviation against $B$ gives exponents $-1.011$ and $-0.531$, respectively, close to the predicted $-1$ and $-1/2$.
	
	\begin{figure}[t]
		\centering
		\includegraphics[width=0.96\linewidth]{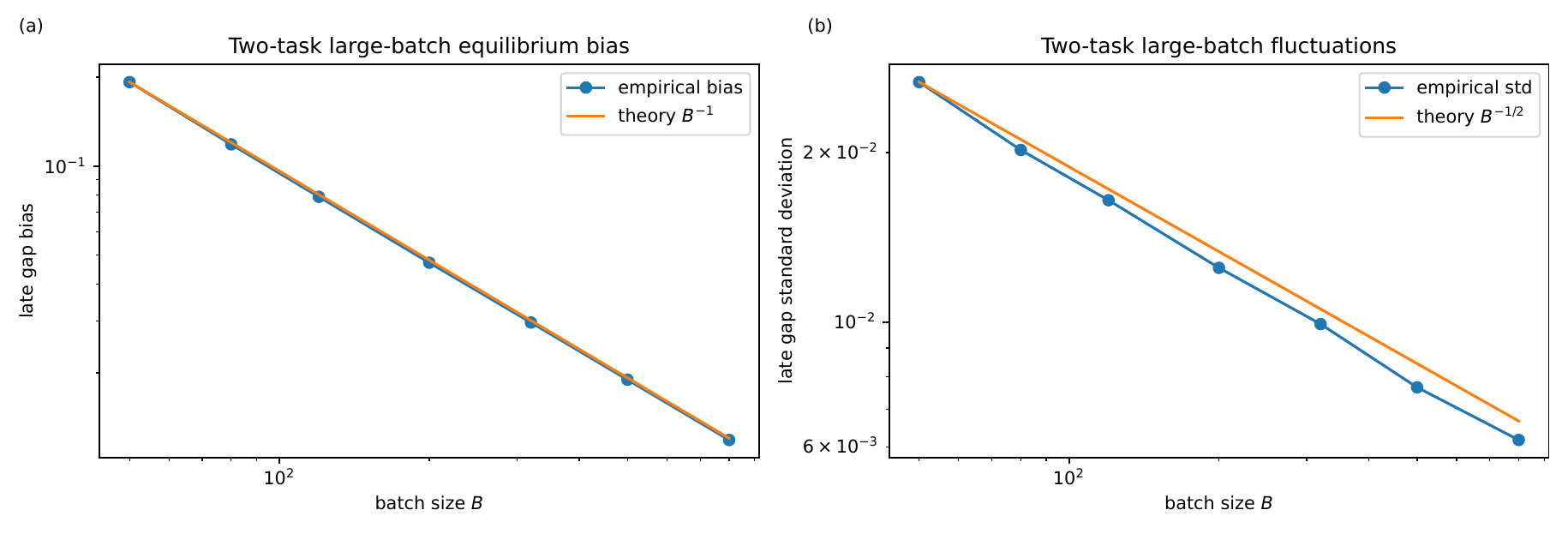}
		\caption{\textbf{Two-task large-batch asymptotics.} Both panels use $p_2=0.05$, $\eta=0.003$, and 64 trajectories. (a) Absolute late-time bias of $\Delta_1-\Delta_2$ from the full-batch target $\log(p_1/p_2)$ versus batch size $B$ on log--log axes; the fitted exponent is $-1.011$, compared with the $B^{-1}$ prediction. (b) Late-time stationary standard deviation of the gap versus $B$; the fitted exponent is $-0.531$, compared with the $B^{-1/2}$ diffusion prediction.}
		\label{fig:num-two-task-refined-app}
	\end{figure}
	
	\paragraph{Variance-equivalent sampling under momentum.}
	For the momentum-collapse check we keep $p_2=0.05$, use $\eta=0.002$ and $T=30{,}000$, and target
	$B_{\rm eff}\in\{40,80,160,320\}$. For every $B\in\{2,4,8,16,32\}$ with $B\le B_{\rm eff}$, we choose
	\[
	\beta=\frac{B_{\rm eff}/B-1}{B_{\rm eff}/B+1},
	\]
	so that $B(1+\beta)/(1-\beta)=B_{\rm eff}$. The same 60\% burn-in convention is used. For the tracking check we instead fix $B=4$ and $\beta=0.95$, sweep
	\[
	\eta\in\{5\!\times\!10^{-4},0.002,0.01,0.03,0.05,0.10,0.20\},
	\]
	and choose
	\[
	T=\operatorname{round}(50/\eta),
	\]
	so that the continuous-time horizon satisfies $\tau_{\max}=\eta T\approx50$ for every learning rate. We again discard the first 60\% of updates and compute the mean gap over the retained late-time window; the plotted mean-gap error is its absolute deviation from the full-batch target $\log(p_1/p_2)$. The measured error grows once $\eta/(1-\beta)$ is no longer small.
	
	\noindent The corresponding collapse/tracking summary is shown directly in Figure~\ref{fig:momentum-main} of the main text.
	
	\subsection{Hierarchical well-sampled perturbation scaling}\label{app:num-hierarchical}
	
	These checks are deliberately \emph{frozen-state} Monte Carlo experiments rather than optimization trajectories. We set $M=3$, $C=8$, block masses $(0.5,0.3,0.2)$, and test the third block. To remove any unspecified initialization, the reference logit matrix is fixed deterministically as follows. If column $j$ belongs to block $\nu$ and row $i$ to block $\mu$, then
	\[
	W_{ij}=\begin{cases}
		2+0.1\nu, & i=j,\\
		0.3+0.05\nu, & \mu=\nu,~i\neq j,\\
		-0.2-0.05\mu+0.03\nu, & \mu\neq\nu,
	\end{cases}
	\]
	with zero-based block indices $\mu,\nu\in\{0,1,2\}$. We set $A=I-S(W)$, draw a multinomial frequency vector $q$, decompose it into its block-average and within-block parts $q=q^\parallel+q^\perp$, and estimate the derivative
	\[
	\frac{\operatorname{msgn}(A\operatorname{diag}(q^\parallel+hq^\perp))-
		\operatorname{msgn}(A\operatorname{diag}(q^\parallel-hq^\perp))}{2h}.
	\]
	Each point below averages $N_{\rm MC}$ independent draws of $q$.
	
	For cancellation we fix $B=640$ and use
	$h\in\{10^{-2},3\!\times\!10^{-3},10^{-3},3\!\times\!10^{-4},10^{-4}\}$.
	For spatial scaling we fix $h=10^{-4}$ and sweep
	$B\in\{80,160,320,640,1280,2560\}$.
	For momentum scaling we freeze the same state, set $B=16$, and sweep
	$\beta\in\{0.50,0.75,0.875,0.94,0.97\}$. For each $\beta$, the EMA frequency vector is initialized at $p$ and run for $\lceil 12/(1-\beta)\rceil$ mini-batches before measuring the same finite-difference derivative. The resulting spatial and momentum exponents are $-0.512$ and $-0.509$, respectively.
	
	\begin{figure}[t]
		\centering
		\includegraphics[width=0.98\linewidth]{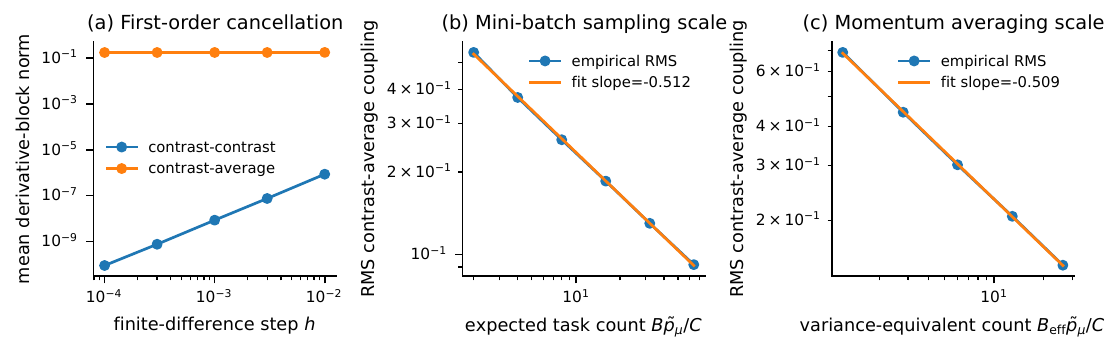}
		\caption{\textbf{Higher-rank perturbation checks at a frozen block-symmetric state.} (a) Mean finite-difference derivative norm versus step $h$: the contrast-to-contrast block decays to numerical zero while contrast-to-average coupling remains $O(1)$, verifying first-order cancellation. (b) RMS contrast-to-average coupling versus expected task count $B\widetilde p_\mu/C$; the fitted log--log slope is $-0.512$, matching the predicted $-1/2$. (c) The same coupling versus variance-equivalent task count $B_{\rm eff}\widetilde p_\mu/C$ under EMA momentum; the fitted slope is $-0.509$. Each point uses $N_{\rm MC}=8192$ draws.}
		\label{fig:num-hierarchical-app}
	\end{figure}
	
	\subsection{Label-noise large-margin drift: projected validation and full-matrix transfer}\label{app:num-label-drift}
	
	\paragraph{Projected large-margin formula.}
	The large-margin drift $\chi_{B,p}$ in~\eqref{eq:app-label-large-margin-response} is exact for the longitudinal projected counting model, not for the complete categorical matrix. We validate that reduced model directly before testing transfer to full matrix Muon. We use $K=8$, $p_j=1/8$, margin $12$, $\alpha\in\{0.05,0.10,0.20\}$, and expected corrupted-sample counts
	\[
	\lambda=Bp_j\rho_\alpha\in\{0.20,0.35,0.50,\log2,1.0,1.4,2.0\}
	\]
	(up to integer rounding of $B$). Each point uses $16{,}384$ samples. Direct projected Monte Carlo has RMSE $0.00618$ from $\chi_{B,p}$. The complete matrix update shows the same sign transition around $\lambda\simeq\log2$ but is not expected to equal the projected formula; its RMSE relative to the projected prediction is $0.0233$.
	
	The full-matrix transfer summary is shown in panel~(a) of Figure~\ref{fig:label-regimes-main}; the direct projected-model validation is isolated below so that reduced-model validation and full-matrix transfer are not conflated.
	
	\begin{figure}[t]
		\centering
		\includegraphics[width=0.68\linewidth]{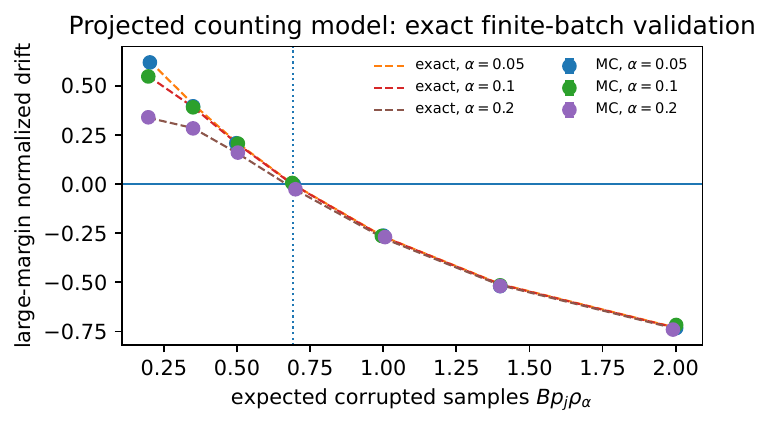}
		\caption{\textbf{Direct validation of the projected large-margin counting law.} Points are Monte Carlo estimates in the longitudinal projected model and dashed curves are the exact finite-$B$ formula. The full-matrix transfer test is shown in Figure~\ref{fig:label-regimes-main}(a), where deviations from the projected formula are expected because transverse categorical modes are retained.}
		\label{fig:num-label-projected-validation-app}
	\end{figure}
	
	\subsection{Label-noise local response: direct closure test and categorical stress test}\label{app:num-label-local}
	
	\paragraph{Direct test of the stated isotropic covariance closure.}
	We simulate exactly the Gaussian proxy used to derive~\eqref{eq:app-label-explicit-rate-finite-K}. We take $K=6$, $\alpha=0.2$, $B=5000$, local displacement $h=0.01$, and frequency scales
	\[
	p\in\{0.02,0.04,0.08,0.16,0.32\}.
	\]
	The projected signal and isotropic Gaussian covariance follow Appendix~\ref{app:label-dynamics}, and each point uses $32{,}768$ antithetic samples. Here $p$ is the frequency scale of the local proxy, not a common probability assigned to all six tasks in a normalized categorical distribution. The fitted frequency exponent is $0.502$ and the mean relative error of the explicit finite-$K$ coefficient is $0.61\%$.
	
	The direct Gaussian-closure check is retained below; panel~(b) of Figure~\ref{fig:label-regimes-main} instead shows the stronger full-categorical frequency-scaling stress test.
	
	\begin{figure}[t]
		\centering
		\includegraphics[width=0.62\linewidth]{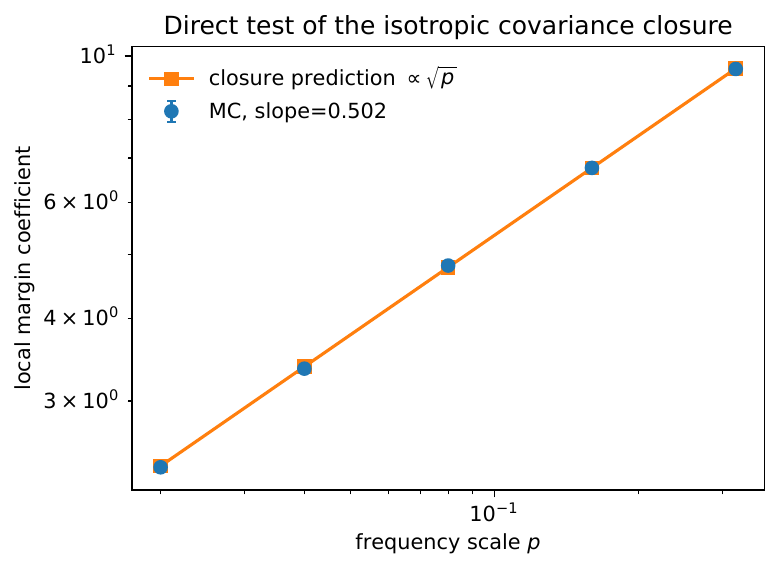}
		\caption{\textbf{Direct check of the moment-matched isotropic covariance closure.} The Monte Carlo coefficient follows the predicted $p^{1/2}$ scaling with fitted exponent $0.502$. This panel validates the stated Gaussian proxy itself; the main text uses the full-categorical simulation as the more demanding stress test.}
		\label{fig:num-label-closure-app}
	\end{figure}
	
	\paragraph{Full-categorical heterogeneous-frequency stress test.}
	We separately retain the stronger complete-categorical simulation. It uses $K=6$, $\alpha=0.2$, $B=5000$, frequencies $p=(0.35,0.25,0.16,0.11,0.08,0.05)$, and a local margin perturbation $\delta=0.03$. Each target task is evaluated with $N_{\rm MC}=8192$ complete mini-batches. The empirical Muon frequency exponent is $0.49397$, while the corresponding SGD exponent is $1$; the mean relative error against the auxiliary isotropic-closure coefficient is $3.25\%$. Because the exact categorical covariance is anisotropic, this agreement is reported as a stress test beyond the closure assumptions, not as theorem-level validation.
	
	\noindent The full-categorical frequency-scaling result is the quantity plotted in Figure~\ref{fig:label-regimes-main}(b).
	
	\subsection{Muon--SignSGD sparse-support contrast}\label{app:num-signsgd}
	
	Figure~\ref{fig:signsgd-main} uses the exact two-task recursions from Appendices~\ref{app:two-task-gradient} and~\ref{app:two-task-signsgd}. We set $p_1=0.95$, $p_2=0.05$, $B=1$, $\eta=0.01$, $T=3000$, and $\theta=1.15$ radians, and run $64$ independent paired mini-batch streams. The same task draw is fed to Muon and SignSGD at every step. For SignSGD, the positive answer-embedding scalar $\widetilde c$ in~\eqref{eq:app-signsgd-ctilde} is absorbed into the learning rate so that the experiment isolates the rotation-dependent coupling.
	
	For Muon, a $B=1$ batch updates only the sampled task, giving asymptotic rates $\sqrt2 p_1$ and $\sqrt2 p_2$ per unit optimization time. The rare-task margin therefore remains nonnegative and grows on its occurrence clock. For SignSGD, $\theta>\pi/4$ makes every task-1-only step decrease the task-2 margin, and the sufficient condition in Proposition~\ref{prop:signsgd-contrast} holds. We record both margins every 20 steps and plot the median and interquartile range across the 64 paired runs.
	
	The final recorded medians at $\tau=\eta T=30$ are approximately $40.28$ and $2.15$ for the frequent and rare Muon margins, versus $38.40$ and $-12.35$ for SignSGD. These trajectories are the two panels of Figure~\ref{fig:signsgd-main}.
	
	\subsection{Language-model pretraining comparison}
	\label{app:llm-trajectory}
	
	We compare standard Muon with the fast--slow construction of \citet{zhu2026multiscale} on a 0.12B-parameter dense language model trained for 30,000 steps. The remaining training settings are matched. The metric is validation loss, with one recorded trajectory per method. The export contains 140 checkpoints; the first two have no matching Muon value, leaving 138 paired checkpoints at 200-step intervals from step 2,600 to 30,000. We use these paired observations without interpolation or smoothing.
	
	The fast--slow run has lower validation loss at every paired checkpoint. Final losses are $3.009736$ for Muon and $2.988844$ for fast--slow momentum, a difference of $0.020892$. The final-ten-checkpoint means are $3.010568$ and $2.990669$, respectively. Figure~\ref{fig:fast-slow-pretraining} displays the 51 paired checkpoints from step 20,000 to 30,000. Checkpoints within a trajectory are not independent training replications; these observations do not establish variability across seeds or isolate the contribution of each algorithmic component.
	
\end{document}